\documentclass[twoside]{article}
\usepackage[preprint]{aistats2026}

\usepackage[round]{natbib}

\usepackage[utf8]{inputenc}
\usepackage[T1]{fontenc}

\usepackage{amsmath,amssymb,amsfonts,amsthm}
\usepackage{mathtools}

\usepackage{graphicx}
\usepackage[export]{adjustbox}
\usepackage{wrapfig}
\usepackage{subcaption}
\usepackage{caption}

\usepackage{booktabs}
\usepackage{multirow}

\usepackage{xcolor}
\usepackage{url}
\usepackage{microtype}
\usepackage{nicefrac}
\usepackage{enumitem}
\usepackage{xifthen}
\usepackage{xargs}
\usepackage{bm}
\usepackage{bbm}
\usepackage{stmaryrd}
\usepackage{nccmath}
\usepackage{multicol}
\usepackage{comment}
\usepackage{titlesec}
\usepackage{titletoc}
\usepackage{siunitx}
\usepackage{algorithm,algorithmic}
\usepackage[textsize=tiny]{todonotes}
\usepackage[most]{tcolorbox} 
\usepackage{xparse}          

\usepackage[colorlinks=true,linkcolor=blue,filecolor=blue,citecolor=purple,urlcolor=cyan]{hyperref}
\usepackage{cleveref}

\newtheorem{thm}{Theorem}[section]
\newtheorem{corollary}[thm]{Corollary}
\newtheorem{lemma}[thm]{Lemma}

\newtheorem{example}[thm]{Example}
\newtheorem{definition}[thm]{Definition}
\newtheorem{proposition}[thm]{Proposition}

\numberwithin{equation}{section}

\crefname{thm}{Theorem}{Theorems}
\Crefname{thm}{Theorem}{Theorems}
\crefname{lemma}{Lemma}{Lemmas}
\Crefname{lemma}{Lemma}{Lemmas}
\crefname{proposition}{Proposition}{Propositions}
\Crefname{proposition}{Proposition}{Propositions}
\crefname{corollary}{Corollary}{Corollaries}
\Crefname{corollary}{Corollary}{Corollaries}
\crefname{definition}{Definition}{Definitions}
\Crefname{definition}{Definition}{Definitions}
\crefname{assumption}{Assumption}{Assumptions}
\Crefname{assumption}{Assumption}{Assumptions}
\crefname{remark}{Remark}{Remarks}
\Crefname{remark}{Remark}{Remarks}
\crefname{example}{Example}{Examples}
\Crefname{example}{Example}{Examples}

\newcommand{\eps}{\varepsilon
}

\newcommand{\pcal}{{\cal P}}
\newcommand{\ind}[1]{\one\ens{#1}}

\def\ce{\mathrm{CE}}
\def\uc{\mathrm{UC}}

\newcommand{\ucal}{{\cal U}}

\newcommand{\bcal}{{\cal B}}

\newcommand{\xcal}{{\cal X}}
\newcommand{\ycal}{{\cal Y}}

\newcommand{\rad}{\mathfrak{R}}

\newcommand{\gcal}{\mathcal{G}}
\newcommand{\Gcal}{\mathcal{G}}
\newcommand{\Hcal}{\mathcal{H}}
\newcommand{\E}{\mathbb{E}}
\newcommand{\R}{\mathbb{R}}

\newcommand{\dcal}{{\mathcal D}}
\newcommand{\fcal}{{\cal F}}

\newcommand{\hcal}{{\cal H}}

\newcommand{\beq}{\begin{equation}}
\newcommand{\eeq}{\end{equation}}

\newcommand{\one}{\mathbbm{1}}

\def\rset{\mathbb{R}}

\def\nset{\mathbb{N}}

\newcommand{\old}[1]{\textbf{\textcolor{gray}{}}}

\DeclareMathOperator*{\expec}{\displaystyle \mathbb{E}}
\DeclareMathOperator*{\prob}{\displaystyle \mathbb{P}}

\newcommand{\mce}{\mathrm{MCE}}

\newtheorem{result}{Result}

\newcommand{\1}{\mathds{1}}

\DeclareMathOperator*{\argmax}{\displaystyle  \text{argmax}}

\DeclareMathOperator{\Sub}{Sub}
\DeclareMathOperator{\VC}{VC}
\DeclareMathOperator{\Pdim}{Pdim}

\newcommand{\norm}[1]{\left \lVert #1 \right \rVert}
\newcommand{\open}[1]{\left ( #1 \right )}
\newcommand{\closed}[1]{\left [#1 \right]}
\newcommand{\ens}[1]{\left \{ #1 \right \} }
\newcommand{\modu}[1]{\left \lvert #1 \right \rvert }
\newcommand\inner[2]{\left \langle #1, #2 \right \rangle}

\newtcolorbox[auto counter,number within=section]{pabox}[2][]{%
colback=red!5!white,colframe=red!75!black,coltitle = white!20!white,fonttitle=\bfseries,
title=~\thetcbcounter: #2,#1}

\newcounter{algbox}

\newcommand{\algbox}[3][]{%
  \begin{tcolorbox}[colframe=black, colback=white, boxrule=0.5mm, width=\linewidth, arc=0mm, outer arc=0mm]
  #3
  \end{tcolorbox}
  \begin{center}
    \ifx&#1&%
    \refstepcounter{algbox}
      \textbf{Algorithm \thealgbox:} #2
    \else
      \textbf{#1}{#2}
    \fi
  \end{center}
}

\definecolor{blush}{rgb}{0.87, 0.36, 0.51}

\newcommand{\tri}{\Delta}

\newcounter{mydefcounter}
\newcounter{myrescounter}

\def\xcal{\mathcal{X}}
\def\pcal{\mathcal{P}}

\def\ycal{\mathcal{Y}}
\def\wcal{\mathcal{W}}

 \def\rset{\mathbb{R}}
\def\nset{\mathbb{N}}  
\def\eqdef{:=}

\def\argmax{\mathrm{argmax}}

\newcommand{\CPE}[3][]
{\ifthenelse{\equal{#1}{}}{{\mathbb E}\left[\left. #2 \, \right| #3 \right]}{{\mathbb E}_{#1}\left[\left. #2 \, \right | #3 \right]}}

\def\1{\mathsf{1}}

\def\MM12{{\bf {\sf MM1\&2}}}

\newcommand{\cinv}[2]{\mathbb{I}\closed{#1,#2}}
\def\cutoff{\Delta_{\operatorname{Cutoff}}}
\def\lin{\mathrm{lin}}
\def\rank{\mathrm{rank}}

\begin{document}

\twocolumn[

\aistatstitle{Scalable Utility-Aware Multiclass Calibration}

\aistatsauthor{ Mahmoud Hegazy \And Michael I. Jordan \And Aymeric Dieuleveut }

\aistatsaddress{ CMAP, \'Ecole polytechnique,\\ IP Paris, France; \\Inria Paris, France \And  Inria Paris, France; \\ EECS, UC Berkeley, USA \And CMAP, \'Ecole polytechnique,\\ IP Paris, France } ]

\begin{abstract}
    Ensuring that classifiers are well-calibrated, i.e., their predictions align with observed frequencies, is a minimal and fundamental requirement for classifiers to be viewed as trustworthy. Existing methods for assessing multiclass calibration often focus on specific aspects associated with prediction (e.g., top-class confidence, class-wise calibration) or utilize computationally challenging variational formulations. In this work, we study scalable \emph{evaluation} of multiclass calibration. To this end, we propose utility calibration, a general framework that measures the calibration error relative to a specific utility function that encapsulates the goals or decision criteria relevant to the end user. We demonstrate how this framework can unify and re-interpret several existing calibration metrics, particularly allowing for more robust versions of the top-class and class-wise calibration metrics, and, going beyond such binarized approaches, toward assessing calibration for richer classes of downstream utilities.
\end{abstract}

\section{INTRODUCTION}

Calibration is a fundamental property of probabilistic predictors. A calibrated model produces predictions that, on average, align with observed frequencies. For instance, if a weather forecaster predicts a 30\% chance of rain on a given day, rain should occur on approximately 30\% of such days. In multiclass classification problems, calibration ensures that the predicted probabilities reflect the true likelihood of each class. Formally, let $\mathcal{X}$ denote the input space, $\mathcal{Y} = \{e_1, \ldots, e_C\}$ the output space, where $e_i$ is the $i$-th canonical basis vector in $\rset^C$, and $\tri^{C-1}\coloneqq\ens{x\in \rset_+^C|\sum_i x_i =  1}$ denote the simplex in $\rset^C$. A predictor $f: \mathcal{X} \to \Delta^{C-1}$ is said to be perfectly calibrated with respect to a distribution $D$ over $\mathcal{X} \times \mathcal{Y}$ if $\mathbb{E}[Y \mid f(X)] = f(X)$.
The most direct metric for quantifying the deviation from perfect calibration is the Mean Calibration Error ($\mce$).
\begin{definition}[Mean Calibration Error]
For a distribution $D$ such that $(X, Y) \sim D$ and a predictor~$f$, the mean calibration error is defined as $    \mathrm{MCE}(f) \coloneqq  \expec\closed{  \norm{\mathbb{E}[Y \mid f(X)] - f(X) }_2}.$
\end{definition}
Without further assumptions, the $\mathrm{MCE}$ is fundamentally impossible to estimate, even in the binary setting \citep{lee2023t,duchi2024lecturenotes}. While assumptions like H\"older continuity of $\expec\closed{Y|f(X)}$ allow for consistent estimators of $\expec\closed{Y|f(X)}$ or minimax optimal tests for $\mathrm{MCE}(f)$ \citep{lee2023t,popordanoska2022consistent,tsybakov2009nonparametric}, their sample complexity scales exponentially with the dimension $C$.

Due to the difficulty of measuring $\mce$, multiple relaxations have been proposed, falling into two main categories: \emph{binarized} and \emph{variational}. First, binarized approaches \citep{gupta2022top,panchenko2022class,guo2017calibration} simplify the problem by focusing on specific binary events derived from the multiclass predictions, e.g.~top-class or class-wise calibration. However, these methods are by nature presumptive of downstream tasks. Moreover, their reliance on binning schemes or kernel estimators for the underlying binary subproblems introduces sensitivity to estimator choices and can suffer from high bias \citep{roelofs2022mitigating}. Second, variational approaches \citep{jung2021moment,blasiok2023unifying,gopalan2024computationally,kumar2018trainable,widmann2019calibration,zhao2021calibrating} assess calibration through optimization problems, such as the distance to the nearest perfectly calibrated predictor or the worst-case error against a class of witness functions. Unfortunately, these methods can be computationally intensive and can scale poorly in $C$.

To address these limitations and provide an application-focused perspective on calibration, we introduce \textit{utility calibration}. This framework evaluates a model $f$ by considering a downstream user who employs its predictions $f(X)$. The core idea is to measure calibration error relative to a specific \textit{utility function}, denoted $u$, which encapsulates the goals, costs, or decision criteria relevant to this end user. Utility calibration then assesses how well the \textit{expected utility} (as estimated by the user based on $f(X)$ and $u$) aligns with the \textit{realized utility} (obtained when the true outcome $Y$ is observed). 

In practice, models often serve diverse users or a single user with multiple objectives. We thus extend utility calibration to handle \textit{classes of utility functions}. The overall utility calibration for  a class~$\cal U$ can be defined as the worst-case error over $u \in \ucal$, denoted $\uc(f, \ucal)$. A notable aspect of this formulation is that it provides a structured way to express and analyze various existing calibration notions. In particular, by defining appropriate utility functions within $\ucal$, concepts such as top-class and class-wise calibration can be cast within the utility calibration framework. This offers a unified perspective and a superior alternative to binning for examining those notions of calibration.

\vspace{-0.75em}

\paragraph{Contributions and outline:} In \Cref{sec:related_works}, we review related literature on calibration metrics and post-hoc calibration methods.
In \Cref{sec:utility_calibration}, we define utility calibration and relate it to existing measures of calibration. In addition, we demonstrate how this framework can be used to frame several existing calibration concepts within a common utility-aware perspective, offering consistent interpretations and providing examples of relevant utility classes. Then, in \Cref{sec:measure_uc}, to characterize the difficulty of achieving utility calibration for classes of utility functions, we introduce the notions of \textit{proactive} and \textit{interactive} measurability. While for rich utility classes, proactive measurability is not possible, we show that interactive measurability is achievable for many classes of interest. Drawing on these insights, in \Cref{sec:experiments}, we empirically demonstrate the application of our proposed metrics and evaluation methodology. While the main motivation of this work is developing a better methodology for calibration assessment, we also show that our framework naturally yields a simple patching-style post-hoc calibration algorithm that achieves competitive performance against commonly used post-hoc calibration methods.

\vspace{-0.75em}

\paragraph{Notation:} For any vector $w \in \rset^C$, $w_i$ denotes its $i$-th component and $\gamma(w)\coloneqq\argmax_iw_i$. For a probability vector $p \in \Delta^{C-1}$, we write $Z \sim p$ to denote a categorical random variable $Z$ taking values in $\ycal = \{e_1, \dots, e_C\}$ such that $\prob\ens{Z=e_i} = p_i$, where $e_i$ is the $i$-th canonical basis vector.  We use $\one_E$ for the indicator function of an event $E$. $\mathbb{E}[\cdot]$ denotes expectation, which is typically taken w.r.t.~$(X,Y)\sim D$ and, for $k\in \nset_+$, $[k]=\ens{1,\ldots, k}$. Finally, for $a,b \in \rset$ with $a<b$, we denote $\cinv{a}{b}$ to be the set of closed interval subsets of $[a,b]$.

\section{RELATED WORK}\label{sec:related_works}
\vspace{-0.25cm}
\def\tce{\mathrm{TCE}}
\def\cwe{\mathrm{CWE}}
\def\topk{\mathrm{topK}}
\def\rank{\mathrm{rank}}

In this section, we review three classical and related approaches to measuring or ensuring a form of calibration, namely post-hoc calibration methods, binarized relaxations, and variational approaches.

First, \textbf{post-hoc calibration} refers to techniques applied to a pre-trained model's outputs to improve the alignment between its predicted probabilities and the true likelihood of outcomes without altering the original model parameters.  Such methods are advantageous as they decouple calibration concerns from model training.

Popular post-hoc calibration algorithms include Temperature Scaling and its multi-parameter extensions, Vector Scaling and Matrix Scaling \citep{guo2017calibration}, which may all be regarded as a multiclass extension of Platt's scaling \citep{platt1999probabilistic}. Dirichlet calibration assumes the model's predicted probability vectors can be modeled by a Dirichlet distribution, whose parameters are learned on a calibration set to transform the original probabilities \citep{kull2019beyond}. Nonparametric methods such as Histogram Binning \citep{zadrozny2001obtaining} and Isotonic Regression \citep{zadrozny2002transforming} learn calibration maps by discretizing the probability space or fitting monotonic (order-preserving) functions, respectively. Other methods also include: \citep{patel2020multi}, which applies a specific binning strategy followed by recalibration to minimize class-wise calibration error, \citep{rahimi2020intra}, which uses order-preserving transformations for recalibration to maintain accuracy.
Finally, a related body of literature aims to improve calibration by regularizing the training objective, e.g.~\citep{mukhoti2020calibrating, popordanoska2022consistent, marx2024calibration}.

While post-hoc methods aim to minimize the calibration error, assessing the calibration error itself remains a non-trivial task.  \textbf{Binarized relaxations} aim to circumvent the difficulty of measuring the calibration error of a high-dimensional predictor $f$ by measuring the $\mathrm{MCE}$ of a single or multiple downstream binary versions of $f$ instead. Two commonly used relaxations are the Top-Class calibration Error ($\tce$) \citep{guo2017calibration} and the Class-Wise calibration Error ($\cwe$) \citep{panchenko2022class}, defined as

{%
\setlength{\abovedisplayskip}{4pt}
\setlength{\belowdisplayskip}{4pt}
\setlength{\abovedisplayshortskip}{4pt}
\setlength{\belowdisplayshortskip}{4pt}
\setlength{\jot}{2pt} 
\begin{align*}
\tce(f) &= \expec\big[\,|\,\expec[Y_*\mid p_*]-p_*\,|\,\big],\\
\cwe(f) &= \sum_{i\in[C]} w_i\,\expec\big[\,|\,\expec[\one_{Y=e_i}\mid f_i]-f_i\,|\,\big],
\end{align*}
}

with the shorthands $f_i\coloneqq f(X)_i$, $p_*\coloneq f(X)_{i_*}$, and $Y_*\coloneq \one_{Y=e_{i_*}}$, where $i_*\coloneq \argmax_i f(X)_i$. In addition, $w_i$ is a class-dependent weight, which can be set to $1/C$, $w_i=\prob\{Y = e_i\}$, or another choice $w\in \tri^{C-1}$.

Both $\tce$ and $\cwe$ require the estimation of conditional expectation, which is typically approximated using binning schemes. For $(B_j)_{j\in[m]}$ a partition of $[0,1]$, the binned estimators are

{%
\setlength{\abovedisplayskip}{4pt}
\setlength{\belowdisplayskip}{4pt}
\setlength{\abovedisplayshortskip}{4pt}
\setlength{\belowdisplayshortskip}{4pt}
\setlength{\jot}{2pt} 
\begin{equation}
\label{eq:tce}
\tce^{\mathrm{bin}}(f) =  \sum_{j\in[m]} \big|\expec\closed{(p_* - Y_*)\one_{p_*\in B_j}} \big|
\end{equation}
\begin{equation}
\label{eq:cwce}
\cwe^{\mathrm{bin}}(f) = \sum_{\mathclap{{i\in[C],j \in [m]}}} w_i \modu{\expec\closed{\big(f_i - \one_{Y=e_i}\big) \one_{f_i\in B_j}}}.
\end{equation}
}

\citet{gupta2022top} unified multiple instances of binarized proxies of $\mce$, such as $\tce, \cwe$ and $\topk$ confidence calibration, introduced in \citep{gupta2021calibration}, and proposed additional binarized reductions which offer stronger notions of calibration.  Unfortunately, the binning schemes used in such binarized proxies are known to have a large effect on the estimated error \citep{roelofs2022mitigating, gruber2022better}. Apart from the simpler equal-size bins \citep{guo2017calibration} and equal-weight bins \citep{zadrozny2001obtaining}, multiple binning schemes built on top of different heuristics have been proposed~\citep[see, e.g.,][]{roelofs2022mitigating,patel2020multi, naeini2015obtaining, nixon2019measuring}. \citet{gupta2021distribution} showed a simple equal-weight binning scheme with better sample complexity guarantees for estimating bin averages. \citet{kumaar2018VUC} developed adaptive binning schemes with guarantees for discrete $f$ and showed that for any binning scheme, there exists a worst-case continuous $f$ such that the bias of $\tce^{\mathrm{bin}}(f)$ as an estimate of $\tce(f)$ is lower bounded by $0.49$ (noting that $\tce$ is bounded between $0$ and $1$). 

On the other hand, there exist binning-free alternatives for binarized reductions~\citep[see, e.g.,][]{popordanoska2022consistent, gupta2021calibration}. Nonetheless, in an assumption-free setting, it is generally impossible to consistently estimate the $\mce$ of binary predictors \citep{lee2023t,duchi2024lecturenotes, rossellini2025can}. As such, it is generally difficult to control the calibration error defined by binarized relaxations. 
 
On the other hand, \textbf{variational approaches} do not strictly aim to measure the $\mce$. Instead, they consider alternative formulations that do not require direct estimation of the conditional expectation. For example, Distance to Calibration ($\mathrm{DC}$) 
measures the distance between $f$ and the nearest perfectly calibrated predictors \citep{blasiok2023unifying}:

{
  \setlength{\abovedisplayskip}{4pt}
  \setlength{\belowdisplayskip}{4pt}
  \setlength{\abovedisplayshortskip}{4pt}
  \setlength{\belowdisplayshortskip}{4pt}
  \begin{equation*}
    \mathrm{DC}(f) \,\coloneqq\, \inf_{\mce(g)=0} \, \mathbb{E}\bigl[\|f(X)-g(X)\|_1 \bigr].
    \end{equation*}
}

A unified formulation of variational measures of calibration is weighted calibration, which assesses the calibration error against a class of witness functions \citep{jung2021moment}. Concretely, let $\wcal$ be a class of functions mapping $\Delta^{C-1}$ to $[-1,1]^{C}$. Then, weighted calibration error with witness class $\wcal$ is
\label{def:CE_witness}
\begin{equation}\label{eq:weightCE}
 \ce_{\mathcal{W}}(f) = \mathop{\sup}_{w \in \wcal} \mathbb{E}_{X,Y}\closed{\inner{w\bigl(f(X)\bigr)}{f(X) - Y}}.
\end{equation}
A specific instance of weighted calibration is the Kernel Calibration Error (KCE) \citep{lin2023taking}, which sets $\wcal$ to be the unit ball of the reproducing kernel Hilbert space (RKHS) of a multivariate universal kernel. This allows for efficient computation of the supremum but it remains hard to interpret the impact of low KCE for a user of $f$. \citet{blasiok2023unifying} showed that in the binary setting, $\mathrm{DC}(f)$ and $\ce_{\mathrm{Lip(1)}}(f)$ are equivalent up to a (low-degree) polynomial scaling and that $\ce_{\mathrm{Lip}(1)}(f)$ can be well approximated by the RKHS of the Laplace kernel, where $\mathrm{Lip}(1)$ is the class of $1$-Lipschitz functions from $\tri^{C-1}$ to $[-1,1]$.

The result on the equivalence between $\mathrm{CE}_{\mathrm{Lip}(1)}(f)$ and $\mathrm{DC}(f)$ was further extended to the multiclass setting in \cite[Theorem 15.5.5]{duchi2024lecturenotes} and \cite[Lemma 3.3]{gopalan2024computationally}. In particular, \citet{gopalan2022low} showed that measuring either $\mathrm{DC}(f)$ or $\mathrm{CE}_{\mathrm{Lip}(1)}(f)$ requires an exponential number of samples in $C$ \citep[Theorem 3.2. and Theorem 3.4.]{gopalan2024computationally}. Thus, even though $\mathrm{DC}\open{f}$ can be efficiently assessed in the binary setting, it is quickly intractable as $C$ increases.

A particular case is \textit{Decision calibration}, introduced by \citet{zhao2021calibrating}, that tailors calibration guarantees to downstream decision-making tasks. A predictor $f$ is considered decision calibrated of order $K$ if, for any decision problem involving at most $K$ actions, the expected loss computed using the model's predictions $f(X)$ accurately matches the true expected loss incurred. Formally, for any loss function $\ell$ mapping an outcome-action pair to a real-valued loss, decision calibration of order $K$ requires:
\begin{equation*}
\mathbb{E}\big[\ell\big(\hat{Y}, \delta(f(X))\big)\big] \,=\, \mathbb{E}\big[\ell\big(Y, \delta(f(X))\big)\big].
\end{equation*}
where $\hat{Y}\sim f(X)$  and $\delta$ is a decision rule that picks the best action among $K$ actions under the model's prediction $f(X)$. This ensures that decision-makers can reliably estimate the consequences of actions when using the predictor. A key contribution of \citet{zhao2021calibrating} is showing that decision calibration of order $K$ can be achieved by having $\sup_{p\in P(K)}\|\expec[(Y-f(X))\one\ens{f(X)\in p}]\|=0$, where $P(K)$ is the set of polytopes with at most $K$ supporting hyperplanes. Moreover, low decision calibration guarantees a notion of no-regret to downstream users, i.e., they cannot improve their utility by using any other best response policy; associated with another loss $\ell'$ instead of $\ell$.

In a similar aim to \citet{zhao2021calibrating}, multiple works have studied calibration through a downstream decision-theoretic perspective. Broadly, two complementary approaches have emerged: (i) online, regret-driven forecasting that enforces event-conditional notions of calibration and yields no (swap)-regret guarantees for best-responding agents \citep{noarov2025highdimensional, roth2024forecasting}; and (ii) decision-centric objectives that guarantee good performance across rich families of utilities or proper scoring rules, e.g., U-calibration \citep{kleinberg2023u}. 

Foundationally, these formulations target an ambitious goal, ensuring that downstream agents who best-respond using the predictor $f$ achieve (near-)optimal utility across scenarios and losses. However, this ambition often comes with scalability costs: U-calibration is developed in the binary online setting; the swap-regret guarantees in \cite{roth2024forecasting} rely on Lipschitz utilities in low dimensions (e.g., \(d\in\{1,2\}\)) and, in higher dimensions, assume a bounded action set and smooth best responses; and measuring decision calibration is computationally intractable in $C$, with exponential complexity even for $K=2$ \citep{gopalan2024computationally}.

By contrast, we adopt a scalable assessment perspective. Our utility calibration (shortly introduced in \Cref{sec:utility_calibration}) requires that the predicted utility \(v_u(X)\) is a reliable regressor of the realized utility \(u(f(X),Y)\). This prioritizes the reliability of utility estimation--not enforcing optimal downstream decisions--so users can recognize when their expected utility is poor rather than be misled by optimistic forecasts. Operationally, it reduces multiclass assessment to a binning-free worst-interval deviation in \(v_u\), scales well with \(C\), and supports interactively measurable audits across broad utility classes, complementing online/regret-focused formulations.

\section{UTILITY CALIBRATION (UC)}
\label{sec:utility_calibration}

We consider the following utility-centric formulation of calibration. In particular, we are interested in the setting where, for some input $X$, a downstream user leverages $f(X)$ as an estimate of $\expec\closed{Y\mid X}$. Based on this estimate of the conditional expectation, the user may then take arbitrary actions or decisions. Finally, the user observes the true realization of the label $Y$ and, based on this realization, may then suffer some loss or achieve some gain. To model such a pipeline of observation, action, then consequences, we consider a utility function $u:\tri^{C-1}\times\ycal\rightarrow[-1,1]$ such that $u(f(X),Y)$ models the reward obtained or the loss suffered by the decision-makers after using $f(X)$ to take arbitrary actions/decisions. In such a setting, predictability is highly desirable, in the sense that when using the predictor $f$, the utility obtained is similar to the utility expected. More concretely, for $\hat{Y}\sim f(X)$ and a given input $X$, the user can use $f(X)$ to construct the following estimate of utility:
\begin{equation}
    v_u(X) \coloneqq \expec\closed{u(f(X),\hat{Y}) \mid X }= \inner{f(X)}{\vec{u}(X)}, \label{eq:vU}
\end{equation}
where $\vec{u}:\xcal\rightarrow[-1,1]^C$ is defined as $\vec{u}(X)\coloneqq (u\open{f(X),e_i})_{i\in [C]}$. Ideally, we want the function $v_u(X)$ to be an unbiased estimator of the true utility. As such, we define the utility calibration with respect to a utility function $u$, denoted by $\uc\open{f,u}$, as
\begin{equation}\label{eq:defUC}
    \sup_{I\in\cinv{-1}{1}}  \modu{\expec\closed{(u(f(X),Y)-v_u(X))\one_{v_u(X)\in I}}}.
\end{equation}
We say that $f$ is $\eps$-calibrated with respect to a utility function $u$ if $\uc(f,u)\leq \eps$. Note that for $I = [a,b]$, the inner term in \eqref{eq:defUC} can be rewritten as
\begin{equation*}
\modu{\expec\closed{\open{u(f(X),Y)-v_u(X)} \big|  v_u(X)\in[a,b]}} p_{a,b}, 
\end{equation*}
where $p_{a,b} =  \prob\ens{v_u(X)\in [a,b]}$.  In words, looking at the instances where $v_u(X)\in [a,b]$, the bias between the utility the decision-maker expects to get (while using $f(X)$ to take decisions and to estimate the utility) and the actual utility the decision-maker achieves (when using $f(X)$ to take decisions), is at most $\eps$ after being weighted by the probability of $\ens{v_u(X)\in [a,b]}$. 

Combining \eqref{eq:vU} and \eqref{eq:defUC} above, one obtains that $\uc\open{f,u}$ is equivalent to

\begin{equation}
\label{eq:UCasWCE}
\sup_{I\in \cinv{-1}{1}} \left| \expec\!\left[ \inner{Y-f(X)}{\vec{u}(X)}\one\{v_u(X)\in I\} \right] \right|.
\end{equation}

Thus, utility calibration is equivalent to weighted calibration \eqref{eq:weightCE}, with the witness class $\wcal$ set to $\wcal\open{u} \coloneqq  \ens{x\mapsto  \xi\vec{u}(x)\one\ens{v_u(x)\in I}|I\in\cinv{-1}{1}}$. In addition, our notion of utility calibration requires that the predicted label $\hat{Y}\sim f(X)$ can be used for an unbiased estimate of the utility. This is related to {Outcome Indistinguishability (OI)} \citep{dwork2021outcome}, where a predictor $f$ is considered reliable if its simulated outcomes $\hat{Y} \sim f(X)$ are computationally indistinguishable from Nature's true outcomes $Y$. We also note that this perspective connects to recent work that leverages OI variants to establish links between loss minimization guarantees, omnipredictors, and multicalibration \citep{gopalan2022omnipredictors,gopalan2023loss, gopalan2023swap}. 

\subsection{Decision-Theoretic Implications of UC
}

In a recent work, for the binary classification setting, \citet{rossellini2025can} introduced the CutOff calibration metric, which assesses the calibration error by measuring against the worst-case bin, and demonstrated that it provides robust decision-theoretic guarantees. We defer a more detailed discussion of CutOff calibration to Appendix \ref{app:cutoff_calibration}. By assessing the $\uc\open{f,u}$ on the worst-case interval of $v_u(\cdot)$, our construction of utility calibration can be seen as a generalization of CutOff calibration to multiple dimensions and arbitrary utility functions, and it in fact inherits analogous decision-theoretic guarantees to those shown in \citet[Prop 2.1 and 3.2]{rossellini2025can}. 

In particular, consider a decision rule based on thresholding the predicted utility $v_u(X)$ at some level $t_0 \in [-1,1]$, i.e., taking the action $\hat{U}_{t_0}\coloneqq\one\ens{v_u(X) \ge t_0}$. This models the situation in which a user needs to commit a binary decision after estimating the utility using $f(X)$. Then, the quality of this decision can be assessed by the loss  $\ell_{\mathrm{util}}(\tilde u, \widehat{U}; t) = \modu{ \tilde u- t}\one\{\hat{U} \neq \ind{u \ge t}\}$, which penalizes the \textit{deviation} between the true utility $u_Y$  and the decision threshold $t_0$ when a mismatch between $\hat{U}_{t_0}$ and the ideal decision occurs. Consequently, let $R_{\mathrm{util}}(g; t_0) = \mathbb{E}[\ell_{\mathrm{util}}(u(f(X),Y), \hat{U}_{t_0}; t_0)]$ be the associated risk. Then, we  show that the decision process $\hat{U}_{t_0}$ cannot significantly be improved by any simple post-processing of $v_u(\cdot)$ through a composition with a monotone function. 

\begin{proposition}[Utility Risk Gap] \label{prop:uc_risk_gap_main}
For any utility $u$ and threshold $t_0\in[-1,1]$,
\begin{equation*}
\begin{aligned}
R_{\mathrm{util}}(v_u& (X); t_0) - 
\\&\inf_{\substack{h:[-1,1]\to[-1,1]\\ \textnormal{monotone}}} R_{\mathrm{util}}(h(v_u(X)); t_0)\le 2\,\uc(f,u).
\end{aligned}
\end{equation*}
\end{proposition}

In words, \Cref{prop:uc_risk_gap_main} indicates that, if $f$ is utility calibrated,  in such a binary decision-making scenario, the user can barely benefit from any monotonic post-processing to $v_u$. Another interpretation of $v_u(X)$ is as a regressor for the realized utility $u_Y \coloneqq u(f(X),Y) \in[-1,1]$. Similar to \citet[Prop 2.1]{rossellini2025can}, we can show that the regressor $v_u$ satisfies a notion of calibration itself. First, note that distance from calibration naturally extends to such a single-dimension regression problem by considering a function $g_u(X)$ to be a perfectly calibrated predictor of $u_Y$ if $\mathbb{E}[u_Y \mid g_u(X)] = g_u(X)$, a.s. We denote this extended notion of distance from calibration as $\mathrm{DCU}(f,u)$, the Distance to Calibrated Utility Predictor for $v_u(X)$ with respect to the realized utility $u(f(X),Y)$:
\begin{align*}
\mathrm{DCU}\open{f}
&\coloneqq \inf_{g_u: \mathcal{X} \to [-1,1]} \; \mathbb{E}\left| g_u(X) - v_u(X) \right|
\\
&\text{s.t. } \; \mathbb{E}[u_Y \mid g_u(X)] = g_u(X) \, .
\end{align*}
We show that $\mathrm{DCU}(f,u)$ can be effectively controlled through $\uc(f,u)$.
\begin{proposition} [Utility Calibration Upper Bounds $\mathrm{DCU}$] \label{prop:uc_bounds_dcu_main}
Let $u:\tri^{C-1}\times\ycal\rightarrow[-1,1]$ be a utility function. Then,
\begin{equation*}
\mathrm{DCU}\open{f} \;\le\; \sqrt{8\,\uc(f,u)} + \uc(f,u).
\end{equation*}
\end{proposition}
 
\Cref{prop:uc_bounds_dcu_main} implies that if $\uc(f,u)$ is small, then $v_u(X)$, seen as a regressor for the true utility $u(f(X),Y)$, is a calibrated predictor itself. This further strengthens the interpretation of $\uc(f,u)$: not only does it \textit{ensure actionable decisions based on $v_u(X)$}, but it also \textit{guarantees that $v_u(X)$ is not far from calibration}. 

\subsection{Measuring \texorpdfstring{$\uc\open{f,u}$}{UC(f,u)}}
A naturally arising question is on the difficulty of measuring and achieving a small utility calibration error. We show in \Cref{lem:estimate_uc_single} that both the computational and sample complexity of estimating $\uc\open{f,u}$ are generally feasible and of limited dependence on the dimension, allowing its scalability to predictors with thousands of classes.

\begin{lemma}[Estimating Utility Calibration Against a Single Function]\label{lem:estimate_uc_single}
Let $u:\tri^{C-1}\times\ycal\rightarrow[-1,1]$ be a fixed utility function and $f: \mathcal{X} \to \Delta^{C-1}$ be a given predictor. Define the empirical estimator $\widehat{\uc}(f,u; S)$ based on $n$ i.i.d. samples $S = \{(X_i, Y_i)\}_{i=1}^n \sim D^n$ as
\begin{align*}
 \sup_{{I\in \cinv{-1}{1}}}\Big|\frac{1}{n}\sum_{i=1}^n \big[{(u(f(X_i),Y_i)-v_u(X_i))} \one_{v_u(X_i)\in I}\big]\Big|.
\end{align*}
Then, for any $\delta > 0$, with probability at least $1-\delta$ over the draws of the sample $S$,
\begin{equation}\label{eq:error_single_uc}
 \textstyle   |\widehat{\uc}(f,u; S) - \uc(f,u)| \le \Tilde{O} \open{ \sqrt{\frac{\log(1/\delta)}{n}}}.
\end{equation}
Furthermore, $\widehat{\uc}(f,u; S)$ can be computed from $S$ in $O(n\log(n) + n T_{eval})$ time, where $T_{eval}$ is the time to evaluate $f(X_i)$ and $u(\cdot, \cdot)$.
\end{lemma}

First,  we note that the constants hidden in the $\Tilde{O}(\cdot)$ in \eqref{eq:error_single_uc} are dimension-independent. Similarly, the only dimension-dependent term in the computational complexity is $T_{eval}$. As such, $\uc\open{f,u}$ is a completely scalable notion of calibration, allowing it to be implemented for classifiers with a thousand classes -- as exemplified in \Cref{sec:experiments}.  In addition, given that $\uc(f,u)$ can be formulated as weighted calibration (see eq.~\eqref{eq:UCasWCE}) and that $\widehat{\uc}(f,u; S)$ is both computationally and sample efficient, we can leverage standard patching-style post-hoc calibration \citep{jung2021moment, hebert2018multicalibration, duchi2024lecturenotes} to recalibrate $f$ so as to minimize $\uc\open{f,u}$ while decreasing its Brier score. Concretely, the recalibrator iteratively identifies worst-interval witnesses for $u$ and applies corrections to $f(X)$ that reduce the utility calibration error and the Brier score; see Appendix \ref{app:post_hoc} for a detailed description and Appendix \ref{app:experiments} for implementational details and experiments. With these encouraging facts on the utility calibration w.r.t.~a single $u$ being established, we next turn our attention to Utility Calibration against a function class $\mathcal U$.

\subsection{Utility calibration against a function class}

In many real-world scenarios, a single probabilistic predictor $f$ might serve multiple downstream users, or a single user might employ it under varying conditions or objectives. The exact utility function relevant at the time of decision-making may not be known beforehand, or it might even change over time (e.g., due to changing costs, available actions, or strategic goals), or it might be fundamentally user-dependent.

Therefore, ensuring reliability often requires guarantees that hold not just for a single, pre-specified utility function, but for an entire class of plausible or relevant utility functions, denoted by $\ucal$. This provides a more robust assurance that the model's predictions are trustworthy across a range of potential downstream applications. To capture this requirement, overloading the notation, we define utility calibration against a function class as the worst-case performance over the class, i.e. 
\begin{equation}
    \uc\open{f,\ucal} = \sup_{u\in \ucal} \uc\open{f,u}. \label{eq:utilcalclass}
\end{equation}

To illustrate the practical relevance of this concept, we exhibit hereafter several examples of utility classes, each motivated by different downstream tasks. We first demonstrate how to recover similar notions to top-class \eqref{eq:tce} and class-wise $\eqref{eq:cwce}$ using the framework of utility calibration~\eqref{eq:utilcalclass}.

\begin{example}[Top-Class and Class-Wise Utilities ($\ucal_{\tce}, \ucal_{\cwe}$)]\label{ex:top_class_wise}
Again using the shorthands $f_i\coloneqq f(X)_i$, $p_*\coloneq f(X)_{i_*}$, and $Y_*\coloneq \one_{Y=e_{i_*}}$, where $i_*\coloneq \argmax_i f(X)_i$, define 
\begin{align*}
\uc(f,\ucal_{\tce}) &= \sup_{I\subseteq[0,1]} |\expec[(Y_* - p_*)\one_{p_*\in I}] |, \\
\uc(f,\ucal_{\cwe}) &= \sup_{\mathclap{c\in[C], I \subseteq[0,1]}}| \expec[(\one_{Y=e_c} - f(X)_c) \one_{f(X)_c\in I}] |.
\end{align*}
\end{example}
In contrast to the binned estimators $\tce^{\mathrm{bin}}$ \eqref{eq:tce} and $\cwe^{\mathrm{bin}}$ \eqref{eq:cwce}, utility calibration with the classes $\ucal_{\tce}$ and $\ucal_{\cwe}$ offers a more robust, binning-free, computable assessment. Specifically, $\uc\open{f,\ucal_{\tce}}$ and $\uc\open{f,\ucal_{\cwe}}$ are determined by maximizing the calibration deviation over \textit{any} possible  interval $I \subseteq [0,1]$ (and additionally over classes for $\ucal_{\cwe}$), effectively identifying the worst-case interval-based error. This approach inherently avoids fixed binning schemes, thereby circumventing pathologies where bin choices drastically alter estimated errors \citep{roelofs2022mitigating, kumaar2018VUC}. Consequently, for any binning scheme using $m$ bins,  $m \cdot \uc\open{f, \ucal_\tce}$ and $m \cdot \uc\open{f, \ucal_\cwe}$ upper bound $\tce^{\mathrm{bin}}(f)$ and  $\cwe^{\mathrm{bin}}(f)$ respectively, while the converse is not true. We refer to Appendix \ref{app:uc_bound_binarized} for the formal statement. Furthermore, by \Cref{prop:uc_risk_gap_main}, a small $\uc\open{f, \ucal_{\tce}}$ guarantees that decisions based on thresholding top-class confidence are robust to monotonic recalibration, and  by \Cref{prop:uc_bounds_dcu_main}  that this confidence is a calibrated predictor of actual top-class accuracy. Analogous guarantees hold for $\uc\open{f, \ucal_{\cwe}}$ for individual class confidences, offering assurances for downstream applications.

Beyond the binarized perspectives offered by $\ucal_{\tce}$ and $\ucal_{\cwe}$, the utility calibration framework readily accommodates richer and more complex classes of utility functions. This allows us to move beyond presumptive binary events and consider more nuanced downstream applications. In particular, consider settings where the utility derived from an outcome $Y$ is intrinsic to the outcome itself, independent of the model's prediction $f(X)$. For example, in medical diagnosis, the cost or severity tied to a specific disease $Y=e_j$ might be a fixed value $a_j$, irrespective of the diagnostic prediction. Formally, such situations can be modeled using a utility function $u_a: \Delta^{C-1} \times \mathcal{Y} \to [-1, 1]$ defined by a payoff vector $a \in [-1, 1]^C$, where the utility function and the expected utility are respectively $u_a(\cdot, e_j) = a_j$ and $v_{u_a}(X) = \inner{f(X)}{a}$, with $a_j$ representing the utility if the true outcome is $e_j$.

\begin{example}[Linear Utilities ($\ucal_{\mathrm{lin}}$)]\label{ex:util_linear}
Define the class of linear utilities as $\ucal_{\mathrm{lin}} \coloneqq \{ u_a \mid a \in [-1, 1]^C \}$, noting that the predicted utility $v_{u_a}(X)$ is linear in the prediction $f(X)$.
\end{example}
 A small $\uc(f, \ucal_{\mathrm{lin}})$ ensures that for any payoff vector $a$, the predicted expected utility $v_{u_a}(X)$, as a regressor of the realized utility, is close to calibration. 

Alternatively, in applications like information retrieval or recommender systems, the realized utility depends on the rank assigned to the true outcome $Y=e_j$. Given a model's prediction $p=f(X)$, assuming $p_1,\ldots,p_C$ are distinct (or that ties are broken arbitrarily/randomly among equal coordinates), the rank of class $j$, denoted $\rank(p, j)$, is its position across $p$, i.e. $\rank(p, j) \coloneqq\sum_{i\in [C]}\one\ens{p_j\leq p_i}$. Using a valuation vector $\theta \in [-1,1]^C$, a rank-based utility function can then be constructed as $u_\theta(p, e_j) = \theta_{\rank(p,j)}$ with the associated expected utility function $v_{u_\theta}(X) = \sum_{i=1}^C f(X)_i \theta_{\rank(f(X), i)}$. Calibrating for such utilities ensures the model's expected rank-based performance aligns with reality.
A prominent special case is $\topk$ utility, where the valuation vector $\theta^{(K)}$ for a given $K \in [C]$ is defined such that $\theta^{(K)}_r = 1$ if $r \le K$ and $\theta^{(K)}_r = 0$ if $r > K$.

\begin{example}[Rank-Based and Top-$K$ Utilities ($\ucal_{\rank}, \ucal_{\topk}$)]\label{ex:util_rank}
The class of general rank-based utilities is $\ucal_{\rank} \coloneqq \{ u_\theta \mid \theta \in [-1, 1]^C \}$.
The class of top-$K$ utilities is then $\ucal_{\topk} \coloneqq \{ u_{\theta^{(K)}} \mid K \in [C] \}$, where $\theta^{(K)}_r = \one\{r \le K\}$. Equivalently, $u_K(p, e_j) = \one\{\rank(p, j) \le K\}$.
A small $\uc(f, \ucal_{\rank})$ (or $\uc(f, \ucal_{\topk}$)) ensures reliable prediction for general rank (or specifically top-$K$ accuracy) valuations, validating the model's ranking capabilities.
\end{example}

As discussed in \Cref{sec:related_works}, decision calibration \citep{zhao2021calibrating} ensures that for problems with up to $K$ actions, the model's predicted utility for its recommended action matches the actual realized utility. We can frame a similar guarantee within utility calibration. For any bounded loss function $l: \mathcal{Y} \times [K] \to [-1,1]$ and a prediction $p=f(X)$, the optimal action is $\delta_l(p) = \arg\min_{a \in [K]} \mathbb{E}_{\hat{Y} \sim p}[l(\hat{Y}, a)]$. The utility function is then $u_l(p, y) = -l(y, \delta_l(p))$, representing the negative loss from outcome $y$ under action $\delta_l(p)$. The predicted expected utility is $v_{u_l}(X) = -\mathbb{E}_{\hat{Y} \sim f(X)}[l(\hat{Y}, \delta_l(f(X)))|X]$.

\begin{example}[Decision Calibration Utilities ($\ucal_{\mathrm{dec}, K}$)]
Let $\mathcal{L}_K = \{ l : \mathcal{Y} \times [K] \to [-1,1] \}$ be the class of all bounded $K$-action loss functions, and the utility class is $\ucal_{\mathrm{dec}, K} \coloneqq \{ u_l,  l \in \mathcal{L}_K \}$.
A small $\uc(f, \ucal_{\mathrm{dec}, K})$ implies that for any $K$-action decision problem $l \in \mathcal{L}_K$, the model's prediction of expected utility for its chosen action $\delta_l(f(X))$ reliably reflects the achieved 
utility $-l(Y, \delta_l(f(X)))$.
\end{example}

These aforementioned examples illustrate that calibrating against classes $\ucal$ provides guarantees tailored to diverse user needs, moving beyond simplistic binarized assessments. A critical question then arises: how can $\uc\open{f, \ucal}$ be measured for a given class $\ucal$, which we address in the next section. 
\section{SCALABLE EVALUATION OF UTILITY CALIBRATION}\label{sec:measure_uc}
Estimating $\sup_{u\in \ucal}\uc\open{f,u}$ in \eqref{eq:utilcalclass} presents two key challenges: the \textit{computational complexity} of the optimization,  and the \textit{sample complexity} required for the empirical supremum to converge to its true value. We introduce the two notions of proactive and interactive measurability to decouple these two aspects.

\begin{definition}[Proactive Measurability] \label{def:proactive}
The utility calibration error w.r.t.~class $\ucal$ is \textit{proactively measurable} if there exists an algorithm $A$ and polynomial functions $N_{poly}, T_{poly}$ such that for any $\eps, \delta > 0$ and $n \ge N_{poly}(C , 1/\eps, 1/\delta)$ samples $S \sim D^n$, algorithm $A(S)$ outputs $\hat{u}$ satisfying $|\uc\open{f,\hat{u}} - \uc\open{f,\ucal}| \le \eps$ with probability at least $1-\delta$ and the runtime of $A(S)$ is bounded by $T_{poly}(C, n)$.
\end{definition}

Generally, for a finite class $\ucal$, if $\modu{\ucal}$ grows polynomially in $C$ then by \Cref{lem:estimate_uc_single} we can guarantee proactive measurability. Nonetheless, even for simple infinite classes such as $\ucal_{\mathrm{lin}}$, proactive measurability reduces to a non-convex optimization problem that cannot be generally solved in polynomial time. In fact, even aiming for a weaker notion, namely \textit{improper auditing}, \citet{gopalan2024computationally} showed that assessing both weaker and stronger notions than $\uc\open{f, \ucal_{\mathrm{lin}}}$ cannot be done in polynomial time in both the error ${\eps}^{-1}$ and the dimension $C$ \citep[Theorem 1.3, Theorem 5.2, and Theorem 8.6]{gopalan2024computationally}. A more detailed description of \citet{gopalan2024computationally} hardness results is in \Cref{app:proactive_cal_hard}. Next, we thus propose an alternative criterion of measurability that decouples the statistical guarantee from the computational complexity of verifying the supremum.

\begin{definition}[Interactive Measurability] \label{def:interactive}
The utility calibration error w.r.t.~class $\ucal$ is \textit{interactively measurable} if there exists an estimator $\widehat{\uc}(f,u; S)$ and a polynomial function $N_{poly}$ such that for $n \ge N_{poly}(C, 1/\eps, 1/\delta)$ samples $S \sim D^n$, it holds with probability at least $1-\delta$ that $\textstyle \sup_{u \in \ucal} |\widehat{\uc}(f,u; S) - \uc(f, u)| \le \eps$.
\end{definition}

Interactive measurability represents a much more achievable goal. For example, while decision calibration is computationally hard to measure, \citet{zhao2021calibrating} showed that it admits polynomial sample complexity. In Appendix~\ref{app:interactive_measure}, we prove the interactive measurability of $\ucal_\lin$ and $\ucal_\rank$ and that for both the sample complexity if $\tilde{O}(C)$. In addition, we demonstrate that for a class of utilities $\ucal$, it is interactively measurable if the functions of the form $f(X)\rightarrow u(f(X),e_i)\one_{v_u(X)\in I}$ for $u\in \ucal$, $i\in [C]$, and $I\in \cinv{-1}{1}$ admit a controlled Rademacher complexity. 

In summary, while proactively measuring the worst-case utility calibration error $\uc(f, \ucal) = \sup_{u \in \ucal} \uc(f, u)$ is often computationally prohibitive for expressive utility classes $\ucal$, interactive measurability allows for efficient estimation of $\uc(f,u)$ uniformly for any \textit{specific} $u \in \ucal$. Next, we leverage this distinction to propose a scalable evaluation methodology that, instead of pursuing the intractable worst-case error, characterizes the \textit{distribution} of utility calibration errors across $\ucal$. This provides a more nuanced understanding of a model $f$'s calibration reliability over a spectrum of potential downstream applications, that we then evaluate in experiments.

Our approach considers a probability distribution $\mathcal{D}_{\ucal}$ over the utility class $\ucal$. Many utility classes of interest admit a finite-dimensional parameterization, making sampling from $\mathcal{D}_{\ucal}$ practical. We sample $M$ utility functions $\{u_m\}_{m=1}^M$ from $\mathcal{D}_{\ucal}$ and, for each $u_m$, compute its estimated error $\widehat{E}_{m,n} \coloneqq \widehat{\uc}(f, u_m; S)$ using $n$ data points from a sample~$S$. These $M$ error estimates then form an \textit{empirical Cumulative Distribution Function (eCDF)}, $\textstyle \widehat{F}_{E,M,n}(e) \coloneqq \frac{1}{M} \sum_{m=1}^M \one\{\widehat{E}_{m,n} \le e\}$, which serves as an empirical proxy for the true CDF,  $F_E(e) \coloneqq \prob_{u \sim \mathcal{D}_{\ucal}}(\uc(f, u) \le e)$. We provide  guarantees on the difference between $F_E(e)$ and $\hat{F}_{E,M,n}(e)$ in Appendix~\ref{app:quantile_gaurantees}. Informally, if individual utility-calibration errors can be uniformly estimated to accuracy $\eps_{\text{stat}}$ from $n$ samples, then with high probability over the draw of the dataset and the $M$ sampled utilities, the deviation between the true CDF $F_E$ and the empirical eCDF $\widehat{F}_{E,M,n}$ in $L_2$ scales as $\tilde{O}\!\left(\sqrt{\eps_{\text{stat}}} + \sqrt{\tfrac{1}{M}}\right)$.

In particular, $\ucal_{\mathrm{lin}}$ (\Cref{ex:util_linear}) and $\ucal_\mathrm{rank}$  (\Cref{ex:util_rank}) both admit finite-dimension parameterization. For $\ucal_{\mathrm{lin}}$, we construct $\dcal_{\ucal_{\mathrm{lin}}}$ by sampling the payoff vectors $a$ uniformly in $\partial B_\infty \coloneq \ens{a\in \rset^C:\norm{a}_\infty = 1}$. Meanwhile, for $\ucal_{\mathrm{rank}}$, we also sample from $\dcal_{\ucal_{\mathrm{rank}}}$ by uniformly sampling valuation vectors $\partial B_\infty$, which satisfy $\theta_1 \ge \theta_2 \ge \dots \ge \theta_C$. This is to reflect a rational preference for better ranks, i.e. the higher the rank of the true realization within the predictions of $f(X)$, the higher the utility.

\vspace{-0.25cm}
\section{EXPERIMENTS}
\label{sec:experiments}
\label{sec:quantile_eval}
\vspace{-0.2cm}

\begin{table}[t!]
  \centering
  \resizebox{\columnwidth}{!}{
  \begin{tabular}{@{}l@{\hspace{0.5em}}@{}c@{\hspace{0.5em}}@{}c@{\hspace{0.5em}}@{}c@{\hspace{0.5em}}@{}c@{}}
    \toprule
    Method & Brier Score ($\times 10^{2}$) & $\mathrm{CWE}^{\mathrm{bin}}$ ($\times 10^{4}$) & $\mathrm{TCE}^{\mathrm{bin}}$ ($\times 10^{3}$) & $\mathcal{U}_{\mathrm{comb}}$ ($\times 10^{3}$) \\
    \midrule
    Uncalibrated & $22.5 \pm 0.13$ & $2.47 \pm 0.01$ & $94.50 \pm 0.83$ & $124.0 \pm 0.48$ \\
    Dirichlet & $\bm{21.3 \pm 0.10}$ & $1.35 \pm 0.01$ & $13.80 \pm 0.41$ & $26.1 \pm 0.67$ \\
    IR & $22.8 \pm 0.11$ & $\bm{1.12 \pm 0.01}$ & $31.70 \pm 0.59$ & $54.1 \pm 0.94$ \\
    Temp. Scaling & $21.6 \pm 0.10$ & $1.50 \pm 0.01$ & $21.20 \pm 0.41$ & $45.2 \pm 0.53$ \\
    Vector Scaling & $22.7 \pm 0.12$ & $1.55 \pm 0.02$ & $34.60 \pm 1.03$ & $37.4 \pm 2.13$ \\
    Patching & $21.6 \pm 0.13$ & $1.56 \pm 0.03$ & $\bm{11.10 \pm 0.66}$ & $\bm{22.1 \pm 1.66}$ \\
    \bottomrule
\end{tabular}
  }
  \caption{\texttt{ViT}-ImageNet-1K results. We report mean $\pm$ 2 standard errors over $10$ train/test splits.}\label{tab:imagenet_results_specific_uc} 
  \vspace{-0.15cm}
\end{table}

\begin{figure}[t!]
    \centering
    \includegraphics[width=\columnwidth]{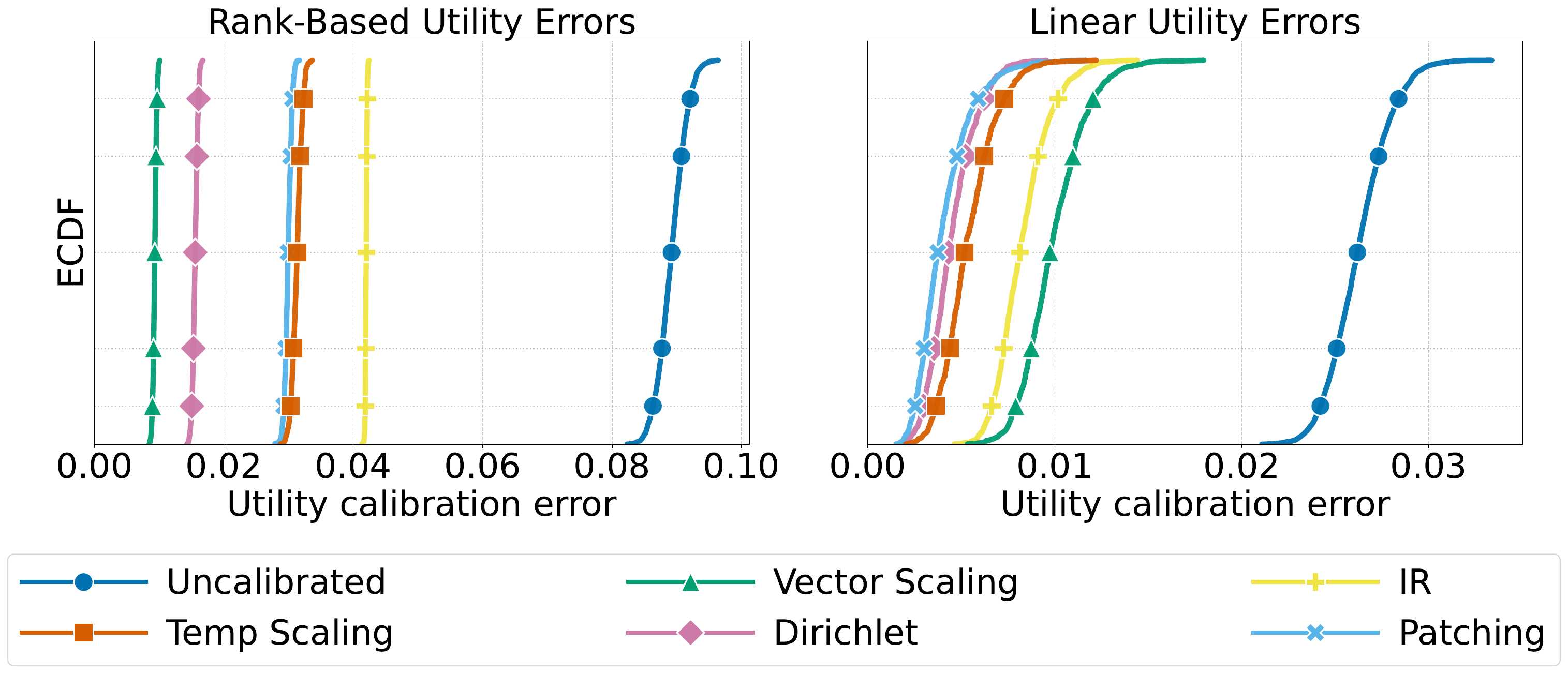}
    \caption{eCDF evaluation for \texttt{ViT} on ImageNet-1K on $\ucal_\rank$ and $\ucal_\lin$.}
    \label{fig:ecdf_linear_utilities} 
\end{figure}

We now demonstrate our evaluation on ImageNet-1K \citep{deng2009imagenet} with a pretrained Vision Transformer (\texttt{ViT}) \citep{dosovitskiy2021an}. We compare the calibration performance of Temperature Scaling (Temp. Scaling) \citep{platt1999probabilistic}, Vector Scaling \citep{kull2017beyond}, Dirichlet recalibration \citep{kull2019beyond}, and Isotonic Regression (I.R.) \citep{zadrozny2002transforming}. In addition, we use a simple auditing/patching-style post-hoc calibration algorithm inspired by previous work \citep{hebert2018multicalibration,jung2021moment}. In summary, this patching algorithm iteratively evaluates the utility calibration on the class $\mathcal{U}_{\mathrm{comb}}$, which combines class-wise and top-$K$ utilities, i.e., $\ucal_{\mathrm{comb}} \coloneqq \ucal_{\cwe}\cup\ucal_{\topk}$. Then, the algorithm \textit{patches} the pretrained predictor $f$ to improve the utility calibration error against the class $\ucal_{\mathrm{comb}}$. We provide a detailed description of the algorithm and our implementation in Appendices~\ref{app:post_hoc}\&\ref{app:experiments}.

Moreover, Appendix~\ref{app:experiments} provides extended results: additional architectures, datasets, and modalities, sensitivity to aligned vs. misaligned utilities, and additional commonly used utility families such as discounted cumulative gain \citep{jarvelin2002cumulated} and hierarchical classification loss \citep{deng2012hedging}. These experiments further illustrate the advantages of our framework and the failure modes of traditional metrics.

{In \Cref{tab:imagenet_results_specific_uc}}, we compare Brier score, binned binarized metrics $\mathrm{TCE}_{\mathrm{binned}}$ and $\cwe_{\mathrm{binned}}$ (with $15$ equal-weight bins), and a combined utility calibration metric w.r.t. the class $\mathcal{U}_{\mathrm{comb}}$. We note that, given that both $\ucal_{\cwe}$ and $\ucal_{\topk}$ are finite classes, $\mathcal{U}_{\mathrm{comb}}$ can be directly measured with no sampling required. As expected, all post-hoc methods improve Brier and calibration errors. Nonetheless, no single method performed the best uniformly across all metrics. We note that the patching-style post-hoc calibration achieves the best top-class binned calibration error and $\mathcal{U}_{\mathrm{comb}}$.

{Beyond specific utility functions, \Cref{fig:ecdf_linear_utilities}} displays eCDFs of utility calibration errors for broader utility classes: rank-based ($\ucal_{\mathrm{rank}}$) and linear ($\ucal_{\mathrm{lin}}$). For utility sampling, we used the approach described in \Cref{sec:measure_uc}. Each eCDF (from $M{=}1500$ sampled utilities) shows the proportion of utilities with error below a threshold; curves further left indicate better calibration across the class. On \texttt{ViT}-ImageNet-1K, the uncalibrated model is consistently worst across both classes, while post-hoc methods shift the curves left to varying degrees, revealing method-dependent tradeoffs. Interestingly, the eCDF curves surface trends that the aggregate scores in \Cref{tab:imagenet_results_specific_uc} can obscure: for example, how tightly errors concentrate around the median, and whether heavy tails persist, even when summary metrics look similar. In addition, performance is not uniform across utility families: some methods yield larger gains on $\ucal_{\lin}$ while others help more on $\ucal_{\rank}$, underscoring the importance of matching the calibrator to the target utility class. For example, while Vector Scaling performed the best on $\ucal_{\rank}$, it performed the worst on $\ucal_\lin$. 

In conclusion, utility calibration provides a robust, unified, and application-centric framework for evaluating classifier reliability. Its specific instantiations, $\ucal_{\cwe}$ and $\ucal_{\tce}$, offer superior, binning-free alternatives to traditional metrics with actionable guarantees. Furthermore, the eCDF plots across broader utility classes deliver crucial nuanced insights into model behavior that single-metric evaluations obscure.

\section*{Acknowledgments} 
The work of Aymeric Dieuleveut and Mahmoud Hegazy is supported by French State aid managed by the Agence Nationale de la Recherche (ANR) under the France 2030 program with the reference ANR-23-PEIA-005 (REDEEM project), and ANR-23-IACL-0005, in particular the Hi!Paris FLAG chair. Additionally, this project was funded by the European Union (ERC-2022-SYG-OCEAN-101071601). Views and opinions expressed are, however, those of the author(s) only and do not necessarily reflect those of the European Union or the European Research Council Executive Agency. Neither the European Union nor the granting authority can be held responsible for them. This publication is part of the Chair “Markets and Learning,” supported by Air Liquide, BNP PARIBAS ASSET MANAGEMENT Europe, EDF, Orange and SNCF, sponsors of the Inria Foundation.

\bibliographystyle{unsrtnat}
\bibliography{ref}

\begin{thebibliography}{61}
\providecommand{\natexlab}[1]{#1}
\providecommand{\url}[1]{\texttt{#1}}
\expandafter\ifx\csname urlstyle\endcsname\relax
  \providecommand{\doi}[1]{doi: #1}\else
  \providecommand{\doi}{doi: \begingroup \urlstyle{rm}\Url}\fi

\bibitem[Lee et~al.(2023)Lee, Huang, Hassani, and Dobriban]{lee2023t}
Donghwan Lee, Xinmeng Huang, Hamed Hassani, and Edgar Dobriban.
\newblock T-cal: An optimal test for the calibration of predictive models.
\newblock \emph{Journal of Machine Learning Research}, 24\penalty0 (335):\penalty0 1--72, 2023.

\bibitem[Duchi(2024)]{duchi2024lecturenotes}
John~C. Duchi.
\newblock Information theory and statistics.
\newblock \url{https://web.stanford.edu/class/stats311/lecture-notes.pdf}, 2024.
\newblock Lecture Notes for STATS 311 / EE 377, Stanford University. Version from March 12, 2024. Accessed: April 30, 2025.

\bibitem[Popordanoska et~al.(2022)Popordanoska, Sayer, and Blaschko]{popordanoska2022consistent}
Teodora Popordanoska, Raphael Sayer, and Matthew Blaschko.
\newblock A consistent and differentiable lp canonical calibration error estimator.
\newblock \emph{Advances in Neural Information Processing Systems}, 35:\penalty0 7933--7946, 2022.

\bibitem[Tsybakov(2009)]{tsybakov2009nonparametric}
Alexandre~B Tsybakov.
\newblock Nonparametric estimators.
\newblock \emph{Introduction to Nonparametric Estimation}, pages 1--76, 2009.

\bibitem[Gupta and Ramdas(2022)]{gupta2022top}
Chirag Gupta and Aaditya Ramdas.
\newblock Top-label calibration and multiclass-to-binary reductions.
\newblock In \emph{International Conference on Learning Representations}, 2022.

\bibitem[Panchenko et~al.(2022)Panchenko, Benmerzoug, and de~Benito~Delgado]{panchenko2022class}
Michael Panchenko, Anes Benmerzoug, and Miguel de~Benito~Delgado.
\newblock Class-wise and reduced calibration methods.
\newblock In \emph{2022 21st IEEE International Conference on Machine Learning and Applications (ICMLA)}, pages 1093--1100. IEEE, 2022.

\bibitem[Guo et~al.(2017)Guo, Pleiss, Sun, and Weinberger]{guo2017calibration}
Chuan Guo, Geoff Pleiss, Yu~Sun, and Kilian~Q Weinberger.
\newblock On calibration of modern neural networks.
\newblock In \emph{International conference on machine learning}, pages 1321--1330. PMLR, 2017.

\bibitem[Roelofs et~al.(2022)Roelofs, Cain, Shlens, and Mozer]{roelofs2022mitigating}
Rebecca Roelofs, Nicholas Cain, Jonathon Shlens, and Michael~C Mozer.
\newblock Mitigating bias in calibration error estimation.
\newblock In \emph{International Conference on Artificial Intelligence and Statistics}, pages 4036--4054. PMLR, 2022.

\bibitem[Jung et~al.(2021)Jung, Lee, Pai, Roth, and Vohra]{jung2021moment}
Christopher Jung, Changhwa Lee, Mallesh Pai, Aaron Roth, and Rakesh Vohra.
\newblock Moment multicalibration for uncertainty estimation.
\newblock In \emph{Conference on Learning Theory}, pages 2634--2678. PMLR, 2021.

\bibitem[B{\l}asiok et~al.(2023)B{\l}asiok, Gopalan, Hu, and Nakkiran]{blasiok2023unifying}
Jaros{\l}aw B{\l}asiok, Parikshit Gopalan, Lunjia Hu, and Preetum Nakkiran.
\newblock A unifying theory of distance from calibration.
\newblock In \emph{Proceedings of the 55th Annual ACM Symposium on Theory of Computing}, pages 1727--1740, 2023.

\bibitem[Gopalan et~al.(2024)Gopalan, Hu, and Rothblum]{gopalan2024computationally}
Parikshit Gopalan, Lunjia Hu, and Guy~N Rothblum.
\newblock On computationally efficient multi-class calibration.
\newblock \emph{arXiv preprint arXiv:2402.07821}, 2024.

\bibitem[Kumar et~al.(2018)Kumar, Sarawagi, and Jain]{kumar2018trainable}
Aviral Kumar, Sunita Sarawagi, and Ujjwal Jain.
\newblock Trainable calibration measures for neural networks from kernel mean embeddings.
\newblock In \emph{International Conference on Machine Learning}, pages 2805--2814. PMLR, 2018.

\bibitem[Widmann et~al.(2019)Widmann, Lindsten, and Zachariah]{widmann2019calibration}
David Widmann, Fredrik Lindsten, and Dave Zachariah.
\newblock Calibration tests in multi-class classification: A unifying framework.
\newblock \emph{Advances in neural information processing systems}, 32, 2019.

\bibitem[Zhao et~al.(2021)Zhao, Kim, Sahoo, Ma, and Ermon]{zhao2021calibrating}
Shengjia Zhao, Michael Kim, Roshni Sahoo, Tengyu Ma, and Stefano Ermon.
\newblock Calibrating predictions to decisions: A novel approach to multi-class calibration.
\newblock \emph{Advances in Neural Information Processing Systems}, 34:\penalty0 22313--22324, 2021.

\bibitem[Platt et~al.(1999)]{platt1999probabilistic}
John Platt et~al.
\newblock Probabilistic outputs for support vector machines and comparisons to regularized likelihood methods.
\newblock \emph{Advances in large margin classifiers}, 10\penalty0 (3):\penalty0 61--74, 1999.

\bibitem[Kull et~al.(2019)Kull, Perello~Nieto, K{\"a}ngsepp, Silva~Filho, Song, and Flach]{kull2019beyond}
Meelis Kull, Miquel Perello~Nieto, Markus K{\"a}ngsepp, Telmo Silva~Filho, Hao Song, and Peter Flach.
\newblock Beyond temperature scaling: Obtaining well-calibrated multi-class probabilities with dirichlet calibration.
\newblock \emph{Advances in neural information processing systems}, 32, 2019.

\bibitem[Zadrozny and Elkan(2001)]{zadrozny2001obtaining}
Bianca Zadrozny and Charles Elkan.
\newblock Obtaining calibrated probability estimates from decision trees and naive bayesian classifiers.
\newblock In \emph{Icml}, volume~1, pages 609--616, 2001.

\bibitem[Zadrozny and Elkan(2002)]{zadrozny2002transforming}
Bianca Zadrozny and Charles Elkan.
\newblock Transforming classifier scores into accurate multiclass probability estimates.
\newblock In \emph{Proceedings of the eighth ACM SIGKDD international conference on Knowledge discovery and data mining}, pages 694--699, 2002.

\bibitem[Patel et~al.(2020)Patel, Beluch, Yang, Pfeiffer, and Zhang]{patel2020multi}
Kanil Patel, William Beluch, Bin Yang, Michael Pfeiffer, and Dan Zhang.
\newblock Multi-class uncertainty calibration via mutual information maximization-based binning.
\newblock \emph{arXiv preprint arXiv:2006.13092}, 2020.

\bibitem[Rahimi et~al.(2020)Rahimi, Shaban, Cheng, Hartley, and Boots]{rahimi2020intra}
Amir Rahimi, Amirreza Shaban, Ching-An Cheng, Richard Hartley, and Byron Boots.
\newblock Intra order-preserving functions for calibration of multi-class neural networks.
\newblock \emph{Advances in Neural Information Processing Systems}, 33:\penalty0 13456--13467, 2020.

\bibitem[Mukhoti et~al.(2020)Mukhoti, Kulharia, Sanyal, Golodetz, Torr, and Dokania]{mukhoti2020calibrating}
Jishnu Mukhoti, Viveka Kulharia, Amartya Sanyal, Stuart Golodetz, Philip Torr, and Puneet Dokania.
\newblock Calibrating deep neural networks using focal loss.
\newblock \emph{Advances in neural information processing systems}, 33:\penalty0 15288--15299, 2020.

\bibitem[Marx et~al.(2024)Marx, Zalouk, and Ermon]{marx2024calibration}
Charlie Marx, Sofian Zalouk, and Stefano Ermon.
\newblock Calibration by distribution matching: Trainable kernel calibration metrics.
\newblock \emph{Advances in Neural Information Processing Systems}, 36, 2024.

\bibitem[Gupta et~al.(2021)Gupta, Rahimi, Ajanthan, Mensink, Sminchisescu, and Hartley]{gupta2021calibration}
Kartik Gupta, Amir Rahimi, Thalaiyasingam Ajanthan, Thomas Mensink, Cristian Sminchisescu, and Richard Hartley.
\newblock Calibration of neural networks using splines.
\newblock In \emph{International Conference on Learning Representations}, 2021.
\newblock URL \url{https://openreview.net/forum?id=eQe8DEWNN2W}.

\bibitem[Gruber and Buettner(2022)]{gruber2022better}
Sebastian Gruber and Florian Buettner.
\newblock Better uncertainty calibration via proper scores for classification and beyond.
\newblock \emph{Advances in Neural Information Processing Systems}, 35:\penalty0 8618--8632, 2022.

\bibitem[Naeini et~al.(2015)Naeini, Cooper, and Hauskrecht]{naeini2015obtaining}
Mahdi~Pakdaman Naeini, Gregory Cooper, and Milos Hauskrecht.
\newblock Obtaining well calibrated probabilities using bayesian binning.
\newblock In \emph{Proceedings of the AAAI conference on artificial intelligence}, volume~29, 2015.

\bibitem[Nixon et~al.()Nixon, Dusenberry, Zhang, Jerfel, and Tran]{nixon2019measuring}
Jeremy Nixon, Michael~W Dusenberry, Linchuan Zhang, Ghassen Jerfel, and Dustin Tran.
\newblock Measuring calibration in deep learning.

\bibitem[Gupta and Ramdas(2021)]{gupta2021distribution}
Chirag Gupta and Aaditya Ramdas.
\newblock Distribution-free calibration guarantees for histogram binning without sample splitting.
\newblock In \emph{International conference on machine learning}, pages 3942--3952. PMLR, 2021.

\bibitem[Kumar et~al.(2019)Kumar, Liang, and Ma]{kumaar2018VUC}
Ananya Kumar, Percy~S Liang, and Tengyu Ma.
\newblock Verified uncertainty calibration.
\newblock In H.~Wallach, H.~Larochelle, A.~Beygelzimer, F.~d\textquotesingle Alch\'{e}-Buc, E.~Fox, and R.~Garnett, editors, \emph{Advances in Neural Information Processing Systems}, volume~32. Curran Associates, Inc., 2019.
\newblock URL \url{https://proceedings.neurips.cc/paper_files/paper/2019/file/f8c0c968632845cd133308b1a494967f-Paper.pdf}.

\bibitem[Rossellini et~al.(2025)Rossellini, Soloff, Barber, Ren, and Willett]{rossellini2025can}
Raphael Rossellini, Jake~A Soloff, Rina~Foygel Barber, Zhimei Ren, and Rebecca Willett.
\newblock Can a calibration metric be both testable and actionable?
\newblock \emph{arXiv preprint arXiv:2502.19851}, 2025.

\bibitem[Lin et~al.(2023)Lin, Trivedi, and Sun]{lin2023taking}
Zhen Lin, Shubhendu Trivedi, and Jimeng Sun.
\newblock Taking a step back with {KC}al: Multi-class kernel-based calibration for deep neural networks.
\newblock In \emph{International Conference on Learning Representations}, 2023.
\newblock URL \url{https://openreview.net/forum?id=p_jIy5QFB7}.

\bibitem[Gopalan et~al.(2022{\natexlab{a}})Gopalan, Kim, Singhal, and Zhao]{gopalan2022low}
Parikshit Gopalan, Michael~P Kim, Mihir~A Singhal, and Shengjia Zhao.
\newblock Low-degree multicalibration.
\newblock In \emph{Conference on Learning Theory}, pages 3193--3234. PMLR, 2022{\natexlab{a}}.

\bibitem[Noarov et~al.(2025)Noarov, Ramalingam, Roth, and Xie]{noarov2025highdimensional}
Georgy Noarov, Ramya Ramalingam, Aaron Roth, and Stephan Xie.
\newblock High-dimensional prediction for sequential decision making.
\newblock In \emph{Forty-second International Conference on Machine Learning}, 2025.
\newblock URL \url{https://openreview.net/forum?id=uRAgIVnAO6}.

\bibitem[Roth and Shi(2024)]{roth2024forecasting}
Aaron Roth and Mirah Shi.
\newblock Forecasting for swap regret for all downstream agents.
\newblock In \emph{Proceedings of the 25th ACM Conference on Economics and Computation}, pages 466--488, 2024.

\bibitem[Kleinberg et~al.(2023)Kleinberg, Leme, Schneider, and Teng]{kleinberg2023u}
Bobby Kleinberg, Renato~Paes Leme, Jon Schneider, and Yifeng Teng.
\newblock U-calibration: Forecasting for an unknown agent.
\newblock In \emph{The Thirty Sixth Annual Conference on Learning Theory}, pages 5143--5145. PMLR, 2023.

\bibitem[Dwork et~al.(2021)Dwork, Kim, Reingold, Rothblum, and Yona]{dwork2021outcome}
Cynthia Dwork, Michael~P. Kim, Omer Reingold, Guy~N. Rothblum, and Gal Yona.
\newblock Outcome indistinguishability.
\newblock In \emph{Proceedings of the 53rd Annual ACM SIGACT Symposium on Theory of Computing}, STOC 2021, page 1095–1108, New York, NY, USA, 2021. Association for Computing Machinery.
\newblock ISBN 9781450380539.
\newblock \doi{10.1145/3406325.3451064}.
\newblock URL \url{https://doi.org/10.1145/3406325.3451064}.

\bibitem[Gopalan et~al.(2022{\natexlab{b}})Gopalan, Kalai, Reingold, Sharan, and Wieder]{gopalan2022omnipredictors}
Parikshit Gopalan, Adam~Tauman Kalai, Omer Reingold, Vatsal Sharan, and Udi Wieder.
\newblock Omnipredictors.
\newblock In \emph{13th Innovations in Theoretical Computer Science Conference (ITCS 2022)}, pages 79--1. Schloss Dagstuhl--Leibniz-Zentrum f{\"u}r Informatik, 2022{\natexlab{b}}.

\bibitem[Gopalan et~al.(2023{\natexlab{a}})Gopalan, Hu, Kim, Reingold, and Wieder]{gopalan2023loss}
Parikshit Gopalan, Lunjia Hu, Michael~P Kim, Omer Reingold, and Udi Wieder.
\newblock Loss minimization through the lens of outcome indistinguishability.
\newblock In \emph{14th Innovations in Theoretical Computer Science Conference (ITCS 2023)}, pages 60--1. Schloss Dagstuhl--Leibniz-Zentrum f{\"u}r Informatik, 2023{\natexlab{a}}.

\bibitem[Gopalan et~al.(2023{\natexlab{b}})Gopalan, Kim, and Reingold]{gopalan2023swap}
Parikshit Gopalan, Michael Kim, and Omer Reingold.
\newblock Swap agnostic learning, or characterizing omniprediction via multicalibration.
\newblock \emph{Advances in Neural Information Processing Systems}, 36:\penalty0 39936--39956, 2023{\natexlab{b}}.

\bibitem[H{\'e}bert-Johnson et~al.(2018)H{\'e}bert-Johnson, Kim, Reingold, and Rothblum]{hebert2018multicalibration}
Ursula H{\'e}bert-Johnson, Michael Kim, Omer Reingold, and Guy Rothblum.
\newblock Multicalibration: Calibration for the (computationally-identifiable) masses.
\newblock In \emph{International Conference on Machine Learning}, pages 1939--1948. PMLR, 2018.

\bibitem[Deng et~al.(2009)Deng, Dong, Socher, Li, Li, and Fei-Fei]{deng2009imagenet}
Jia Deng, Wei Dong, Richard Socher, Li-Jia Li, Kai Li, and Li~Fei-Fei.
\newblock Imagenet: A large-scale hierarchical image database.
\newblock In \emph{2009 IEEE conference on computer vision and pattern recognition}, pages 248--255. Ieee, 2009.

\bibitem[Dosovitskiy et~al.(2021)Dosovitskiy, Beyer, Kolesnikov, Weissenborn, Zhai, Unterthiner, Dehghani, Minderer, Heigold, Gelly, Uszkoreit, and Houlsby]{dosovitskiy2021an}
Alexey Dosovitskiy, Lucas Beyer, Alexander Kolesnikov, Dirk Weissenborn, Xiaohua Zhai, Thomas Unterthiner, Mostafa Dehghani, Matthias Minderer, Georg Heigold, Sylvain Gelly, Jakob Uszkoreit, and Neil Houlsby.
\newblock An image is worth 16x16 words: Transformers for image recognition at scale.
\newblock In \emph{International Conference on Learning Representations}, 2021.
\newblock URL \url{https://openreview.net/forum?id=YicbFdNTTy}.

\bibitem[Kull et~al.(2017)Kull, Silva~Filho, and Flach]{kull2017beyond}
Meelis Kull, Telmo~M Silva~Filho, and Peter Flach.
\newblock Beyond sigmoids: How to obtain well-calibrated probabilities from binary classifiers with beta calibration.
\newblock \emph{Electronic Journal of Statistics}, 11:\penalty0 5052--5080, 2017.

\bibitem[J{\"a}rvelin and Kek{\"a}l{\"a}inen(2002)]{jarvelin2002cumulated}
Kalervo J{\"a}rvelin and Jaana Kek{\"a}l{\"a}inen.
\newblock Cumulated gain-based evaluation of ir techniques.
\newblock \emph{ACM Transactions on Information Systems (TOIS)}, 20\penalty0 (4):\penalty0 422--446, 2002.

\bibitem[Deng et~al.(2012)Deng, Krause, Berg, and Fei-Fei]{deng2012hedging}
Jia Deng, Jonathan Krause, Alexander~C Berg, and Li~Fei-Fei.
\newblock Hedging your bets: Optimizing accuracy-specificity trade-offs in large scale visual recognition.
\newblock In \emph{2012 IEEE Conference on Computer Vision and Pattern Recognition}, pages 3450--3457. IEEE, 2012.

\bibitem[Bartlett and Mendelson(2002)]{bartlett2002rademacher}
Peter~L Bartlett and Shahar Mendelson.
\newblock Rademacher and gaussian complexities: Risk bounds and structural results.
\newblock \emph{Journal of machine learning research}, 3\penalty0 (Nov):\penalty0 463--482, 2002.

\bibitem[Vapnik and Chervonenkis(1971)]{vapnik1971uniform}
VN~Vapnik and A~Ya Chervonenkis.
\newblock On the uniform convergence of relative frequencies of events to their probabilities.
\newblock \emph{Theory of Probability \& Its Applications}, 16\penalty0 (2):\penalty0 264--280, 1971.

\bibitem[Mohri et~al.(2018)Mohri, Rostamizadeh, and Talwalkar]{mohri2018foundations}
Mehryar Mohri, Afshin Rostamizadeh, and Ameet Talwalkar.
\newblock Foundations of machine learning. adaptive computation and machine learning, 2018.

\bibitem[Maurer(2016)]{maurer2016vector}
Andreas Maurer.
\newblock A vector-contraction inequality for rademacher complexities.
\newblock In \emph{Algorithmic Learning Theory: 27th International Conference, ALT 2016, Bari, Italy, October 19-21, 2016, Proceedings 27}, pages 3--17. Springer, 2016.

\bibitem[Kupavskii(2020)]{kupavskii2020vc}
Andrey Kupavskii.
\newblock The vc-dimension of k-vertex d-polytopes.
\newblock \emph{Combinatorica}, 40\penalty0 (6):\penalty0 869--874, 2020.

\bibitem[Anthony and Bartlett(2009)]{anthony2009neural}
Martin Anthony and Peter~L Bartlett.
\newblock \emph{Neural network learning: Theoretical foundations}.
\newblock cambridge university press, 2009.

\bibitem[Goldberg and Jerrum(1993)]{goldberg1993bounding}
Paul Goldberg and Mark Jerrum.
\newblock Bounding the vapnik-chervonenkis dimension of concept classes parameterized by real numbers.
\newblock In \emph{Proceedings of the sixth annual conference on Computational learning theory}, pages 361--369, 1993.

\bibitem[Massart(1990)]{massart1990tight}
Pascal Massart.
\newblock The tight constant in the dvoretzky-kiefer-wolfowitz inequality.
\newblock \emph{The annals of Probability}, pages 1269--1283, 1990.

\bibitem[Krizhevsky et~al.(2009)Krizhevsky, Hinton, et~al.]{krizhevsky2009learning}
Alex Krizhevsky, Geoffrey Hinton, et~al.
\newblock Learning multiple layers of features from tiny images, 2009.

\bibitem[Zhang et~al.(2015)Zhang, Zhao, and LeCun]{zhang2015character}
Xiang Zhang, Junbo Zhao, and Yann LeCun.
\newblock Character-level convolutional networks for text classification.
\newblock \emph{Advances in neural information processing systems}, 28, 2015.

\bibitem[Chen(2020)]{chenyaofo_pytorch_cifar_models_2020}
Yaofo Chen.
\newblock {pytorch-cifar-models: Pretrained models for CIFAR10/100 in PyTorch}, February 2020.
\newblock URL \url{https://github.com/chenyaofo/pytorch-cifar-models}.

\bibitem[Wightman()]{Wightman_PyTorch_Image_Models}
Ross Wightman.
\newblock {PyTorch Image Models}.
\newblock URL \url{https://github.com/huggingface/pytorch-image-models}.

\bibitem[Salvador(2022)]{Salvador_Calibration_Baselines_2022}
Tiago Salvador.
\newblock Calibration baselines.
\newblock \url{https://github.com/tiagosalvador/calibration-baselines}, 8 2022.
\newblock URL \url{https://github.com/tiagosalvador/calibration-baselines}.
\newblock Last commit August 2022. Accessed: May 22, 2025.

\bibitem[Head et~al.(2018)Head, MechCoder, Louppe, Shcherbatyi, fcharras, Vinícius, cmmalone, Schröder, nel215, Campos, Young, Cereda, Fan, rene rex, Shi, Schwabedal, carlosdanielcsantos, Hvass-Labs, Pak, SoManyUsernamesTaken, Callaway, Estève, Besson, Cherti, Pfannschmidt, Linzberger, Cauet, Gut, Mueller, and Fabisch]{tim_head_2018_1207017}
Tim Head, MechCoder, Gilles Louppe, Iaroslav Shcherbatyi, fcharras, Zé Vinícius, cmmalone, Christopher Schröder, nel215, Nuno Campos, Todd Young, Stefano Cereda, Thomas Fan, rene rex, Kejia~(KJ) Shi, Justus Schwabedal, carlosdanielcsantos, Hvass-Labs, Mikhail Pak, SoManyUsernamesTaken, Fred Callaway, Loïc Estève, Lilian Besson, Mehdi Cherti, Karlson Pfannschmidt, Fabian Linzberger, Christophe Cauet, Anna Gut, Andreas Mueller, and Alexander Fabisch.
\newblock scikit-optimize/scikit-optimize: v0.5.2, March 2018.
\newblock URL \url{https://doi.org/10.5281/zenodo.1207017}.

\bibitem[Blondel et~al.(2021)Blondel, Berthet, Cuturi, Frostig, Hoyer, Llinares-L{\'o}pez, Pedregosa, and Vert]{jaxopt_implicit_diff}
Mathieu Blondel, Quentin Berthet, Marco Cuturi, Roy Frostig, Stephan Hoyer, Felipe Llinares-L{\'o}pez, Fabian Pedregosa, and Jean-Philippe Vert.
\newblock Efficient and modular implicit differentiation.
\newblock \emph{arXiv preprint arXiv:2105.15183}, 2021.

\bibitem[Mikolov et~al.(2013)Mikolov, Chen, Corrado, and Dean]{mikolov2013efficient}
Tomas Mikolov, Kai Chen, Greg Corrado, and Jeffrey Dean.
\newblock Efficient estimation of word representations in vector space.
\newblock \emph{arXiv preprint arXiv:1301.3781}, 2013.

\bibitem[Rehurek and Sojka(2011)]{rehurek2011gensim}
Radim Rehurek and Petr Sojka.
\newblock Gensim--python framework for vector space modelling.
\newblock \emph{NLP Centre, Faculty of Informatics, Masaryk University, Brno, Czech Republic}, 3\penalty0 (2), 2011.

\end{thebibliography}
\clearpage
\appendix
\thispagestyle{empty}

\clearpage

\appendix
\thispagestyle{empty}

\startcontents[app]
\onecolumn

\aistatstitle{Scalable Utility-Aware Multiclass Calibration: \\
Supplementary Materials}

\section*{Organization of the Appendix} 

The appendix is organized as follows: \Cref{app:post_hoc} presents the post-hoc patching algorithm for utility calibration; \Cref{app:deferred_content} collects deferred material, including CutOff calibration (\Cref{app:cutoff_calibration}), bounds for binned estimators (\Cref{app:uc_bound_binarized}), computational aspects of proactive measurability (\Cref{app:proactive_cal_hard}), sample complexity bounds for interactive measurability (\Cref{app:interactive_measure}), and eCDF guarantees (\Cref{app:quantile_gaurantees}). Then, \Cref{app:proof} presents proofs of statements that were explicitly stated in the main body. Finally, \Cref{app:experiments} provides additional experiments.

\begingroup
\setcounter{tocdepth}{2}
\printcontents[app]{l}{1}{\section*{Appendix Contents}}
\endgroup
\newpage

\section{Post-Hoc Patching Algorithm for Utility Calibration} \label{app:post_hoc_weighted}
\label{app:post_hoc}

This section details a post-hoc calibration algorithm aimed at reducing the utility calibration error $\uc(f, \ucal)$ with respect to a specific utility function class $\ucal$. As established in the main text (\cref{eq:UCasWCE}), utility calibration can be cast as a form of weighted calibration. Recall that for a single $u$,
\begin{align*}
    \uc\open{f,u} =  \sup_{w\in\wcal(u)} \mathbb{E}\left[\inner{f(X)-Y}{w(f(X))}\right],
\end{align*}
where $\wcal(u) = \ens{x\rightarrow \xi\vec{u}(X)\one\ens{v_u(X)\in I}|I\in \cinv{-1}{1}, \xi\in \ens{-1,1}}$. This naturally extends to a utility function class $\ucal$ by defining $\wcal(\ucal) = \cup_{u\in \ucal}\wcal\open{u}$. Then it holds that $\uc\open{f,\ucal} = \ce_{\wcal\open{\ucal}}\open{f}$.

Given that utility calibration can be cast as weighted calibration, algorithms from the literature can be used to post-hoc calibrate $f$ such that it has a small utility calibration error.  In particular, \Cref{alg:patching_uc_weighted_cal} iteratively finds the \textit{worst} witness $w \in \wcal\open{\ucal}$ and adjusts the predictor $f$ to correct for this specific violation. This approach is common in (multi)calibration literature, for example \citep{hebert2018multicalibration,gopalan2022low, jung2021moment}, a more general version is presented in \citep[Chapter 15]{duchi2024lecturenotes}. A key property is that these adjustments can be made such that the model's Brier score decreases in every iteration. 

\begin{algorithm}[H]
\caption{Iterative Patching for Utility Calibration}
\label{alg:patching_uc_weighted_cal}
\begin{algorithmic}[1]
    \STATE \textbf{Input:} Initial predictor $f^{(0)}: \xcal \to \Delta^{C-1}$, witness class $\wcal\open{\ucal}$, target tolerance $\eps > 0$. \textbf{Set} $t\leftarrow0$
    \LOOP 
        \STATE Find $w_t \in \argmax_{w\in\wcal\open{\ucal}}\mathbb{E}\left[ \inner{f^{(t)}(X)-Y}{w(f^{(t)}(X))} \right]$.
        \STATE Let $\text{err}_t \leftarrow \mathbb{E}\left[ \inner{f^{(t)}(X)-Y}{w_t(f^{(t)}(X))} \right]$.
        \IF{$\text{err}_t \le \eps$}
            \STATE \textbf{break}
        \ENDIF
        \STATE Choose stepsize $\eta_t$.
        \STATE Update predictor: ${f}^{(t+1)}(X) \leftarrow \pi_{\Delta^{C-1}} \left( {f}^{(t)}(X) - \eta_t w_t(f^{(t)}(X)) \right)$.
        \STATE $t \leftarrow t + 1$.
    \ENDLOOP
    \STATE \textbf{Return} $f^{(t)}$.
\end{algorithmic}
\end{algorithm}
Here $\pi_{\tri^{C-1}}$ is the projection onto the simplex and $\eta_t$ is a possibly adaptive stepsize.
\begin{proposition}[Convergence and Brier Score Guarantee] \label{prop:convergence_patching_weighted_revised}
Assume oracle access to compute $\mathrm{err}_t$ and the corresponding witness $w_t \in \wcal\open{\ucal}$ at each iteration $t$. Let $C$ be the number of classes. With the stepsize $\eta_t = \mathrm{err}_t / C$, Algorithm \ref{alg:patching_uc_weighted_cal} terminates in $T = O(C/\eps^2)$ iterations 
\end{proposition}
\begin{proof}
We include the proof for completeness. It follows the approach of \cite{hebert2018multicalibration} in using the Brier score as a potential function and showing that it monotonically decreases across iterations. It is the same as the version in \citet{gopalan2024computationally}. A more general version can be found in \citet[Chapter 15]{duchi2024lecturenotes}.
The change in Brier score $L(f) = \mathbb{E}[\|Y-f(X)\|_2^2]$ from $f^{(t)}$ to $f^{(t+1)}$ is:
\begin{align*}
L(f^{(t+1)}) - L(f^{(t)}) &\le \mathbb{E}\left[ \|Y - (f^{(t)}(X) - \eta_t w_t(f^{(t)}(X)))\|_2^2 - \|Y - f^{(t)}(X)\|_2^2 \right] \\
&= \mathbb{E}\left[ 2\eta_t \inner{Y - f^{(t)}(X)}{w_t(f^{(t)}(X))} + \eta_t^2 \|w_t(f^{(t)}(X))\|_2^2 \right] \\
&= \mathbb{E}\left[ -2\eta_t \inner{f^{(t)}(X)-Y}{w_t(f^{(t)}(X))} + \eta_t^2 \|w_t(f^{(t)}(X))\|_2^2 \right] \\
&= -2\eta_t \text{err}_t + \eta_t^2 \mathbb{E}[\|w_t(f^{(t)}(X))\|_2^2].
\end{align*}
Each witness $w \in \wcal(\ucal)$ is of the form $p \mapsto \xi \vec{u}(p)\one\{v_u(p) \in I\}$ for $\xi\in\ens{-1,1}$ and $u\in \ucal$. Since $u(\cdot, e_j) \in [-1,1]$, $\|\vec{u}(p)\|_2^2 \leq \norm{\vec{u}(p)}_\infty\norm{\vec{u}(p)}_1\leq C$. Thus, $\|w_t(p)\|_2^2 \le C$. Substituting $\eta_t = \text{err}_t / C$:
\begin{align*}
L(f^{(t+1)}) - L(f^{(t)}) &\le -2 \frac{\text{err}_t}{C} \text{err}_t + \left(\frac{\text{err}_t}{C}\right)^2 \mathbb{E}[\|w_t(f^{(t)}(X))\|_2^2]\\
&\le -\frac{2\text{err}_t^2}{C} + \frac{\text{err}_t^2 C}{C^2} = -\frac{\text{err}_t^2}{C}.
\end{align*}
This proves that the Brier score does not increase and strictly decreases if $\text{err}_t > 0$. If the algorithm continues, it is because $\text{err}_t > \eps$, so the decrease is at least $\eps^2/C$. Since $L(f) \in [0,2]$ (as $\|Y-f(X)\|_2^2 \le \|Y\|_2^2+\|f(X)\|_2^2 \le 1+1=2$), and at step $\text{err}_t > \eps$ decreases $L(f)$ by at least $\eps^2/C$, the algorithm must terminate in at most $O\open{C/\eps^2}$ such steps.
\end{proof}

\paragraph{Empirical Implementation and Sample Complexity.}
In practice, direct computation of expectations within \Cref{alg:patching_uc_weighted_cal} is infeasible. Instead, the algorithm is implemented using a dataset $S^t$ of $N$ i.i.d. samples at each iteration $t$. Both the maximization step to find the witness $w_t$ and the computation of the error $\text{err}_t$ are performed using empirical averages $\hat{\mathbb{E}}_N[\cdot]$ over $S^t$.

The theoretical convergence guarantees of \Cref{prop:convergence_patching_weighted_revised} (i.e., termination in $O(C/\eps^2)$ iterations) can be extended to this empirical setting, provided that the empirically estimated error $\widehat{\text{err}}_t \coloneqq \hat{\mathbb{E}}_N[\inner{f^{(t)}(X)-Y}{w_t(X)}]$ is a sufficiently accurate approximation of the true error $\text{err}_t$. Specifically, if $\widehat{\text{err}}_t$ is within $O(\eps)$ of $\text{err}_t$ whenever $\text{err}_t > \eps$, the iteration complexity remains $O(C/\eps^2)$.

The number of samples $N$ required per iteration to achieve an $O(\eps)$-accurate estimation of $\text{err}_t$, with probability $1-\delta$, depends on the complexity of the witness class $\wcal(\ucal)$. For the Top-Class utility $\ucal_{\tce}$, where $|\mathcal{U}|=1$, using \Cref{lem:estimate_uc_single}, $N\leq \tilde{O}\open{\frac{\log(1/\delta)}{\eps^2}}$. Similarly, Class-Wise utility $\ucal_{\cwe}$ and Top-K utility $\ucal_{\topk}$, where for both $|\mathcal{U}|=C$, using a simple union bound, we recover $N\leq \tilde{O}\open{\frac{\log(C/\delta)}{\eps^2}}$.

\section{Deferred Content}
\label{app:deferred_content}
\subsection{CutOff Calibration}\label{app:cutoff_calibration}

In \Cref{sec:utility_calibration}, we highlighted that our utility calibration framework, particularly its focus on worst-case interval-based deviations of predicted utility, can be seen as a natural extension of the binary CutOff calibration concept to multiclass scenarios and general utility functions. This extension preserves important decision-theoretic properties. For the binary setting ($Y \in \{0,1\}, f:\xcal\rightarrow[0,1]$), \citet{rossellini2025can} demonstrate that if the metric
\begin{equation*}
 \cutoff(f) \coloneqq \sup_{I \in \cinv{0}{1}} |\mathbb E[(Y-f(X))\one\ens{f(X) \in I}]|
\end{equation*}
is small, then a simple decision rule $\widehat{Y}_\tau:X\rightarrow\ens{0,1}$ of the form $\hat{Y}_\tau = \one\ens{f(X) \ge \tau}$ evaluated against its associated binary decision loss, cannot be substantially improved by monotonic post-hoc calibration. More concretely, let $R_\mathrm{bd}(g; \tau) := \mathbb E[\ell_\mathrm{bd}(Y, \hat{Y}_\tau; \tau)]$ be the risk under the binary decision loss $\ell_\mathrm{bd}(Y, \widehat{Y}; \tau) = \tau (1-Y)\widehat{Y} + (1-\tau) Y(1-\widehat{Y})$. Then, \citet[Prop. 3.2]{rossellini2025can} show that for any $\tau \in [0,1]$:
\begin{equation} \label{eq:cutoff_risk_gap_rossellini_formal}
\begin{aligned}
    R_\mathrm{bd}(f; \tau)
    - \inf_{\substack{h: [0,1] \rightarrow [0,1]\\ \textnormal{monotone}}}
      R_\mathrm{bd}(h\circ f; \tau)
    \;\leq\; 2 \, \cutoff(f).
\end{aligned}
\end{equation}
This guarantee implies that if $\cutoff(f)$ is small, the decision-maker, thresholding $f(X)$ to make a binary decision, gains little by applying any monotonic recalibration to $f(X)$. Similarly, they showed that in the binary case, CutOff calibration error can be used to bound distance from calibration. As such, \Cref{prop:uc_risk_gap_main,prop:uc_bounds_dcu_main} extend the results of \citet{rossellini2025can} to the multiclass setting. 

\subsection{Bounding Binned Estimators using Utility Calibration}\label{app:uc_bound_binarized}

To illustrate the relationship between binned estimations of calibration error and utility calibration, let $p_X \coloneqq f(X)_{\gamma(f(X))}$ and correctness indicator $Y_X \coloneqq \one\{Y=e_{\gamma(f(X))}\}$, definitions for $m$ bins $(B_j)_{j \in [m]}$ are:
\begin{align*}
\tce^{\mathrm{bin}}(f) &= \sum_{j=1}^m \left| \mathbb{E}\left[\left( p_X - Y_X \right) \one\{p_X \in B_j\}\right] \right|, \\
\uc(f, \ucal_{\tce}) &= \sup_{I \in \cinv{0}{1}} \left| \mathbb{E}\left[\left( Y_X - p_X \right) \one\{p_X \in I\}\right] \right|.
\end{align*}
Each binned term $\left| \expec\left[(p_X - Y_X) \one\{p_X \in B_j\}\right] \right| \le \uc(f, \ucal_{\tce})$ (by setting $I=B_j$), thus $\tce^{\mathrm{bin}}(f) \le m  \uc(f, \ucal_{\tce})$.  Conversely, small binned errors do not imply small utility calibration, as binned errors can cancel within bins, while utility calibration is the supremum over intervals.

For instance, let $p_X$ be $0.45$ or $0.55$ (each with probability $0.5$), with $\expec[Y_X | p_X=0.45] = 0.05$ and $\expec[Y_X | p_X=0.55] = 0.95$. If $\tce^{\mathrm{bin}}(f)$ uses bins $B_1=[0,1/3), B_2 = [1/3, 2/3), B_3=[2/3,1]$, the expected error is $0$. Thus, $\tce^{\mathrm{bin}}(f) = 0$.

However, the inverse is not true. For example, for $\uc(f, \ucal_{\tce})$, consider the interval $I_1 = [0.45, 0.46]$. The term $\left| \expec\left[(Y_X - p_X) \one\{p_X \in I_1\}\right] \right|$ becomes $\left| P(p_X=0.45)  (\expec[Y_X|p_X=0.45] - 0.45) \right| = 0.2$. Since $\uc(f, \ucal_{\tce})$ is the supremum over such intervals, $\uc(f, \ucal_{\tce}) \ge 0.2$. The binned estimator indicates perfect calibration, while $\uc(f, \ucal_{\tce})$ does not.

\subsection{Hardness of Proactive Measurability}\label{app:proactive_cal_hard}

Proactive measurability for a utility class $\ucal$, as defined in \Cref{def:proactive}, necessitates an algorithm to efficiently find $\hat{u} \in \ucal$, whose utility calibration error $\uc(f, \hat{u})$ approximates $\sup_{u \in \ucal} \uc(f, u)$. This is equivalent to efficiently finding an approximate worst-case function from the witness class $\wcal\open{\ucal} = \bigcup_{u \in \ucal} \{ X \mapsto \xi \vec{u}(X)\one\ens{v_u(X)\in I} \mid I \in \cinv{-1}{1}, \xi \in \{-1,1\} \}$, given that the worst-case interval for a fixed $u$ is efficiently findable (\Cref{lem:estimate_uc_single}).
The work of \citet{gopalan2024computationally} establishes computational hardness for ``auditing with a witness" for related, expressive classes of witness functions.

\begin{definition}[Auditing with a witness \citep{gopalan2024computationally}]
\label{def:audit-guarantee-gopalan}
    An $(\alpha, \beta)$ auditor for a witness class $\mathcal{W}_{\text{target}}$ is an algorithm that, when given access to a distribution $\mathcal{D}$ where $\ce_{\mathcal{W}_{\text{target}}}(\mathcal{D}) > \alpha$, returns any function $w':\tri^{C-1}\to [-1,1]^C$ such that
   $\expec_{(\xcal,\ycal)\sim \mathcal{D}} [\inner{Y - f(X)}{w'(f(X))}] \geq \beta.$
\end{definition}
First, auditing with a witness is an \textit{easier task than proactive measurability}, as it allows returning any function $w'$, not necessarily from the original witness class. Thus, if auditing is hard for a class $\mathcal{W}_{\text{target}}$, and if our class $\wcal(\ucal)$ is at least as expressive as $\mathcal{W}_{\text{target}}$, then proactive measurability for $\ucal$ is also computationally hard.

\citet{gopalan2024computationally} demonstrate hardness for two key notions. 
\begin{enumerate}[leftmargin=*]
    \item 
First, for decision calibration, their witness class $\mathcal{W}_{\text{dec}}$ involves partitioning $\tri^{C-1}$ using hyperplanes, i.e.  $\mathcal{W}_{\text{dec}}\coloneqq\ens{x\rightarrow g'\one\ens{a^Tx\geq b}+g\one\ens{a^Tx<b}|g',g,a\in \rset^C, b\in \rset, \text{ s.t. } \norm{g'}_2,\norm{g}_2\leq 1}$. Auditing for $\mathcal{W}_{\text{dec}}$ is shown to be computationally hard under standard assumptions (\citet[Thm. 5.1]{gopalan2024computationally}), i.e.~cannot be performed in polynomial time for non-trivial $\alpha$ under standard computational complexity assumptions. Up to a scaling, this is a slightly more general class than~$\ucal_\lin$. 

\item Second, \citet{gopalan2024computationally} introduced projected smooth calibration, which is very similar to our notion of utility calibration for $\ucal_\lin$, but replaced the hard interval indicator with Lipschitz functions. Again, auditing for this notion was also proven computationally hard, in the sense that no auditing algorithm can be polynomial in ${1}/{\alpha}$ \citep[Thm. 8.1]{gopalan2024computationally}.
\end{enumerate}

\subsection{Interactive Measurability}\label{app:interactive_measure}

In this section, we establish conditions under which the utility calibration error $\uc(f, \ucal)$ is interactively measurable, completing \Cref{sec:measure_uc}.  This involves bounding the sample complexity required for the empirical estimates $\widehat{\uc}(f,u; S)$ to uniformly converge to their true values $\uc(f,u)$ over all $u \in \ucal$. Throughout this section, we assume that $\hat{\expec}_n$ denotes the empirical expectation over $n$ i.i.d. samples.

First, we start by bounding the Rademacher complexity of the class of functions 
\begin{equation} \label{eq:defGu}
\mathcal{G}_{\ucal}
\coloneqq \Big\{ (X,Y) \mapsto \inner{Y-f(X)}{\vec{u}(X)\one\ens{v_u(X)\in I}} 
\;\Big|\; u \in \ucal,\ I \in \cinv{-1}{1} \Big\}
\end{equation}
with respect to its coordinate-wise components $\pcal_1,\ldots, \pcal_C$, where 
\begin{equation}\label{eq:defPj}
    \mathcal{P}_j(\ucal) \coloneqq \left\{ X \mapsto u(f(X),e_j)\one\ens{v_u(X)\in I} \;\middle|\; u \in \ucal, I \in \cinv{-1}{1} \right\}.
\end{equation}

This generic bound may be transformed into a sample complexity bound for interactive measurability by \Cref{cor:interactive_from_rad}. As an example, we apply it to the class of linear utility functions $\ucal_\lin$. Nonetheless, in \Cref{app:lin-rank-tight}, we derive tighter bounds for $\ucal_\lin$ and $\ucal_\rank$ by directly bounding their pseudo-dimension. Before proceeding, we first recall standard definitions and results.

\begin{definition}\citep[Rademacher Complexity]{bartlett2002rademacher}  Let $\fcal$ be a class of real-valued functions $h: \mathcal{Z} \to \rset$. Given $n$ samples $S_Z = (Z_1, \dots, Z_n)$ where $Z_i \sim D_Z$, the empirical Rademacher complexity of $\fcal$ given $S_Z$ is
    \begin{equation*}
        \hat{\mathfrak{R}}_n(\fcal|S_Z) = \mathbb{E}_{\bm{\sigma}} \left[\sup_{h \in \fcal} \frac{1}{n} \sum_{i=1}^n \sigma_i h(Z_i)\right],
    \end{equation*}
    where $\sigma_i$ are i.i.d. Rademacher random variables. In addition, the expected Rademacher complexity is defined as 
    $$\mathfrak{R}_n(\fcal) = \mathbb{E}_{S_Z} [\hat{\mathfrak{R}}_n(\fcal|S_Z)].$$
\end{definition}

\begin{definition}\label{def:VC}\citep[VC dimension]{vapnik1971uniform} 
Let $\Hcal$ be a class of binary-valued functions on a domain $\mathcal{Z}$. We say that $\Hcal$ shatters a set $\{z_1,\ldots,z_m\}\subseteq\mathcal{Z}$ if for every labeling $b\in\{0,1\}^m$ there exists $h\in\Hcal$ such that $h(z_i)=b_i$ for all $i$. The VC dimension $\VC(\Hcal)$ is the largest $m$ for which some set of size $m$ is shattered (or $\infty$ if no such largest $m$ exists).
\end{definition}

Finally, we list a common result combining Massart's lemma, which bounds the Rademacher complexity using the growth function \citep[Theorem 3.7]{mohri2018foundations} and Sauer's lemma, which bounds the growth function \citep[Theorem 3.17]{mohri2018foundations}.

\begin{result}\label{res:sauer_massart}Consider a boolean class of functions $\bcal$ such that $d=VC(\bcal)$. Let $d>n$, then it holds that 
\begin{equation*}
   \rad_n(\bcal)\leq \sqrt{\frac{2d\log(en/d)}{n}}.
\end{equation*}
\end{result}

\begin{thm}[Rademacher Complexity Bound for $\mathcal{G}_{\ucal}$]\label{thm:rad_bound_g_ucal_revised}
For a utility class $\ucal$,
the Rademacher complexity of $\mathcal{G}_{\ucal}$ is bounded as:
\begin{equation*}
    \mathfrak{R}_n(\mathcal{G}_{\ucal}) \le 2 \sum_{j=1}^C \mathfrak{R}_n(\mathcal{P}_j(\ucal)).
\end{equation*}
\end{thm}
\begin{proof}
The class $\mathcal{G}_{\ucal}$ consists of functions $g_{u,I}(X,Y) = \inner{Y-f(X)}{\vec{u}(X)\one\ens{v_u(X)\in I}}$.
We consider the empirical Rademacher complexity $\hat{\mathfrak{R}}_n(\mathcal{G}_{\ucal}|S_{XY})$ for $n$ samples $S_{XY} = \{(X_k,Y_k)\}_{k=1}^n$, which is by definition:
\begin{equation*}
    \hat{\mathfrak{R}}_n(\mathcal{G}_{\ucal}|S_{XY}) = \mathbb{E}_{\bm{\sigma}'} \left[\sup_{\substack{u \in \ucal \\ I \in \cinv{-1}{1}}} \frac{1}{n} \sum_{k=1}^n \sigma'_k \inner{Y_k-f(X_k)}{\vec{u}(f(X_k))\one\ens{v_u(X_k)\in I}}\right],
\end{equation*}
where $\sigma'_k$ are i.i.d.~scalar Rademacher random variables.
For each $k\in [n]$, let $W_k = Y_k - f(X_k)$. The function $\phi_{W_k}: \mathbb{R}^C \to \mathbb{R}$ defined by $\phi_{W_k}(z) = \inner{W_k}{z}$ is $L$-Lipschitz and $\phi_{W_k}(\mathbf{0})=0$. Specifically, for the $\ell_2$-norm on $\mathbb{R}^C$, the Lipschitz constant $L_k = \|W_k\|_2 = \|Y_k - f(X_k)\|_2\le \sqrt{2}$. 
Let $\vec{\mathcal{W}}_{\ucal}$ be the class of vector-valued functions from $\mathcal{X}$ to $\mathbb{R}^C$:
\begin{equation} \label{eq:defWvecU}
\begin{aligned}
\vec{\mathcal{W}}_{\ucal}
\coloneqq \Big\{ X \mapsto \vec{u}(X)\one\ens{v_u(X)\in I} 
\;\Big|\; u \in \ucal,\ I \in \cinv{-1}{1} \Big\}.
\end{aligned}
\end{equation}
The functions $w \in \vec{\mathcal{W}}_{\ucal}$ map to $[-1,1]^C$.
Using \citet[Corollary 4]{maurer2016vector}, we have
\begin{equation*}
    \hat{\mathfrak{R}}_n(\mathcal{G}_{\ucal}|S_{XY}) \le \sqrt{2} L  \hat{\mathfrak{R}}_n^{\text{vec}}(\vec{\mathcal{W}}_{\ucal}|S_X),
\end{equation*}
where $\hat{\mathfrak{R}}_n^{\text{vec}}(\vec{\mathcal{W}}_{\ucal}|S_X)$ is the empirical vector Rademacher complexity of $\vec{\mathcal{W}}_{\ucal}$ given samples $S_X = \{X_k\}_{k=1}^n$. It is defined as:
\begin{equation} \label{eq:defRadvec}
\begin{aligned}
 \hat{\mathfrak{R}}_n^{\text{vec}}(\vec{\mathcal{W}}_{\ucal}|S_X)
 &\coloneqq \mathbb{E}_{\bm{\sigma}} \Bigg[\sup_{w \in \vec{\mathcal{W}}_{\ucal}} \frac{1}{n} \sum_{k=1}^n \inner{\bm{\sigma}_k}{w(X_k)} \Bigg].
\end{aligned}
\end{equation}
where $\bm{\sigma}_k \in \{-1,1\}^C$ are vectors whose components $\sigma_{kj}$ are i.i.d.~Rademacher random variables.
Substituting $L=\sqrt{2}$:
\begin{equation*}
    \hat{\mathfrak{R}}_n(\mathcal{G}_{\ucal}|S_{XY}) \le \sqrt{2}  \sqrt{2}  \hat{\mathfrak{R}}_n^{\text{vec}}(\vec{\mathcal{W}}_{\ucal}|S_X) = 2 \hat{\mathfrak{R}}_n^{\text{vec}}(\vec{\mathcal{W}}_{\ucal}|S_X).
\end{equation*}
The vector Rademacher complexity $\hat{\mathfrak{R}}_n^{\text{vec}}(\vec{\mathcal{W}}_{\ucal}|S_X)$ can be then bounded as:
\begin{align*}
    \hat{\mathfrak{R}}_n^{\text{vec}}(\vec{\mathcal{W}}_{\ucal}|S_X) &= \mathbb{E}_{\bm{\sigma}} \left[\sup_{\substack{u \in \ucal \\ I \in \cinv{-1}{1}}} \frac{1}{n} \sum_{k=1}^n \sum_{j=1}^C \sigma_{kj} u(f(X_k),e_j)\one\ens{v_u(X_k)\in I} \right]  \text{ \ \ by~ eq.\eqref{eq:defWvecU} and \eqref{eq:defRadvec}} \\
    &\le \mathbb{E}_{\bm{\sigma}} \left[\sum_{j=1}^C \sup_{\substack{u \in \ucal \\ I \in \cinv{-1}{1}}} \frac{1}{n} \sum_{k=1}^n \sigma_{kj} u(f(X_k),e_j)\one\ens{v_u(X_k)\in I} \right] \quad (\text{subadditivity of sup}) \\
    &= \sum_{j=1}^C \mathbb{E}_{\bm{\sigma}_{\cdot j}} \left[\sup_{\substack{u \in \ucal \\ I \in \cinv{-1}{1}}} \frac{1}{n} \sum_{k=1}^n \sigma_{kj} u(f(X_k),e_j)\one\ens{v_u(X_k)\in I} \right] \\
    &= \sum_{j=1}^C \hat{\mathfrak{R}}_n(\mathcal{P}_j(\ucal)|S_X).
\end{align*}
Here $\bm{\sigma}_{\cdot j}$ denotes the $j$-th column of the matrix $\bm{\sigma} = ((\sigma_{kj})_{k\in [n], j\in [C]})$.
Combining these, we get $\hat{\mathfrak{R}}_n(\mathcal{G}_{\ucal}|S_{XY}) \le 2 \sum_{j=1}^C \hat{\mathfrak{R}}_n(\mathcal{P}_j(\ucal)|S_X)$.
Taking expectation over $S_{XY}$ yields the theorem statement.
\end{proof}

\begin{corollary}[Interactive Measurability from Rademacher Bound]\label{cor:interactive_from_rad}
Let $\ucal$ be a class of utility functions and $\mathcal{G}_{\ucal}$ be the function class defined in \cref{eq:defGu}. With probability at least $1-\delta$ over the draw of~$S \sim D^n$:
\begin{equation*}
    \sup_{u \in \ucal} |\widehat{\uc}(f,u; S) - \uc(f, u)| \le 2 \mathfrak{R}_n(\mathcal{G}_{\ucal}) + 4\sqrt{\frac{\log(2/\delta)}{2n}}.
\end{equation*}
Combined with \Cref{thm:rad_bound_g_ucal_revised}, this implies:
\begin{equation*}
    \sup_{u \in \ucal} |\widehat{\uc}(f,u; S) - \uc(f, u)| \le 4 \sum_{j=1}^C \mathfrak{R}_n(\mathcal{P}_j(\ucal)) + 4\sqrt{\frac{\log(2/\delta)}{2n}}.
\end{equation*}
\end{corollary}
\begin{proof}
For $u\in\ucal$ and $I\in \cinv{-1}{1}$, define $g_{u,I}(X,Y)\coloneqq \inner{Y-f(X)}{\vec{u}(X)\one\ens{v_u(X)\in I}}$. With  $\widehat{\uc}(f,u;S)$ defined in \Cref{lem:estimate_uc_single}, it holds by definition that
\begin{align*}
\sup_{u \in \ucal} |\widehat{\uc}(f,u;S) - \uc(f,u)| &= \sup_{u \in \ucal} \left| \sup_{I \in \cinv{-1}{1}} |\hat{\mathbb{E}}_n[g_{u,I}(X,Y)]| - \sup_{I \in \cinv{-1}{1}} |\mathbb{E}[g_{u,I}(X,Y)]| \right| \\
&\le \sup_{u \in \ucal} \sup_{I \in \cinv{-1}{1}} \left| |\hat{\mathbb{E}}_n[g_{u,I}(X,Y)]| - |\mathbb{E}[g_{u,I}(X,Y)]| \right| \\
&\le \sup_{u \in \ucal} \sup_{I \in \cinv{-1}{1}} \left| \hat{\mathbb{E}}_n[g_{u,I}(X,Y)] - \mathbb{E}[g_{u,I}(X,Y)] \right| \\
&= \sup_{g \in \mathcal{G}_{\ucal}} \left| \hat{\mathbb{E}}_n[g(X,Y)] - \mathbb{E}[g(X,Y)] \right|.
\end{align*}
The first inequality is by {$\sup_x f(x)- \sup_x g(x) \le \sup_x |(f-g)(x)|$}. 
Then using \cite[Theorem 3.3]{mohri2018foundations} (standard symmetrization argument using Rademacher complexity and application of McDiarmid inequality), with probability at least $1-\delta$, it holds that 

\begin{equation*}
\begin{aligned}
    \sup_{g \in \mathcal{G}_{\ucal}} \Big| \hat{\mathbb{E}}_n[g(X,Y)] - \mathbb{E}[g(X,Y)] \Big|
    \;\leq\; 2\, \mathfrak{R}_n(\mathcal{G}_{\ucal}) + 4 \sqrt{\tfrac{\log(2/\delta)}{2n}}\ .
\end{aligned}
\end{equation*}
The corollary follows by substituting the bound for $\mathfrak{R}_n(\mathcal{G}_{\ucal})$ from \Cref{thm:rad_bound_g_ucal_revised} into the equation above. 
\end{proof}

We now apply the statements established above to the case of linear utilities (\Cref{ex:util_linear}).

\begin{corollary}[Interactive Measurability of $\ucal_{\mathrm{lin}}$]\label{cor:interactive_lin_rademacher_refined_final}
The utility calibration error is interactively measurable for the class of linear utilities $\ucal_{\mathrm{lin}}$ (\Cref{ex:util_linear}). The sample complexity is $N = O\left(\frac{C^3\log(n/C) + \log(1/\delta)}{\eps^2}\right)$.
\end{corollary}
\begin{proof}
For $u_a \in \ucal_{\mathrm{lin}}$, we have, by definition, $u_a(f(X), e_j) = a_j$ (see \Cref{ex:util_linear}), where $a \in [-1,1]^C$.
The predicted utility is $v_{u_a}(X) = \inner{f(X)}{a}$.
The component classes $\mathcal{P}_j(\ucal_{\mathrm{lin}})$ are:
\begin{equation*}
\begin{aligned}
\mathcal{P}_j(\ucal_{\mathrm{lin}})
 = \Big\{ X \mapsto a_j \one\{\inner{f(X)}{a} \in I\}
 \;\Big|\; a \in [-1,1]^C,\ I \in \cinv{-1}{1} \Big\}.
\end{aligned}
\end{equation*}

Each function in $\mathcal{P}_j(\ucal_{\mathrm{lin}})$ is a product of $a_j$ and an indicator function $\one\{\inner{f(X)}{a} \in I\}$.
First, we consider the class of functions $\mathcal{H}_{a,I} = \{X \mapsto \one\{\inner{f(X)}{a} \in I\}|a\in [-1,1]^C, I \in \cinv{-1}{1}\}$. $\hcal_{a,I}$ is a subclass of the indicator functions of polytopes with two supporting hyperplanes. Thus, $\mathrm{VC}\open{\hcal_{a,I}} = O\open{C}$ \citep[Theorem 1]{kupavskii2020vc} and $\rad_n\open{\hcal_{a,I}} = O\left(\sqrt{\frac{C\log(n/C)}{n}}\right)$. We now proceed to bound $\rad_n\open{\pcal_j\open{\ucal_\lin}}$. 
Let $S_X = \{X_1, \dots, X_n\}$ be $n$ i.i.d. samples. The empirical Rademacher complexity of $\mathcal{P}_j(\ucal_{\mathrm{lin}})$ is:
\begin{align*}
    \hat{\mathfrak{R}}_n(\mathcal{P}_j(\ucal_{\mathrm{lin}})|S_X) &= \mathbb{E}_{\bm{\sigma}} \left[\sup_{\substack{a \in [-1,1]^C \\ I \in \cinv{-1}{1}}} \frac{1}{n} \sum_{k=1}^n \sigma_k a_j \one\{\inner{f(X_k)}{a} \in I\} \right] \\
    &\le \mathbb{E}_{\bm{\sigma}} \left[\sup_{\substack{a \in [-1,1]^C \\ I \in \cinv{-1}{1}}} |a_j| \left| \frac{1}{n} \sum_{k=1}^n \sigma_k \one\{\inner{f(X_k)}{a} \in I\} \right| \right] \\
    &\le \mathbb{E}_{\bm{\sigma}} \left[\sup_{\substack{a \in [-1,1]^C \\ I \in \cinv{-1}{1}}} \left| \frac{1}{n} \sum_{k=1}^n \sigma_k \one\{\inner{f(X_k)}{a} \in I\} \right| \right] \quad (\text{since } |a_j| \le 1) \\
    &= \mathbb{E}_{\bm{\sigma}} \left[\sup_{\substack{h \in \mathcal{H}_{a,I}\\ \xi\in\ens{-1,1}}} 
     \frac{1}{n} \sum_{k=1}^n\xi  \sigma_k h(X_k) \right].\\
    &\leq \mathbb{E}_{\bm{\sigma}} \left[\sup_{h \in \mathcal{H}_{a,I}} 
    \frac{1}{n} \sum_{k=1}^n \sigma_k h(X_k) \right] + \mathbb{E}_{\bm{\sigma}} \left[\sup_{h \in \mathcal{H}_{a,I}} 
     \frac{1}{n} \sum_{k=1}^n - \sigma_k h(X_k) \right] \\
     &\leq 2\hat{\rad}_n\open{\hcal_{a,I}|S_X}.
\end{align*}
Thus, taking expectations over $S_X$:
\begin{equation*}
    \mathfrak{R}_n(\mathcal{P}_j(\ucal_{\mathrm{lin}})) \le 2 \mathfrak{R}_n(\mathcal{H}_{a,I}) = O\left(\sqrt{\frac{C\log(n/C)}{n}}\right).
\end{equation*}
Using \Cref{cor:interactive_from_rad}, the uniform error bound is:
\begin{align*}
    \sup_{u \in \ucal_{\mathrm{lin}}} |\widehat{\uc}(f,u; S) - \uc(f, u)| &\le 4 \sum_{j=1}^C \mathfrak{R}_n(\mathcal{P}_j(\ucal_{\mathrm{lin}})) + 4\sqrt{\frac{\log(2/\delta)}{2n}} \\
    &= O\left(\sqrt{\frac{C^3\log(n/C)}{n} + \frac{\log(1/\delta)}{n}}\right).
\end{align*}
\end{proof}

\subsubsection{Sharper interactive measurability for \texorpdfstring{$\ucal_{\lin}$}{U\_lin} and \texorpdfstring{$\ucal_{\rank}$}{U\_rank}}
\label{app:lin-rank-tight}

The generic Rademacher approach of \Cref{thm:rad_bound_g_ucal_revised,cor:interactive_from_rad} applies broadly. For linear utilities (\Cref{ex:util_linear}) and rank-based utilities (\Cref{ex:util_rank}), we obtain sharper bounds by analyzing the pseudo-dimension of the underlying real-valued classes, yielding an $O(C\log C)$ upper bound with a log-matching $\Omega(C)$ lower bound. We first recall the definition of pseudo-dimension and some standard results.

\begin{definition}\citep[pseudo-dimension and Subgraph]{anthony2009neural}
For a real-valued class $\fcal$ of functions on a domain $\mathcal{Z}$, its subgraph class is
\[
\Sub(\fcal)\;=\;\big\{\ ((z),r)\mapsto \one\{r\le g(z)\}\ :\ g\in\fcal\ \big\}.
\]
The pseudo-dimension $\Pdim(\fcal)$ is the VC dimension of $\Sub(\fcal)$.
\end{definition}

\begin{result}\citep[Simplified Version of Theorem 2.2]{goldberg1993bounding}\label{prop:semi-alg-vc}
Let $\{h_\theta\}_{\theta\in\R^m}$ be a family of Boolean functions. Suppose that each function $h_\theta$ is decided by a Boolean formula of constant size whose atomic predicates are linear inequalities in $\theta$. Then, it holds that
\[
\VC\big(\{h_\theta: \theta\in\R^m\}\big)\ = O\open{m\log m}.
\]
\end{result}

\begin{result}\label{prop:pdim-uniform}\citep[Theorem 10.6]{mohri2018foundations}
    Let $\fcal$ be a class of functions $h: \mathcal{Z}\to\rset$ with range in $[-B,B]$ and $\Pdim(\fcal)=d$. Then for any $\delta\in(0,1)$, with probability at least $1-\delta$ over $n$ i.i.d. samples,
    \[
    \sup_{h\in\fcal}\Big|\widehat{\E}_n[h]-\E[h]\Big|\;\le\; O\open{ B \frac{\sqrt{2d\log(n/d)}+ \sqrt{\log(1/\delta)}}{\sqrt{n}}}.
    \] 
\end{result}

We now give upper and lower bounds on the pseudo-dimension for the linear-utility class and then deduce the resulting sample complexity for interactive measurability. Rank-based utilities will follow via a composition corollary at the end of this section.

\begin{proposition}[pseudo-dimension for $\ucal_{\lin}$]\label{prop:lin-pdim}
$\Pdim(\Gcal_{\lin})=\Omega(C)$ and  $\Pdim(\Gcal_{\lin})= O\open{C\log C}$, where $\Gcal_{\lin}$ is the class of functions induced by $\ucal_{\lin}$ as in \Cref{eq:defGu}.
\end{proposition}

\begin{proof}
For $a\in[-1,1]^C$ and $[s,t]\subseteq[-1,1]$, define
\begin{equation*}
    g_{a,s,t}(X,Y)\;=\;\big\langle Y-f(X),\,a\big\rangle\;\one\!\Big\{\,s\le \big\langle f(X),a\big\rangle\le t\,\Big\},\qquad
\Gcal_{\lin}=\{g_{a,s,t}:a\in[-1,1]^C,\ s\le t\in[-1,1]\}.
\end{equation*}
 
Let $p=f(X)$ and $W=Y-f(X)$. By the definition of the subgraph class, for any fixed $(X,Y)$ and threshold $r\in\R$,
\[
 r\le g_{a,s,t}(X,Y)\ \iff\ \big( s\le \langle p,a\rangle\le t\ \wedge\ r\le \langle W,a\rangle \big)\ \vee\ \big( \neg[ s\le \langle p,a\rangle\le t ]\ \wedge\ r\le 0 \big).
\]
We introduce the following atomic linear predicates in the parameters $(a,s,t)$:
\[
\begin{aligned}
&A_1:\ \langle p,a\rangle - s \ge 0,\qquad A_2:\ \; t - \langle p,a\rangle \ge 0,\\
&A_3:\ \langle W,a\rangle - r \ge 0,\qquad A_4:\ \; -r \ge 0\, .
\end{aligned}
\]
Note that $(A_1\wedge A_2)\iff s\le \langle p,a\rangle\le t$, $A_3\iff r\le \langle W,a\rangle$, and $A_4\iff r\le 0$.  Therefore,
\[
 r\le g_{a,s,t}(X,Y)\ \iff\ \big( A_1 \wedge A_2 \wedge A_3 \big)\ \vee\ \big( (\neg A_1 \vee \neg A_2) \wedge A_4 \big).
\]

Thus, under the above equivalent, we obtain that $\Pdim\open{\gcal_\lin} =   O(C\log(C))$, by invoking \Cref{prop:semi-alg-vc} with $m=C+2$.  

For the lower bound, take $C$ points with $Y=e_1,\ldots, e_C$ and $p=0$; with $s=-1,t=1$, $g_{a,s,t}=a_j$ and thus shattering $C$ points, so $\Pdim(\Gcal_{\lin})\ge C$.
\end{proof}

Given the pseudo-dimension bound, we can now directly deduce the sample complexity for interactive measurability.
\begin{corollary}[Sample complexity for $\ucal_{\lin}$]\label{thm:lin-sample} Combining \Cref{prop:lin-pdim} and \Cref{prop:pdim-uniform}, we obtain that
\[
\sup_{u\in\ucal_{\lin}}\big|\widehat{\uc}(f,u;S)-\uc(f,u)\big|\ \le\ O\open{\frac{\sqrt{2C\log(n/C)}+ \sqrt{\log(1/\delta)}}{\sqrt{n}}}.
\]
Equivalently, interactive measurability holds with $n=\tilde{O}\!\big((C+\log(1/\delta))/\eps^2\big)$ for the $\ucal_{\lin}$ class.
\end{corollary}

\begin{corollary}[Rank utilities via sorting preprocessing]\label{cor:rank-composition}
Let $\pi_p$ denote the permutation that sorts $p=f(X)$ in descending order and define the fixed ordering map $T: (p,Y) \mapsto (\tilde p, \tilde Y)$ with $\tilde p_s = p_{\pi_p(s)}$ and $\tilde Y_s = Y_{\pi_p(s)}$. Then, any rank-based utility may be expressed as a linear utility up to a preprocessing with $T$. Consequently, $\Pdim(\Gcal_{\ucal_{\rank}})= O(C\log C)$ and the sample complexity for interactive measurability for $\ucal_{\rank}$ matches that of $\ucal_{\lin}$.
\end{corollary}

\begin{proof}
Let $\pi_p$ be the permutation that sorts $p=f(X)$ in descending order and define the fixed ordering map $T:(p,Y)\mapsto(\tilde p,\tilde Y)$ with $\tilde p_s=p_{\pi_p(s)}$ and $\tilde Y_s=Y_{\pi_p(s)}$. For any $u_\theta\in\ucal_{\rank}$, we have $[\vec u_\theta(p)]_j=\theta_{\rank(p,j)}$ and $v_{u_\theta}(p)=\sum_{s=1}^C \tilde p_s\,\theta_s=\langle \tilde p,\theta\rangle$. Hence, for any interval $[s,t]\subseteq[-1,1]$,
\[
\inner{Y-f(X)}{\vec u_\theta(f(X))}\;\one\{\,s\le v_{u_\theta}(f(X))\le t\,\}
\;=\; \inner{\tilde Y-\tilde p}{\theta}\;\one\{\,s\le \inner{\tilde p}{\theta}\le t\,\},
\]
which is exactly the linear-utility template on the transformed sample $T(S)$. As \Cref{prop:lin-pdim} is distribution-independent and $T$ is a fixed function, we have $\Pdim(\Gcal_{\ucal_{\rank}}) =  O(C\log C)$. Interactive measurability follows from \Cref{prop:pdim-uniform}, similarly to~\Cref{thm:lin-sample}.
\end{proof}

\subsection{Quantile Guarantees}\label{app:quantile_gaurantees}
The evaluation methodology in \Cref{sec:experiments} relies on plotting the empirical CDF  $\widehat{F}_{E,M,n}$, constructed from $M$ utility functions and $n$ data points. $\widehat{F}_{E,M,n}$ serves as a proxy for the true CDF $F_E$ of utility calibration errors. \Cref{thm:cdf_closeness} bounds the distance between the curves, as characterized by the $L_2$-distance. While weaker than characterizing the deviation using $L_\infty$-norm, $L_2$ distance allows a bound without any assumptions on the  smoothness of the underlying CDF $F_E$. The bound depends on two factors: first, the uniform accuracy $\eps_{\text{stat}}$ with which individual utility calibration errors $\uc(f,u)$ can be estimated by $\widehat{\uc}(f,u;S)$, and second, the number of utility functions $M$ sampled to construct the $\widehat{F}_{E,M,n}$.

\begin{thm}\label{thm:cdf_closeness}
Let $F_E(e) \coloneqq \prob_{u \sim \mathcal{D}_{\ucal}}(\uc(f, u) \le e)$ be the true CDF of utility calibration errors, where $u$ is drawn from a distribution $\mathcal{D}_{\ucal}$ over the utility class $\ucal$. Let $\widehat{F}_{E,M,n}(e) \coloneqq \frac{1}{M} \sum_{m=1}^M \one\{\widehat{\uc}(f, u_m; S) \le e\}$ be its empirical estimate, based on $M$ i.i.d.~utility functions $u_1,\ldots, u_M \sim \mathcal{D}_{\ucal}$ and a data sample $S$ of size $n\open{\delta_S,\eps_{\text{stat}}}$. Assume that, with probability at least $1-\delta_S$ over the draw of $S$, 
\begin{equation*}
    \sup_{u \in \ucal} |\widehat{\uc}(f,u;S) - \uc(f,u)| \le \eps_{\text{stat}}.
\end{equation*}
Then, with probability at least $(1-\delta_S-\delta_M)$ over the draw of $S$ and $\{u_m\}_{m=1}^M$, it holds that
\begin{equation*}
    \|F_E - \widehat{F}_{E,M,n}\|_{L^2([0,2])} \;\le\; \sqrt{2\eps_{\text{stat}}} + \sqrt{\frac{\ln(2/\delta_M)}{M}}.
\end{equation*}
\end{thm}

\begin{proof}
Let $E(u) \coloneqq \uc(f,u)$ be the true utility calibration error for a utility function $u \sim \mathcal{D}_{\ucal}$, and let $\widehat{E}_S(u) \coloneqq \widehat{\uc}(f,u;S)$ be its empirical estimate based on data sample $S$. The CDF of $E(u)$ is $F_E$, and let $F_{\widehat{E}_S}$ be the CDF of $\widehat{E}_S(u)$ when $u \sim \mathcal{D}_{\ucal}$ and $S$ is fixed.

We condition on the event (occurring with probability at least $1-\delta_S$) that $S$ is such that $|E(u) - \widehat{E}_S(u)| \le \eps_{\text{stat}}$ for all $u \in \ucal$. This implies that for any $u$, $E(u) - \eps_{\text{stat}} \le \widehat{E}_S(u) \le E(u) + \eps_{\text{stat}}$.
Consequently, for any $u\in\ucal$, as $E(u)$ and $ \hat{E}(u)$ belong to the set $\closed{0,2}$, it holds for any $e \in [0, 2]$ that
\begin{gather*}
    F_E(e - \eps_{\text{stat}}) \;=\; \prob(E(u) \le e - \eps_{\text{stat}}) \;\le\; \prob(\widehat{E}_S(u) \le e) \;=\; F_{\widehat{E}_S}(e) \\
    F_{\widehat{E}_S}(e) \;=\; \prob(\widehat{E}_S(u) \le e) \;\le\; \prob(E(u) \le e + \eps_{\text{stat}}) \;=\; F_E(e + \eps_{\text{stat}}).
\end{gather*}
Thus, $F_E(e - \eps_{\text{stat}}) \le F_{\widehat{E}_S}(e) \le F_E(e + \eps_{\text{stat}})$.
This implies that
\begin{equation*}
    |F_{\widehat{E}_S}(e) - F_E(e)| \;\le\; \max\{F_E(e+\eps_{\text{stat}})-F_E(e), F_E(e)-F_E(e-\eps_{\text{stat}})\}.
\end{equation*}
Let $\Delta_E(e)$ denote the right-hand side. The $L^2$ distance squared between $F_{\widehat{E}_S}$ and $F_E$ is
\begin{align*}
    \|F_{\widehat{E}_S} - F_E\|_{L^2([0,2])}^2 &= \int_0^{2} (F_{\widehat{E}_S}(e) - F_E(e))^2 \mathrm{d}e \\
    &\le \int_0^{2} \Delta_E(e)^2 \mathrm{d}e \;\le\; \int_0^{2} \Delta_E(e) \mathrm{d}e,
\end{align*}
since $0 \le \Delta_E(e) \le 1$.
Further, as $\max(a,b) \le a+b$ for non-negative $a,b$, it follows that
\begin{align*}
\int_0^{2} \Delta_E(e) \mathrm{d}e &\le \int_0^{2} [F_E(e+\eps_{\text{stat}})-F_E(e)] \mathrm{d}e + \int_0^{2} [F_E(e)-F_E(e-\eps_{\text{stat}})] \mathrm{d}e.
\end{align*}
The first term evaluates to $\int_{2}^{2+\eps_{\text{stat}}} F_E(v)dv - \int_0^{\eps_{\text{stat}}} F_E(v)dv$. Since $F_E(v)=1$ for $v \ge 2$ and $F_E(v) \ge 0$, this term is $\le \eps_{\text{stat}}$.
The second term evaluates to $\int_{2-\eps_{\text{stat}}}^{2} F_E(v)dv - \int_{-\eps_{\text{stat}}}^0 F_E(v)dv$. Since $F_E(v)=0$ for $v < 0$ (errors are non-negative) and $F_E(v) \le 1$, this term is $\le \eps_{\text{stat}}$.
Thus, $\int_0^{2} \Delta_E(e) de \le 2\eps_{\text{stat}}$, which implies $\|F_{\widehat{E}_S} - F_E\|_{L^2([0,2])} \le \sqrt{2\eps_{\text{stat}}}$.

Next, $\widehat{F}_{E,M,n}(e)$ is the empirical CDF of $M$ i.i.d.~samples $\{\widehat{E}_S(u_m)\}_{m=1}^M$ drawn according to $F_{\widehat{E}_S}$. By the Dvoretzky–Kiefer–Wolfowitz (DKW) inequality \citep{massart1990tight}, with probability at least $1-\delta_M$:
\begin{equation*}
    \|\widehat{F}_{E,M,n} - F_{\widehat{E}_S}\|_{\infty} \;\le\; \sqrt{\frac{\ln(2/\delta_M)}{2M}}.
\end{equation*}
The $L^2$ distance can be bounded by the $L^\infty$ distance, as
\begin{align*}
    \|F_{\widehat{E}_S} - \widehat{F}_{E,M,n}\|_{L^2([0,2])}^2 &= \int_0^{2} (F_{\widehat{E}_S}(e) - \widehat{F}_{E,M,n}(e))^2 \mathrm{d}e \\
    &\le \int_0^{2} \|\widehat{F}_{E,M,n} - F_{\widehat{E}_S}\|_{\infty}^2 \mathrm{d}e \\
    &= 2 \|\widehat{F}_{E,M,n} - F_{\widehat{E}_S}\|_{\infty}^2 \;\le\;  \frac{\ln(2/\delta_M)}{M}.
\end{align*}
So, $\|F_{\widehat{E}_S} - \widehat{F}_{E,M,n}\|_{L^2([0,2])} \le \sqrt{ \frac{\ln(2/\delta_M)}{M}}$. Finally, by the triangle inequality, combining the two bounds:
\begin{align*}
    \|F_E - \widehat{F}_{E,M,n}\|_{L^2([0,2])} &\le \|F_E - F_{\widehat{E}_S}\|_{L^2([0,2])} \quad + \|F_{\widehat{E}_S} -\widehat{F}_{E,M,n}\|_{L^2([0,2])} \\
    &\le \sqrt{2\eps_{\text{stat}}} + \sqrt{ \frac{\ln(2/\delta_M)}{M}}.
\end{align*}
Using a union bound, this holds with probability at least $(1-\delta_S-\delta_M)$.
\end{proof}

\subsection{Proofs}\label{app:proof}
We now proceed to the proofs of the statements in the main paper.

\subsection{Proof of \texorpdfstring{\Cref{prop:uc_risk_gap_main}}{Proof of Proposition}}

\begin{proof}
At a high-level, the proof is similar to the approach of \citet{rossellini2025can} but some details differ. We start by analyzing the loss function
\begin{equation} \label{eq:l_util_appendix}
\ell_{\mathrm{util}}(u_Y, \hat{U}; t_0) = (t_0 - u_Y) (\hat{U} - \one\ens{u_Y \ge t_0}).
\end{equation}
\begin{itemize}[noitemsep,topsep=0pt]
    \item If $\hat{U} = \one\ens{u_Y \ge t_0}$, then $\ell_{\mathrm{util}} = 0$.
    \item If $\hat{U}=1$ and $u_Y < t_0$, the loss is $ t_0-u_Y > 0$.
    \item If $\hat{U}=0$ but $u_Y \ge t_0$ , the loss is $u_Y-t_0 \ge 0$.
\end{itemize}
The loss, which is only non-zero when a mismatch $\hat{U}\neq \one\ens{u_Y\geq t_0}$ occurs, is $|u_Y-t_0|$. This loss function penalizes mismatches between the action taken based on $v_u(X)$ and the ideal action based on $u_Y$. We now consider any monotone non-decreasing function $h:\closed{-1,1}\rightarrow [-1,1].$ The difference in risks between using $v_u(X)$ directly and using $h(v_u(X))$ is
\begin{align*}
\Delta R &= R_{\mathrm{util}}(v_u(X); t_0) - R_{\mathrm{util}}(h(v_u(X)); t_0) \\
&= \mathbb{E}\Big[ (t_0 - u_Y) (\one\ens{v_u(X) \ge t_0} - \one\ens{u_Y \ge t_0}) \Big] - \mathbb{E}\Big[ (t_0 - u_Y) (\one\ens{h(v_u(X)) \ge t_0} - \one\ens{u_Y \ge t_0}) \Big] \\
&= \mathbb{E}\Big[ (t_0 - u_Y) (\one\ens{v_u(X) \ge t_0} - \one\ens{h(v_u(X)) \ge t_0}) \Big].
\end{align*}
Let $E_1 = \{X \mid v_u(X) < t_0, h(v_u(X)) \ge t_0\}$ and $E_2 = \{X \mid v_u(X) \ge t_0, h(v_u(X)) < t_0\}$. The term $\one\ens{v_u(X) \ge t_0} - \one\ens{h(v_u(X)) \ge t_0}$ equals $-1$ on $E_1$ and $1$ on $E_2$, and $0$ elsewhere.
\begin{align*}
\Delta R &= \mathbb{E}\Big[ (t_0 - u_Y) (-\one\ens{X \in E_1}) + (t_0 - u_Y) (\one\ens{X \in E_2}) \Big] \\
&= \mathbb{E}\Big[ (u_Y - t_0) \one\ens{X \in E_1} \Big] - \mathbb{E}\Big[ (u_Y - t_0) \one\ens{X \in E_2} \Big]\\
&= \mathbb{E}\Big[ (u_Y - v_u(X) + v_u(X) - t_0) \one\ens{X \in E_1} \Big] - \mathbb{E}\Big[ (u_Y - v_u(X) + v_u(X) - t_0) \one\ens{X \in E_2} \Big] \\
&= \underbrace{\mathbb{E}\Big[ (u_Y - v_u(X)) \one\ens{X \in E_1} \Big]}_{A} + \underbrace{\mathbb{E}\Big[ (v_u(X) - t_0) \one\ens{X \in E_1} \Big]}_{C} \\
&\quad - \underbrace{\mathbb{E}\Big[ (u_Y - v_u(X)) \one\ens{X \in E_2} \Big]}_{B} - \underbrace{\mathbb{E}\Big[ (v_u(X) - t_0) \one\ens{X \in E_2} \Big]}_{D}.
\end{align*}
On $E_1$, we have $v_u(X) < t_0$, which implies $v_u(X) - t_0 < 0$, so $C \le 0$.
On $E_2$, we have $v_u(X) \ge t_0$, which implies $v_u(X) - t_0 \ge 0$, so $D \ge 0$.
Thus, $\Delta R = A + C - B - D \le A - B$.
\( A - B = \mathbb{E}[ (u_Y - v_u(X)) \one\ens{X \in E_1} ] + \mathbb{E}[ (v_u(X) - u_Y) \one\ens{X \in E_2} ] \).
Since $h$ is monotone non-decreasing, the sets $E_1$ and $E_2$ correspond to $v_u(X)$ lying within specific intervals (or unions of intervals which can be decomposed). Let $I_1$ be the set of $v_u(X)$ values defining $E_1$ (e.g., $v_u(X) < t_0$ and $h(v_u(X)) \ge t_0$) and $I_2$ be the set for $E_2$. The terms $\one\ens{X \in E_1}$ and $\one\ens{X \in E_2}$ effectively restrict the expectation to regions where $v_u(X)$ falls into certain ranges.
By the definition of $\uc(f, u)$, for $E\in \ens{E_1, E_2}$
\( \mathbb{E}[ (u_Y - v_u(X)) \one\ens{X \in E} ] \le \sup_{I\in \cinv{-1}{1}} \modu{\mathbb{E}[ (u_Y - v_u(X)) \one\ens{v_u(X) \in I}]} \le \uc(f, u) \).

Therefore, $\Delta R \le A - B \le \uc(f, u) + \uc(f, u) = 2 \uc(f, u)$.
Since this holds for any monotone non-decreasing function $h$, taking the supremum over monotone functions completes the proof. 
\end{proof}

\subsection{Proof of \texorpdfstring{\Cref{prop:uc_bounds_dcu_main}}{Proof of Proposition}}
\begin{proof} The proof is the same as \cite[Lemma A.2.]{rossellini2025can}, we include it for completeness.
Assume $\uc(f,u) > 0$. Let $U_Y \coloneqq u(f(X),Y)$ denote the realized utility. Both $U_Y$ and $v_u(X)$ take values in $[-1,1]$.

Let $W \in (0, 2]$ be a chosen bin width. We partition the interval $[-1,1]$ into $K_W = \lceil 2/W \rceil$ disjoint intervals $A_1, A_2, \ldots, A_{K_W}$. These intervals are constructed as follows:
For $j = 1, \ldots, K_W-1$, let $A_j = [-1 + (j-1)W, -1 + jW)$.
For $j = K_W$, let $A_{K_W} = [-1 + (K_W-1)W, 1]$.
This construction ensures that $\bigcup_{j=1}^{K_W} A_j = [-1,1]$, and each interval $A_j$ has length $\lambda(A_j) \le W$. 

Let $\psi_W: [-1,1] \to \{1, \ldots, K_W\}$ be the function mapping a value $z \in [-1,1]$ to the index $j$ of the bin $A_j$ such that $z \in A_j$.
We construct a candidate calibrated predictor $g_W(X) \coloneqq \mathbb{E}[U_Y \mid \psi_W(v_u(X))]$.
The function $g_W(X)$ is perfectly calibrated:
$\mathbb{E}[U_Y \mid g_W(X)] = \mathbb{E}[\mathbb{E}[U_Y \mid \psi_W(v_u(X))] \mid g_W(X)]$. Since $g_W(X)$ is, by definition, measurable with respect to the sigma-algebra generated by $\psi_W(v_u(X))$, it follows from the properties of conditional expectation that $\mathbb{E}[U_Y \mid g_W(X)] = g_W(X)$ almost surely.

By the definition of $\mathrm{DCU}(f, u)$, which is the infimum of distances $\mathbb{E}|g(X) - v_u(X)|$ over all perfectly calibrated predictors $g(X)$, it holds that $\mathrm{DCU}(f, u) \leq \mathbb{E}|g_W(X) - v_u(X)|$.
To analyze the right-hand side, we introduce an intermediate term $V_{\mathrm{avgbin}}(X) \coloneqq \mathbb{E}[v_u(X) \mid \psi_W(v_u(X))]$. This represents the average of $v_u(X)$ within the bin $A_{\psi_W(v_u(X))}$ where $v_u(X)$ falls.
Using the triangle inequality: $\mathbb{E}|g_W(X) - v_u(X)| \leq \mathbb{E}|g_W(X) - V_{\mathrm{avgbin}}(X)| + \mathbb{E}|V_{\mathrm{avgbin}}(X) - v_u(X)|$.

To bound the second term, $\mathbb{E}|V_{\mathrm{avgbin}}(X) - v_u(X)|$:
For any realization $X=x$, $v_u(x)$ is a point in some bin $A_j$. $V_{\mathrm{avgbin}}(x)$ is the conditional expectation of $v_u(X')$ given that $v_u(X')$ is in $A_j$. As such, $V_{\mathrm{avgbin}}(x)$ must also lie within the convex hull of $A_j$. Thus, $|V_{\mathrm{avgbin}}(x) - v_u(x)| \leq W$.
Taking the expectation over $X$, we get $\mathbb{E}|V_{\mathrm{avgbin}}(X) - v_u(X)| \leq W$.

To bound the first term:
\begin{align*}
\mathbb{E}|g_W(X) - V_{\mathrm{avgbin}}(X)| &= \mathbb{E}\left|\mathbb{E}[U_Y \mid \psi_W(v_u(X))] - \mathbb{E}[v_u(X) \mid \psi_W(v_u(X))]\right| \\
&= \mathbb{E}\left|\mathbb{E}[U_Y - v_u(X) \mid \psi_W(v_u(X))]\right| \\
&= \sum_{j=1}^{K_W} \mathbb{P}\{\psi_W(v_u(X)) = j\} \left|\mathbb{E}[U_Y - v_u(X) \mid \psi_W(v_u(X))=j]\right|  \\
&= \sum_{j=1}^{K_W} \left|\mathbb{E}\left[(U_Y - v_u(X))\one\ens{\psi_W(v_u(X))=j}\right]\right| \\
&= \sum_{j=1}^{K_W} \left|\mathbb{E}\left[(U_Y - v_u(X))\one\ens{v_u(X) \in A_j}\right]\right|.
\end{align*}
By the definition of utility calibration error $\uc(f,u)$, for each bin $A_j$, we have:
\( \left|\mathbb{E}\left[(U_Y - v_u(X))\one\ens{v_u(X) \in A_j}\right]\right| \leq \sup_{I \in \cinv{-1}{1}} \left|\mathbb{E}\left[(U_Y - v_u(X))\one\ens{v_u(X) \in I}\right]\right| = \uc(f,u) \).
Therefore,
\( \mathbb{E}|g_W(X) - V_{\mathrm{avgbin}}(X)| \leq \sum_{j=1}^{K_W} \uc(f,u) = K_W  \uc(f,u) \).
Since $K_W = \lceil 2/W \rceil$, and for any $x>0$, $\lceil x \rceil \le x + 1$ (with strict inequality if $x$ is not an integer), we have $K_W \le 2/W + 1$.
Thus, $\mathbb{E}|g_W(X) - V_{\mathrm{avgbin}}(X)| \leq (2/W + 1)  \uc(f,u)$.

Combining the bounds for the two terms, we get:
\( \mathrm{DCU}(f, u) \leq \left(\tfrac{2}{W} + 1\right)  \uc(f,u) + W \).
This inequality holds for any chosen bin width $W \in (0, 2]$. Our goal is to select $W$ to minimize this upper bound. Set $W = W_{opt}\coloneqq \sqrt{2 \uc(f,u)}$, noting that $W_{opt}$ is in the domain $(0, 2]$ as $\uc(f,u)\leq 2$, for any $u$ and $f$.

Substituting $W_{opt} = \sqrt{2 \uc(f,u)}$ into the upper bound for $\mathrm{DCU}(f, u)$:
\begin{align*}
\mathrm{DCU}(f, u) &\leq \frac{2}{\sqrt{2 \uc(f,u)}} \uc(f,u) + \uc(f,u) + \sqrt{2 \uc(f,u)} \\
&= 2\sqrt{2 \uc(f,u)} + \uc(f,u).
\end{align*}

\end{proof}

\subsection{Proof of \texorpdfstring{\Cref{lem:estimate_uc_single}}{Proof of Lemma}}

\begin{proof}
    Let $A(X,Y) \eqdef u(f(X),Y) - v_u(X)$ and $V \eqdef v_u(X)$. Since $u(\cdot,\cdot), v_u(\cdot)\in[-1,1]$, we have $A(X,Y)\in[-2,2]$ and $V\in[-1,1]$ almost surely. For any interval $I\in \cinv{-1}{1}$, define
    \begin{gather*}
        \fcal_{\mathrm{int}} \eqdef \big\{\,h_I: (X,Y)\mapsto A(X,Y)\,\one\ens{V\in I}\ \big|\ I\in\cinv{-1}{1}\,\big\},\\
        \fcal_{\le} \eqdef \big\{\,h_t: (X,Y)\mapsto A(X,Y)\,\one\ens{V\le t}\ \big|\ t\in[-1,1]\,\big\}.  
    \end{gather*}
 
    By definition,
    \[
      \uc(f,u) \,=\, \sup_{I\in\cinv{-1}{1}} \big|\E[h_I]\big|,\qquad
      \widehat{\uc}(f,u;S) \,=\, \sup_{I\in\cinv{-1}{1}} \big|\hat{\E}_n[h_I]\big|,
    \]
    and hence
    \[
      \big|\uc(f,u)-\widehat{\uc}(f,u;S)\big| \;\le\; \sup_{I\in\cinv{-1}{1}}\big|(\E-\hat{\E}_n)h_I\big|.
    \]
    where inequality uses $|\sup a-\sup b|\le \sup|a-b|$ and $||x|-|y||\le |x-y|$. For any $a<b$ in $[-1,1]$, we have $h_{(a,b]} = h_b - h_a$ with $h_t\in\fcal_{\le}$. Hence, for any linear functional $T$,
    \begin{equation*}\label{eq:int-to-thr}
      \sup_{I\in\cinv{-1}{1}} |T(h_I)| \;\le\; 2\,\sup_{t\in[-1,1]} |T(h_t)|.
    \end{equation*}
    
    We use the following standard bounded-range uniform deviation bound: for any real-valued function class $\Hcal$ with range $[-M,M]$ and any $\delta\in(0,1)$ \citep[Theorem 3.3]{mohri2018foundations},
    \begin{equation*}\label{eq:unif-dev-bounded}
      \sup_{h\in\Hcal} \big| (\E-\hat{\E}_n)h \big| \;\le\; 2\,\rad_n(\Hcal)\; +\; M\,\sqrt{\tfrac{\log(1/\delta)}{2n}} \quad \text{w.p. }\ge 1-\delta.
    \end{equation*}

    Fix the sample $S=\{(X_i,Y_i)\}_{i=1}^n$. Let $V_i\eqdef v_u(X_i)$ and $A_i\eqdef u(f(X_i),Y_i)-v_u(X_i)$. Let $\pi$ be a permutation that sorts the $V_i$: $V_{(1)}\le\cdots\le V_{(n)}$ and set $A_{(i)}\eqdef A_{\pi(i)}$. Given $S$, the sequence $(A_{(i)})_{i=1}^n$ is deterministic. Let $(\sigma_i)_{i=1}^n$ be i.i.d.~Rademacher signs, independent of $S$, and define the filtration
    \[
      \fcal_k \eqdef \sigma(\sigma_1,\ldots,\sigma_k, S), \qquad k=0,1,\ldots,n.
    \]
    Define partial sums $S_k \eqdef \sum_{i=1}^k \sigma_i A_{(i)}$ with $S_0\eqdef 0$. Then $(S_k,\fcal_k)_{k=0}^n$ is a square-integrable martingale since
    \[
      \E\left[S_k-S_{k-1}\mid \fcal_{k-1}\right] = A_{(k)}\,\E[\sigma_k\mid \fcal_{k-1}] = 0,\qquad \E\left[S_n^2|S\right]=\sum_{i=1}^n A_{(i)}^2.
    \]
    
    We invoke Doob's $L^p$ submartingale maximal inequality, for $p>1$,
    \begin{equation*}\label{eq:doob-lp}
      \E\left[\left(\sup_{0\le k\le n} |S_k|\right)^p\right] \;\le\; \Big(\tfrac{p}{p-1}\Big)^{\!p}\, \E\left[|S_n|^p\right].
    \end{equation*}
    Taking $p=2$ yields
    \begin{equation*}\label{eq:max-L2}
      \E_\sigma\left[\sup_{0\le k\le n} |S_k|^2\right] \;\le\; 4\,\E_\sigma[S_n^2] \;=\; 4\sum_{i=1}^n A_{(i)}^2,\qquad \text{hence}\qquad \E_\sigma\left[\sup_{k} |S_k|\right] \;\le\; 2\,\Big(\sum_{i=1}^n A_{(i)}^2\Big)^{1/2}.
    \end{equation*}
    
    Then, it holds that 
    \[
      \rad_n(\fcal_{\le}\mid S) \,=\, \frac{1}{n}\,\E_\sigma\, \max_{0\le k\le n} \sum_{i=1}^k \sigma_i A_{(i)} \;\le\; \frac{1}{n}\, \E_\sigma\left[\sup_k |S_k|\right] \;\le\; \frac{2}{n}\,\Big(\sum_{i=1}^n A_{(i)}^2\Big)^{1/2} \;\le\; \frac{4}{\sqrt{n}},
    \]
    where the last inequality uses $|A_{(i)}|\le 2$. Then, it holds that 
    \begin{align*}
        \big|\uc(f,u)-\widehat{\uc}(f,u;S)\big| &\leq  \sup_{I\in\cinv{-1}{1}}\big|(\E-\hat{\E}_n)h_I\big| \\
        &\leq 2 \sup_{t\in[-1,1]} \big|(\E-\hat{\E}_n)h_t\big| \\
        &\leq 4 \rad_n(\fcal_{\le}) + 2\sqrt{\tfrac{\log(1/\delta)}{2n}}\\
        &\leq \frac{16}{\sqrt{n}} + 4\sqrt{\tfrac{\log(1/\delta)}{2n}}.
    \end{align*}    
    \end{proof}
\section{Additional Experiments}\label{app:experiments}

\Cref{app:exp_details} describes datasets, models, and calibration methods. \Cref{app:aligned_misaligned} takes a deeper look at the empirical behavior of \Cref{alg:patching_uc_weighted_cal}. To this end, it demonstrates the empirical behavior of \Cref{alg:patching_uc_weighted_cal} on two families of utility functions: aligned and misaligned. Finally, \Cref{app:results} compiles extended tables and eCDF plots across CIFAR10/100, ImageNet, and an additional text classification experiment. In addition, it introduces two additional utility function classes; one inspired by discounted cumulative gain (DCG) \citep{jarvelin2002cumulated} and the other inspired by semantic-based classification loss \citep{deng2012hedging}.

\subsection{Experimental Details}\label{app:exp_details}

\textbf{Experimental Setup:} Our experiments were performed using standard image classification benchmarks: CIFAR10, CIFAR100 \citep[MIT License]{krizhevsky2009learning}, ImageNet-1K \citep[provided for non-commercial use]{deng2009imagenet}, and the test split of Yahoo Answers Topics \citep[non-commercial use]{zhang2015character}. We used publicly available pretrained models. For each model–dataset combination, the validation data was divided into a calibration set ($70\%$) for training post-hoc methods and a test set ($30\%$) for evaluation. We repeated the experiments across $10$ different calibration/test splits and reported the average results. For CIFAR10/100 datasets, we used the model checkpoints available in \citet{chenyaofo_pytorch_cifar_models_2020}. For ImageNet-1K, we used \texttt{timm} checkpoints \citep{Wightman_PyTorch_Image_Models}. For text classification, we used the publicly available model \url{https://huggingface.co/Koushim/distilbert-yahoo-answers-topic-classifier}, provided under the Apache 2.0 license. All post-hoc methods were trained on the designated calibration set and subsequently evaluated on the test set. All experiments (including hyperparameter tuning) took approximately 120 hours on an A100 80GB NVIDIA GPU.

\textbf{Calibration Methods and Evaluation:}
We benchmarked several well-known post-hoc calibration techniques. These include Temperature Scaling (Temp. S.) \citep{platt1999probabilistic, guo2017calibration}, Vector Scaling (V.S.) \citep{kull2017beyond,guo2017calibration}, Dirichlet recalibration \citep{kull2019beyond} with off-diagonal regularization (ODIR), and Isotonic Regression (I.R.) \citep{zadrozny2002transforming}, applied both globally and in a one-vs-all manner. For these methods, we based our implementation on the publicly available code in \citep{Salvador_Calibration_Baselines_2022}. For Dirichlet recalibration, we tuned the regularization parameter over $80$ runs using \texttt{BayesSearchCV} from \texttt{scikit-optimize} \citep{tim_head_2018_1207017} with the goal of minimizing the negative log-likelihood.

\paragraph{\Cref{alg:patching_uc_weighted_cal} implementation:} Building on the algorithmic template in \Cref{alg:patching_uc_weighted_cal}, our implementation iteratively (i) audits to find the worst witness over the chosen utility class using the worst-interval estimator, and (ii) applies a masked update along that witness only for points whose predicted utility falls in the identified interval, followed by projection onto the probability simplex. For the stepsize selection, we also employ a backtracking Armijo line search, using JAXOPT \citep{jaxopt_implicit_diff}, to choose a stepsize that decreases the Brier score. The implementation minimizes the utility calibration against the union class $\ucal = \ucal_{\cwe}\cup \ucal_{\topk}$, as defined in \Cref{sec:utility_calibration}. In addition, at each step, we sampled $264$ utility functions from $\ucal_\rank$ and $\ucal_\lin$ to augment the set of utility functions. Hyperparameter tuning was performed using \texttt{scikit-optimize} \citep{tim_head_2018_1207017} over the stepsize, the number of steps, and whether to sample additional utility functions from $\ucal_\rank$ and $\ucal_\lin$. For hyperparameter tuning, we performed $20$ runs.

\textbf{Evaluation.} Model performance and calibration were assessed using a suite of metrics. Standard evaluations included Accuracy and Brier score. For binned binarized approaches, we compute $\mathrm{TCE}^{\mathrm{bin}}$ (\ref{eq:tce}) and $\mathrm{CWE}^{\mathrm{bin}}$ (\ref{eq:cwce}), using $15$ equal-weight bins (each bin has the same number of datapoints). Furthermore, we evaluated our proposed binning-free utility calibration metrics for specific utility classes: $\ucal_{\tce}$, $\ucal_{\cwe}$, and $\ucal_{\topk}$. In addition, we computed other utility calibration metrics for specific utility classes described in \Cref{app:results}.

\subsection{\Cref{alg:patching_uc_weighted_cal} for Aligned and Misaligned Utility Classes}\label{app:aligned_misaligned}

To better understand the dynamics of calibrating against different utility structures, this set of experiments focuses on the convergence speed (and final performance) of \Cref{alg:patching_uc_weighted_cal} when faced with ``aligned'' versus ``misaligned'' non-linear utility classes. 

\noindent\textbf{Utility model.} Let $R\in[0,1]^{C\times C}$ be a user-defined gain matrix with $R_{ii}=1$, where $R_{ij}$ is the gain for predicting class $j$ when the true class is $i$. For a prediction $p=f(X)$, a rational agent may choose
\[
j^\star(p)\;=\;\argmax_{j\in[C]}\ \sum_{i=1}^C p_i\,R_{ij}.
\]
As we consider the utility function
\[
u_R(p,\,y=e_i)\;=\;R_{i,\,j^\star(p)}\,.
\]
Consistent with our framework, we use the associated regressor
\(
v_{u_R}(X)\coloneqq \sum_{i=1}^C f(X)_i\,R_{i,\,j^\star(f(X))}
\).

We investigate two utility classes to simulate two scenarios: one with all users having the same preferences (aligned) and the other with users having distinct preferences (misaligned).
\begin{itemize}
\item Aligned class: models users with a low tolerance for any error. We set $R_{ii}=1$ and sample all off-diagonal entries ($i\neq j$) i.i.d.\ from $\mathrm{Unif}(0,0.1)$.
\item Misaligned class: models a mix of specialists with distinct biases. We partition the labels into disjoint subsets $(K_1,K_2,\dots)$. For a user specializing in $K_m$, we set $R_{ii}=1$, $R_{ij}=0.2$ for all $j\in K_m$ with $i\neq j$, and all other off-diagonals to $0$. The utility class is a mixture over such specialists.
\end{itemize}

\noindent\textbf{Evaluation.} We illustrate the performance of \Cref{alg:patching_uc_weighted_cal} on these classes on CIFAR100 and ImageNet-1K, using ViT and ResNet56 checkpoints. For each utility class, we sampled $512$ utility functions per iteration. For each dataset, we track (i) the trajectory of the utility calibration error across iterations of \Cref{alg:patching_uc_weighted_cal} for both aligned and misaligned classes (see \Cref{fig:app_cifar100_history}), and (ii) the distribution of per-example utility calibration errors at selected iterations to visualize distributional shifts across the class (see \Cref{fig:app_aligned_misaligned_overview}). In summary, we provide three panels for each dataset: the iteration curve and two distributional grids (aligned and misaligned).

\begin{figure}[htbp]
  \begin{minipage}{0.49\linewidth}
    \centering
    \includegraphics[width=\linewidth]{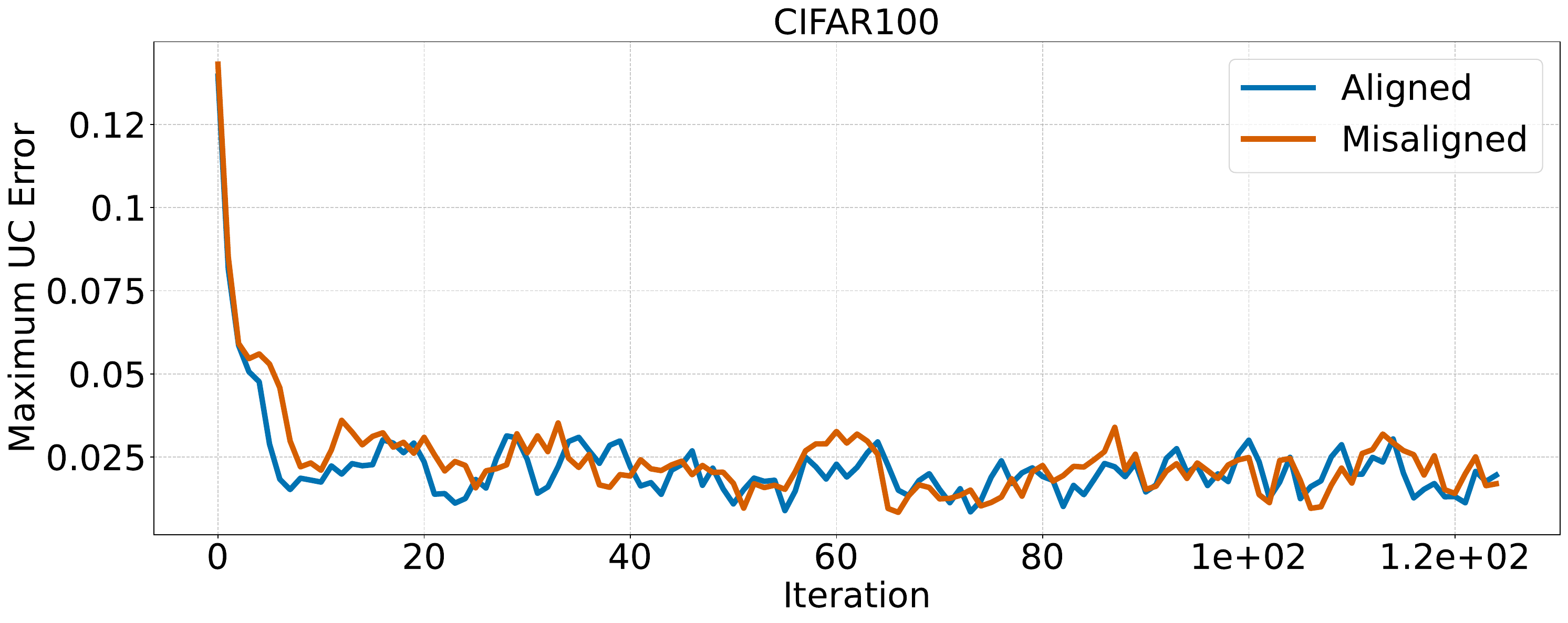}
  \end{minipage}\hfill
  \begin{minipage}{0.49\linewidth}
    \centering
    \includegraphics[width=\linewidth]{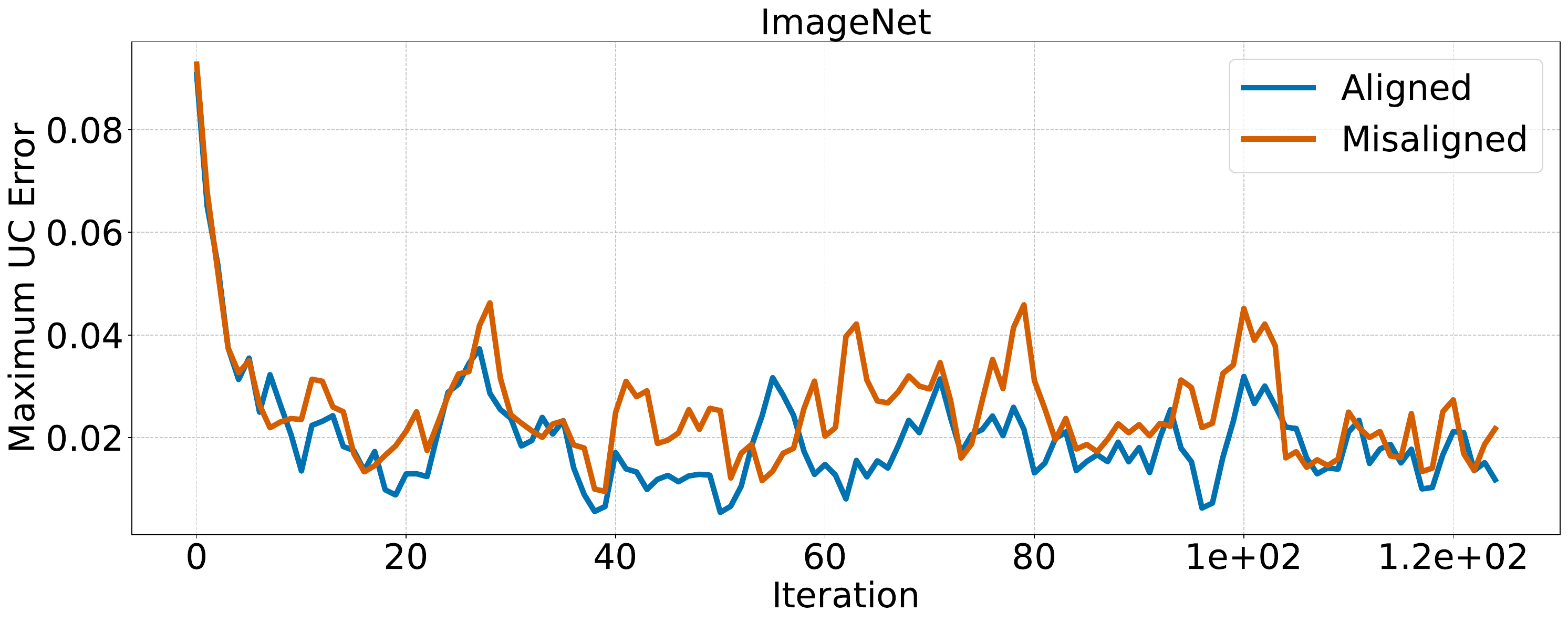}
  \end{minipage}
  \caption{Utility calibration error across iterations for CIFAR100 (left) and ImageNet-1K (right).}
  \label{fig:app_cifar100_history}
\end{figure}

\begin{figure}[htbp]
  \centering
  \vspace{0.6em}
  \includegraphics[width=0.8\linewidth]{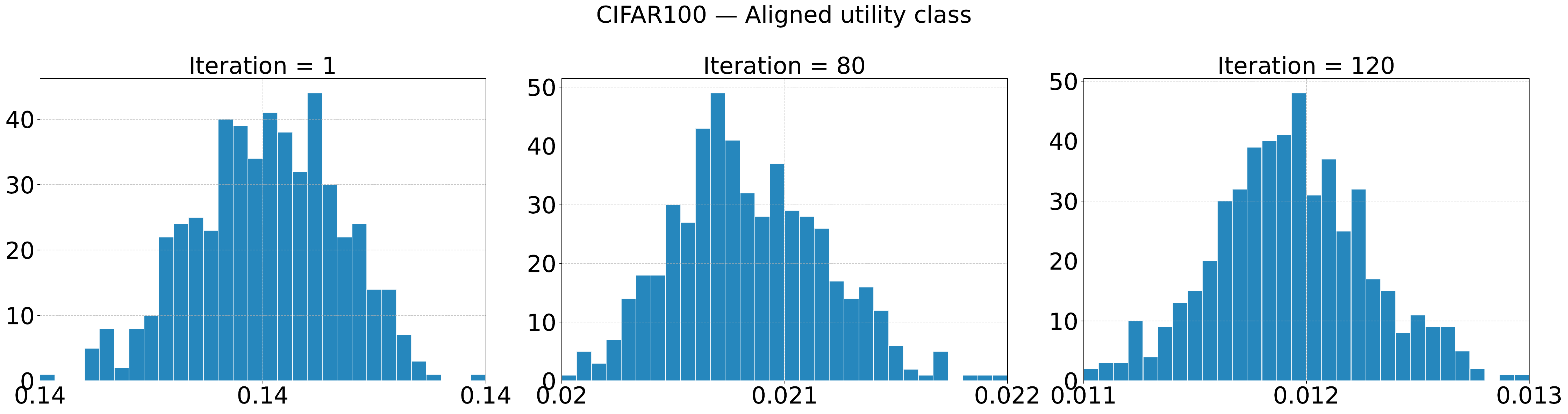}\\[0.25em]
  \includegraphics[width=0.8\linewidth]{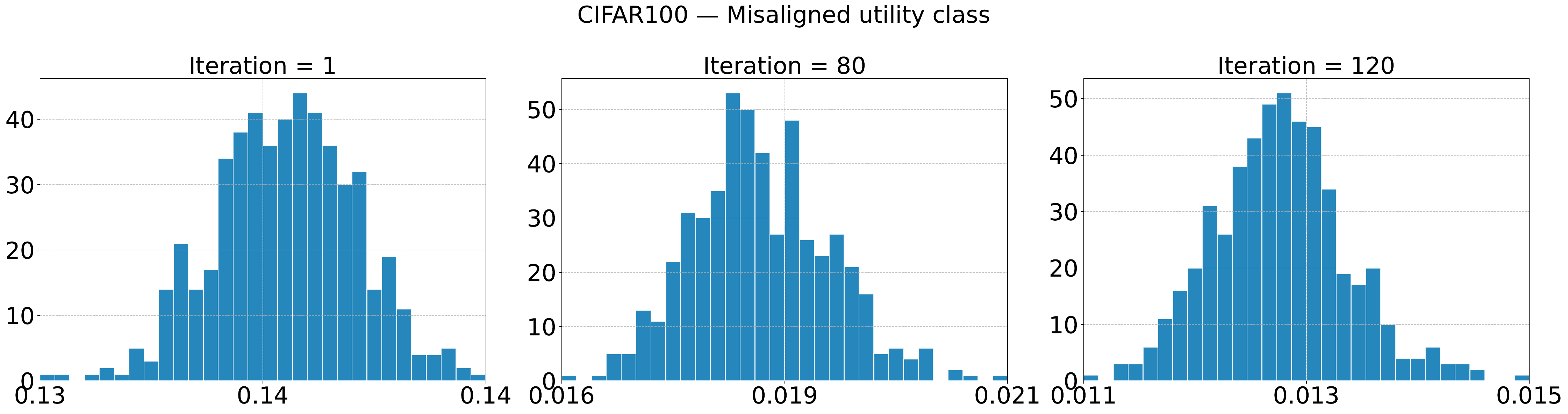}
  \includegraphics[width=0.8\linewidth]{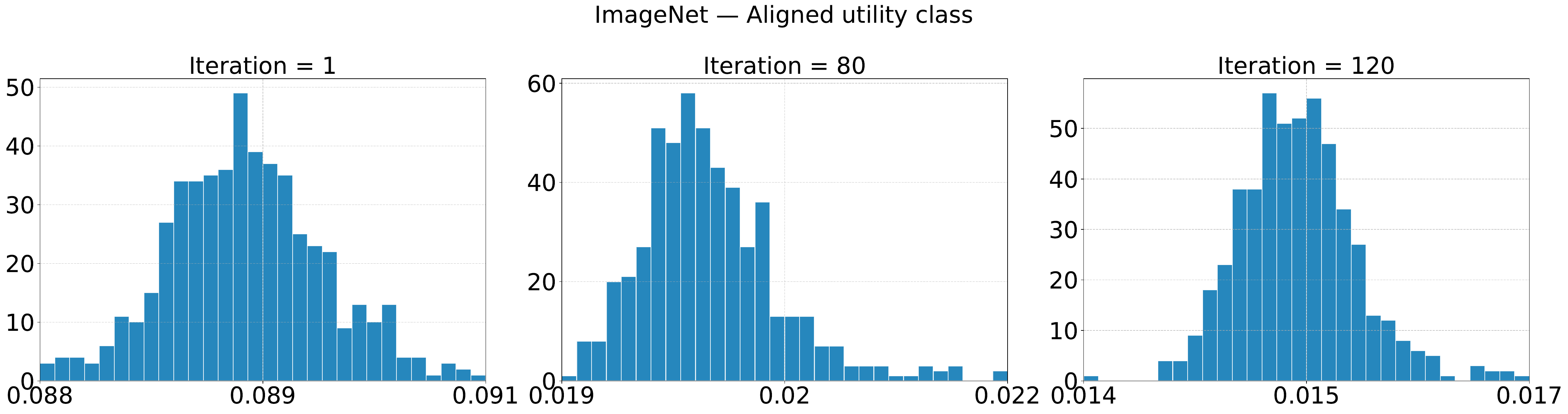}\\[0.25em]
  \includegraphics[width=0.8\linewidth]{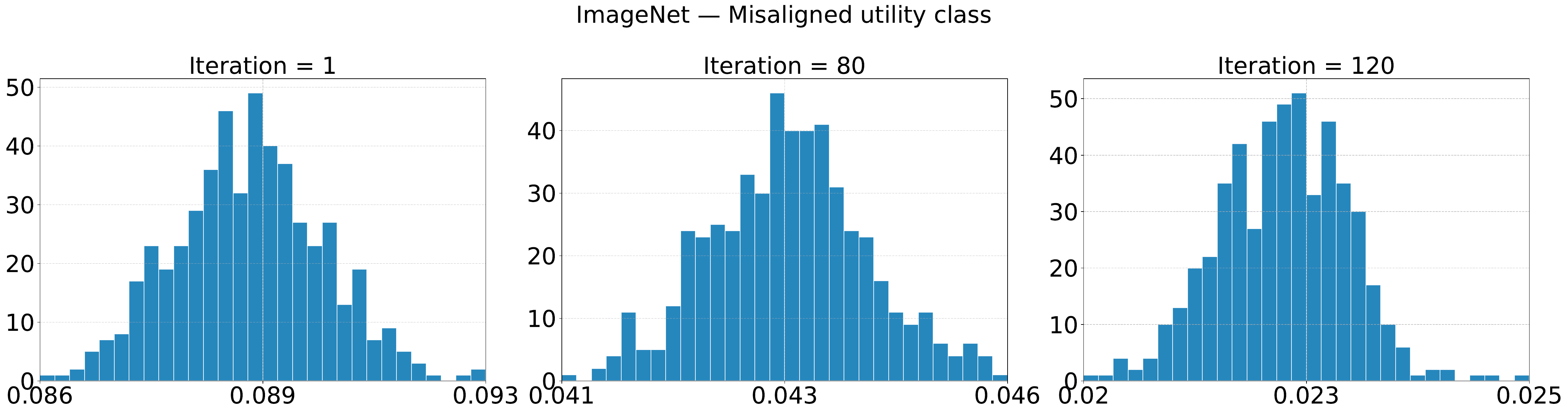}
  \caption{Aligned vs. misaligned utilities across datasets. Top row: utility calibration histories (left: CIFAR100, right: ImageNet-1K). Then: CIFAR100 distribution snapshots at selected iterations (aligned then misaligned), followed by ImageNet-1K distribution snapshots at selected iterations (aligned then misaligned).}
  \label{fig:app_aligned_misaligned_overview}
\end{figure}

\noindent\textbf{Interpretation.}
Across both datasets and both utility families (aligned and misaligned), the worst-case sampled utility calibration error initially drops and then stabilizes. Nonetheless,  utility calibration error distributions shift toward smaller values over iterations. Both aligned and misaligned utility classes displayed similar trends.
 
\subsection{Extended Results and Additional Modalities}\label{app:results}

We begin by introducing two additional utility families below:

\paragraph{Discounted cumulative gain (DCG).} We introduce the following discounted cumulative gain utility class, inspired by graded relevance evaluation \citep{jarvelin2002cumulated}. We note that this class highly resembles rank-based and $\mathrm{topK}$ utilities in \Cref{ex:util_rank}.
Let $\gamma\in \rset_+$ and define the discount sequence $\theta^{(\gamma)}_r=(\log_2(1+r))^{-\gamma}$ for ranks $r\in[C]$. For a prediction $p=f(X)$ and outcome $Y=e_j$, the realized gain is $u_\gamma(p,e_j)=\theta^{(\gamma)}_{\rank(p,j)}$ with associated regressor $v_{u_\gamma}(X)=\sum_i f(X)_i\,\theta^{(\gamma)}_{\rank(f(X),i)}$. Finally, for simplicity of evaluation, we used a fixed grid for $\gamma$ and thus define the class $\ucal_{\mathrm{dcg}} \coloneqq \{u_\gamma\mid\gamma \in\ens{0.5, 0.75, 1, 1.25, 1.5, 2}\}$. 

\paragraph{Semantic similarity utilities.} For ImageNet, we construct class embeddings $E\in\mathbb{R}^{C\times d}$ using the pre-trained Google News Word2Vec model (\texttt{word2vec-google-news-300} loaded via \texttt{gensim} ) \citep{mikolov2013efficient,rehurek2011gensim}. Class names are tokenized by lowercasing, replacing underscores/hyphens with spaces, and keeping alphanumeric tokens. For class $i$, a class embedding $E_i$ is the average of in-vocabulary token vectors. We compute the cosine similarity matrix $S\in[-1,1]^{C\times C}$ such that $S_{i,j}=1$ if $i=j$ and $S_{i,j}$ is the cosine similarity between $E_i$ and $E_j$. Finally, we define the utility by expected similarity: $u_{\mathrm{w2v}}(p,e_j)=\sum_i p_i\,S[i,j]$. We denote this single-utility class by $\ucal_{\mathrm{w2v}}$.

The results are presented in tables and figures as follows. We note two general observations that illustrate the usefulness of our evaluation framework. (1) The order of performance among different post-hoc methods varies with respect to the evaluated utility class, stressing the need to account for downstream usage when evaluating calibration. (2) While post-hoc methods generally improve various metrics, eCDF plots reveal additional information. For instance, in multiple settings, post-hoc methods actively worsened (shifted to the right) the distribution of utility calibration errors. 
\begin{center}
\begin{tabular}{lllll}
\toprule
Metrics & \multicolumn{4}{c}{Results given in} \\
 & CIFAR10 & CIFAR100 & ImageNet & Yahoo \\
\midrule
   $\mathrm{TCE}^{\mathrm{bin}}$,  $\mathrm{CWE}^{\mathrm{bin}}$,  $\ucal_{\cwe}$,  $\ucal_{\topk}$, $\ucal_{\mathrm{dcg}}$, $\ucal_{\mathrm{w2v}}$  & \Cref{tab:cifar10_repvgg,tab:cifar10_resnet} &\Cref{tab:cifar100_repvgg,tab:cifar100_resnet} &\Cref{tab:imagenet_resnet,tab:imagenet_vit} &\Cref{tab:yahoo} \\
  $\ucal_{\lin}$ and  $\ucal_{\rank}$   &  \Cref{fig:app_cifar10_curves}& \Cref{fig:app_cifar100_curves}& \Cref{fig:app_imagenet_curves} & \Cref{fig:app_yahoo_curves} \\
  \bottomrule
\end{tabular}
\end{center}

\begin{table*}[t]
  \caption{{\texttt{RepVGG-CIFAR10} calibration results, comparing post-hoc methods using Accuracy, Brier Score, binned ECEs ($\mathrm{TCE}^{\mathrm{Bin}}$, $\cwe^{\mathrm{Bin}}$), and utility calibration errors: $\ucal_\cwe$, $\ucal_{\topk}$, and $\ucal_{\mathrm{dcg}}$, with mean $\pm$ std.}}
  \label{tab:cifar10_repvgg}
  \centering
  \small
  \resizebox{\linewidth}{!}{
\begin{tabular}{@{}l@{\hspace{0.5em}}@{}c@{\hspace{0.5em}}@{}c@{\hspace{0.5em}}@{}c@{\hspace{0.5em}}@{}c@{\hspace{0.5em}}@{}c@{\hspace{0.5em}}@{}c@{\hspace{0.5em}}@{}c@{}}
    \toprule
    Method & Accuracy ($\times 10^{2}$) & Brier Score ($\times 10^{2}$) & $\mathrm{CWE}^{\mathrm{bin}}$ ($\times 10^{4}$) & $\mathrm{TCE}^{\mathrm{bin}}$ ($\times 10^{3}$) & $\mathcal{U}_{\mathrm{CWE}}$ ($\times 10^{4}$) & $\mathcal{U}_{\topk}$ ($\times 10^{3}$) & $\mathcal{U}_{\mathrm{dcg}}$ ($\times 10^{3}$) \\
    \midrule
    Uncalibrated & $94.5 \pm 0.2$ & $8.9 \pm 0.3$ & $39.80 \pm 5$ & $37.40 \pm 4$ & $108.00 \pm 4$ & $38.2 \pm 2$ & $24.9 \pm 1$ \\
    Dirichlet & $94.6 \pm 0.2$ & $8.0 \pm 0.2$ & $50.40 \pm 2$ & $20.20 \pm 0.9$ & $89.90 \pm 4$ & $18.5 \pm 1$ & $12.8 \pm 1$ \\
    IR & $94.6 \pm 0.2$ & $8.2 \pm 0.2$ & $24.80 \pm 0.9$ & $8.36 \pm 1$ & $76.90 \pm 4$ & $15.4 \pm 0.6$ & $10.5 \pm 0.5$ \\
    Temp. Scaling & $94.5 \pm 0.2$ & $8.1 \pm 0.2$ & $59.70 \pm 2$ & $21.30 \pm 0.8$ & $89.70 \pm 5$ & $22.8 \pm 1$ & $14.0 \pm 1.0$ \\
    Vector Scaling & $94.7 \pm 0.2$ & $8.0 \pm 0.2$ & $47.10 \pm 2$ & $15.80 \pm 1$ & $95.30 \pm 5$ & $16.1 \pm 0.9$ & $9.5 \pm 0.8$ \\
    Patching & $94.6 \pm 0.2$ & $8.3 \pm 0.2$ & $40.80 \pm 2$ & $11.10 \pm 2$ & $80.50 \pm 4$ & $12.7 \pm 0.8$ & $6.4 \pm 1$ \\
    \bottomrule
\end{tabular}
  }
\end{table*}

\begin{table*}[t]
  \caption{{\texttt{ResNet-CIFAR10} calibration results, comparing post-hoc methods using Accuracy, Brier Score, binned ECEs ($\mathrm{TCE}^{\mathrm{Bin}}$, $\cwe^{\mathrm{Bin}}$), and utility calibration errors: $\ucal_\cwe$, $\ucal_{\topk}$, and $\ucal_{\mathrm{dcg}}$, with mean $\pm$ std.}}
  \label{tab:cifar10_resnet}
  \centering
  \small
  \resizebox{\linewidth}{!}{
\begin{tabular}{@{}l@{\hspace{0.5em}}@{}c@{\hspace{0.5em}}@{}c@{\hspace{0.5em}}@{}c@{\hspace{0.5em}}@{}c@{\hspace{0.5em}}@{}c@{\hspace{0.5em}}@{}c@{\hspace{0.5em}}@{}c@{}}
    \toprule
    Method & Accuracy ($\times 10^{2}$) & Brier Score ($\times 10^{2}$) & $\mathrm{CWE}^{\mathrm{bin}}$ ($\times 10^{4}$) & $\mathrm{TCE}^{\mathrm{bin}}$ ($\times 10^{3}$) & $\mathcal{U}_{\mathrm{CWE}}$ ($\times 10^{4}$) & $\mathcal{U}_{\topk}$ ($\times 10^{3}$) & $\mathcal{U}_{\mathrm{dcg}}$ ($\times 10^{3}$) \\
    \midrule
    Uncalibrated & $93.9 \pm 0.2$ & $10.1 \pm 0.3$ & $96.40 \pm 8$  & $41.60 \pm 3$ & $116.00 \pm 6$ & $42.1 \pm 1$ & $27.9 \pm 1$ \\
    Dirichlet & $94.0 \pm 0.1$ & $9.2 \pm 0.2$ & $42.60 \pm 2$ & $14.20 \pm 0.9$ & $85.80 \pm 6$ & $14.2 \pm 1$ & $8.4 \pm 1$ \\
    IR & $94.0 \pm 0.2$ & $9.3 \pm 0.2$ & $29.50 \pm 1$ & $10.80 \pm 1$ & $77.90 \pm 5$ & $16.2 \pm 0.2$ & $11.8 \pm 0.5$ \\
    Temp. Scaling & $93.9 \pm 0.2$ & $9.3 \pm 0.3$ & $55.50 \pm 2$ & $14.70 \pm 1$ & $92.90 \pm 5$ & $17.1 \pm 1.0$ & $9.2 \pm 1$ \\
    Vector Scaling & $94.0 \pm 0.2$ & $9.2 \pm 0.2$ & $42.40 \pm 2$ & $11.90 \pm 0.8$ & $84.70 \pm 4$ & $13.1 \pm 0.5$ & $6.3 \pm 0.8$ \\
    Patching & $94.0 \pm 0.2$ & $9.5 \pm 0.2$ & $42.80 \pm 6$ & $12.30 \pm 3$ & $84.00 \pm 8$ & $12.8 \pm 1$ & $6.0 \pm 1$ \\
    \bottomrule
\end{tabular}
  }
\end{table*}

\begin{table*}[t]
  \caption{{\texttt{RepVGG-CIFAR100} calibration results, comparing post-hoc methods using Accuracy, Brier Score, binned ECEs ($\mathrm{TCE}^{\mathrm{Bin}}$, $\cwe^{\mathrm{Bin}}$), and utility calibration errors: $\ucal_\cwe$ and $\ucal_{\topk}$, with mean $\pm$ std.}}
  \label{tab:cifar100_repvgg}
  \centering
  \small
  \resizebox{\linewidth}{!}{
\begin{tabular}{@{}l@{\hspace{0.5em}}@{}c@{\hspace{0.5em}}@{}c@{\hspace{0.5em}}@{}c@{\hspace{0.5em}}@{}c@{\hspace{0.5em}}@{}c@{\hspace{0.5em}}@{}c@{}}
    \toprule
    Method & Accuracy ($\times 10^{2}$) & Brier Score ($\times 10^{2}$) & $\mathrm{CWE}^{\mathrm{bin}}$ ($\times 10^{4}$) & $\mathrm{TCE}^{\mathrm{bin}}$ ($\times 10^{3}$) & $\mathcal{U}_{\mathrm{CWE}}$ ($\times 10^{4}$) & $\mathcal{U}_{\topk}$ ($\times 10^{3}$) \\
    \midrule
    Uncalibrated & $77.0 \pm 0.4$ & $33.2 \pm 0.3$ & $15.60 \pm 0.06$ & $62.50 \pm 6$ & $76.20 \pm 2$ & $56.0 \pm 0.6$ \\
    Dirichlet & $76.2 \pm 0.6$ & $33.6 \pm 0.3$ & $17.40 \pm 0.9$ & $52.80 \pm 2$ & $78.00 \pm 4$ & $51.1 \pm 1$ \\
    IR & $76.4 \pm 0.3$ & $33.5 \pm 0.5$ & $14.80 \pm 0.5$ & $32.60 \pm 4$ & $66.20 \pm 3$ & $49.5 \pm 2$ \\
    Temp. Scaling & $77.0 \pm 0.4$ & $33.0 \pm 0.3$ & $17.20 \pm 0.4$ & $49.30 \pm 6$ & $76.00 \pm 2$ & $65.9 \pm 2$ \\
    Vector Scaling & $76.5 \pm 0.5$ & $33.7 \pm 0.3$ & $18.20 \pm 0.9$ & $53.30 \pm 1$ & $75.00 \pm 3$ & $51.2 \pm 2$ \\
    Patching & $76.8 \pm 0.6$ & $32.5 \pm 0.3$ & $18.70 \pm 0.3$ & $24.30 \pm 1$ & $67.10 \pm 3$ & $23.7 \pm 4$ \\
    \bottomrule
\end{tabular}
  }
\end{table*}

\begin{table*}[t]
  \caption{{\texttt{ResNet-CIFAR100} calibration results, comparing post-hoc methods using Accuracy, Brier Score, binned ECEs ($\mathrm{TCE}^{\mathrm{Bin}}$, $\cwe^{\mathrm{Bin}}$), and utility calibration errors: $\ucal_\cwe$, $\ucal_{\topk}$, and $\ucal_{\mathrm{dcg}}$, with mean $\pm$ std.}}
  \label{tab:cifar100_resnet}
  \centering
  \small
  \resizebox{\linewidth}{!}{
\begin{tabular}{@{}l@{\hspace{0.5em}}@{}c@{\hspace{0.5em}}@{}c@{\hspace{0.5em}}@{}c@{\hspace{0.5em}}@{}c@{\hspace{0.5em}}@{}c@{\hspace{0.5em}}@{}c@{\hspace{0.5em}}@{}c@{}}
    \toprule
    Method & Accuracy ($\times 10^{2}$) & Brier Score ($\times 10^{2}$) & $\mathrm{CWE}^{\mathrm{bin}}$ ($\times 10^{4}$) & $\mathrm{TCE}^{\mathrm{bin}}$ ($\times 10^{3}$) & $\mathcal{U}_{\mathrm{CWE}}$ ($\times 10^{4}$) & $\mathcal{U}_{\topk}$ ($\times 10^{3}$) & $\mathcal{U}_{\mathrm{dcg}}$ ($\times 10^{3}$) \\
    \midrule
    Uncalibrated & $72.4 \pm 0.4$ & $42.3 \pm 0.6$ & $14.60 \pm 0.8$ & $144.00 \pm 7$ & $79.50 \pm 2$ & $144.0 \pm 3$ & $117.0 \pm 2$ \\
    Dirichlet & $70.4 \pm 0.4$ & $41.0 \pm 0.3$ & $19.10 \pm 0.4$ & $53.70 \pm 4$ & $74.00 \pm 2$ & $61.2 \pm 3$ & $50.4 \pm 3$ \\
    IR & $71.8 \pm 0.3$ & $39.8 \pm 0.5$ & $15.30 \pm 0.4$ & $37.70 \pm 2$ & $70.70 \pm 2$ & $71.2 \pm 2$ & $51.1 \pm 2$ \\
    Temp. Scaling & $72.4 \pm 0.4$ & $38.7 \pm 0.5$ & $15.00 \pm 0.2$ & $31.60 \pm 3$ & $67.20 \pm 2$ & $29.5 \pm 2$ & $15.8 \pm 2$ \\
    Vector Scaling & $71.8 \pm 0.4$ & $39.2 \pm 0.4$ & $17.80 \pm 0.4$ & $35.60 \pm 3$ & $76.00 \pm 3$ & $36.1 \pm 3$ & $22.6 \pm 3$ \\
    Patching & $71.6 \pm 0.4$ & $38.7 \pm 0.5$ & $21.80 \pm 0.7$ & $25.60 \pm 5$ & $69.30 \pm 3$ & $26.2 \pm 3$ & $11.9 \pm 3$ \\
    \bottomrule
\end{tabular}
  }
\end{table*}

\begin{table*}[t]
  \caption{{\texttt{ResNet-ImageNet} calibration results, comparing post-hoc methods using Accuracy, Brier Score, binned ECEs ($\mathrm{TCE}^{\mathrm{Bin}}$, $\cwe^{\mathrm{Bin}}$), and utility calibration errors: $\ucal_\cwe$, $\ucal_{\topk}$, $\ucal_{\mathrm{dcg}}$, and $\ucal_{\mathrm{w2v}}$, with mean $\pm$ std.}}
  \label{tab:imagenet_resnet}
  \centering
  \small
  \resizebox{\linewidth}{!}{
\begin{tabular}{@{}l@{\hspace{0.5em}}@{}c@{\hspace{0.5em}}@{}c@{\hspace{0.5em}}@{}c@{\hspace{0.5em}}@{}c@{\hspace{0.5em}}@{}c@{\hspace{0.5em}}@{}c@{\hspace{0.5em}}@{}c@{\hspace{0.5em}}@{}c@{}}
    \toprule
    Method & Accuracy ($\times 10^{2}$) & Brier Score ($\times 10^{2}$) & $\mathrm{CWE}^{\mathrm{bin}}$ ($\times 10^{4}$) & $\mathrm{TCE}^{\mathrm{bin}}$ ($\times 10^{3}$) & $\mathcal{U}_{\mathrm{CWE}}$ ($\times 10^{4}$) & $\mathcal{U}_{\topk}$ ($\times 10^{3}$) & $\mathcal{U}_{\mathrm{dcg}}$ ($\times 10^{3}$) & $\mathcal{U}_{\mathrm{w2v}}$ ($\times 10^{3}$) \\
    \midrule
    Uncalibrated & $80.5 \pm 0.1$ & $29.5 \pm 0.2$ & $1.42 \pm 0.04$ & $86.20 \pm 2$ & $31.00 \pm 0.6$ & $89.1 \pm 2$ & $62.2 \pm 0.8$ & $42.6 \pm 1.0$ \\
    Dirichlet & $1.0 \pm 0.4$ & $117.0 \pm 6$ & $20.20 \pm 0.4$ & $195.00 \pm 57$ & $1860.00 \pm 580$ & $990.0 \pm 3$ & $429.0 \pm 59$ & $147.0 \pm 43$ \\
    IR & $80.1 \pm 0.1$ & $29.5 \pm 0.2$ & $1.47 \pm 0.01$ & $37.10 \pm 1$ & $30.50 \pm 0.6$ & $65.2 \pm 1$ & $47.3 \pm 0.6$ & $37.9 \pm 0.9$ \\
    Temp. Scaling & $80.5 \pm 0.1$ & $28.5 \pm 0.2$ & $1.69 \pm 0.02$ & $45.50 \pm 1$ & $31.50 \pm 0.7$ & $59.3 \pm 1$ & $24.7 \pm 0.7$ & $16.2 \pm 0.4$ \\
    Vector Scaling & $79.7 \pm 0.1$ & $30.5 \pm 0.2$ & $1.89 \pm 0.01$ & $67.10 \pm 2$ & $31.60 \pm 0.8$ & $67.1 \pm 2$ & $44.4 \pm 1$ & $32.2 \pm 0.9$ \\
    Patching & $80.1 \pm 0.3$ & $28.0 \pm 0.3$ & $2.02 \pm 0.09$ & $13.20 \pm 3$ & $31.30 \pm 0.7$ & $25.9 \pm 4$ & $8.4 \pm 4$ & $15.1 \pm 3$ \\
    \bottomrule
\end{tabular}
  }
\end{table*}

\begin{table*}[t]
  \caption{{\texttt{ViT-ImageNet} calibration results, comparing post-hoc methods using Accuracy, Brier Score, binned ECEs ($\mathrm{TCE}^{\mathrm{Bin}}$, $\cwe^{\mathrm{Bin}}$), and utility calibration errors: $\ucal_\cwe$, $\ucal_{\topk}$, $\ucal_{\mathrm{dcg}}$, and $\ucal_{\mathrm{w2v}}$, with mean $\pm$ std.}}
  \label{tab:imagenet_vit}
  \centering
  \small
  \resizebox{\linewidth}{!}{
\begin{tabular}{@{}l@{\hspace{0.5em}}@{}c@{\hspace{0.5em}}@{}c@{\hspace{0.5em}}@{}c@{\hspace{0.5em}}@{}c@{\hspace{0.5em}}@{}c@{\hspace{0.5em}}@{}c@{\hspace{0.5em}}@{}c@{\hspace{0.5em}}@{}c@{}}
    \toprule
    Method & Accuracy ($\times 10^{2}$) & Brier Score ($\times 10^{2}$) & $\mathrm{CWE}^{\mathrm{bin}}$ ($\times 10^{4}$) & $\mathrm{TCE}^{\mathrm{bin}}$ ($\times 10^{3}$) & $\mathcal{U}_{\mathrm{CWE}}$ ($\times 10^{4}$) & $\mathcal{U}_{\topk}$ ($\times 10^{3}$) & $\mathcal{U}_{\mathrm{dcg}}$ ($\times 10^{3}$) & $\mathcal{U}_{\mathrm{w2v}}$ ($\times 10^{3}$) \\
    \midrule
    Uncalibrated & $85.2 \pm 0.1$ & $22.5 \pm 0.1$ & $2.47 \pm 0.01$ & $94.50 \pm 0.8$ & $29.00 \pm 0.7$ & $124.0 \pm 0.5$ & $101.0 \pm 0.6$ & $69.2 \pm 0.7$ \\
    Dirichlet & $85.5 \pm 0.10$ & $21.3 \pm 0.1$ & $1.35 \pm 0.02$ & $13.80 \pm 0.6$ & $27.00 \pm 0.5$ & $26.1 \pm 0.7$ & $6.4 \pm 0.4$ & $7.9 \pm 0.5$ \\
    IR & $84.9 \pm 0.1$ & $22.8 \pm 0.2$ & $1.12 \pm 0.01$ & $31.70 \pm 0.8$ & $27.40 \pm 0.5$ & $54.1 \pm 0.9$ & $41.6 \pm 0.5$ & $36.6 \pm 1$ \\
    Temp. Scaling & $85.2 \pm 0.1$ & $21.6 \pm 0.1$ & $1.50 \pm 0.009$ & $21.20 \pm 0.6$ & $28.10 \pm 0.6$ & $45.2 \pm 0.5$ & $20.8 \pm 0.6$ & $17.8 \pm 0.7$ \\
    Vector Scaling & $84.8 \pm 0.1$ & $22.7 \pm 0.2$ & $1.55 \pm 0.02$ & $34.60 \pm 1$ & $28.30 \pm 0.3$ & $37.4 \pm 2$ & $21.9 \pm 1$ & $12.7 \pm 0.8$ \\
    Patching & $85.2 \pm 0.1$ & $21.6 \pm 0.1$ & $1.56 \pm 0.03$ & $11.10 \pm 0.7$ & $29.60 \pm 2$ & $22.1 \pm 2$ & $5.5 \pm 0.5$ & $11.9 \pm 1.0$ \\
    \bottomrule
\end{tabular}
  }
\end{table*}

\begin{table*}[t]
  \caption{{Yahoo calibration results, comparing post-hoc methods using Accuracy, Brier Score, binned ECEs ($\mathrm{TCE}^{\mathrm{Bin}}$, $\cwe^{\mathrm{Bin}}$), and utility calibration errors: $\ucal_\cwe$, $\ucal_{\topk}$, and $\ucal_{\mathrm{dcg}}$, with mean $\pm$ std.}}
  \label{tab:yahoo}
  \centering
  \small
  \resizebox{\linewidth}{!}{
\begin{tabular}{@{}l@{\hspace{0.5em}}@{}c@{\hspace{0.5em}}@{}c@{\hspace{0.5em}}@{}c@{\hspace{0.5em}}@{}c@{\hspace{0.5em}}@{}c@{\hspace{0.5em}}@{}c@{\hspace{0.5em}}@{}c@{}}
    \toprule
    Method & Accuracy ($\times 10^{2}$) & Brier Score ($\times 10^{2}$) & $\mathrm{CWE}^{\mathrm{bin}}$ ($\times 10^{4}$) & $\mathrm{TCE}^{\mathrm{bin}}$ ($\times 10^{3}$) & $\mathcal{U}_{\mathrm{CWE}}$ ($\times 10^{4}$) & $\mathcal{U}_{\topk}$ ($\times 10^{3}$) & $\mathcal{U}_{\mathrm{dcg}}$ ($\times 10^{3}$) \\
    \midrule
    Uncalibrated & $72.3 \pm 0.2$ & $39.2 \pm 0.2$ & $102.00 \pm 4$ & $31.40 \pm 2$ & $211.00 \pm 7$ & $37.4 \pm 4$ & $20.1 \pm 0.7$ \\
    Dirichlet & $72.8 \pm 0.2$ & $38.5 \pm 0.2$ & $69.30 \pm 2$ & $15.10 \pm 1$ & $158.00 \pm 4$ & $33.6 \pm 2$ & $12.2 \pm 2$ \\
    IR & $72.7 \pm 0.2$ & $38.7 \pm 0.2$ & $38.90 \pm 2$ & $11.00 \pm 0.6$ & $150.00 \pm 5$ & $30.8 \pm 0.7$ & $11.9 \pm 0.4$ \\
    Temp. Scaling & $72.3 \pm 0.2$ & $39.0 \pm 0.2$ & $103.00 \pm 0.7$ & $14.00 \pm 1.0$ & $213.00 \pm 6$ & $31.4 \pm 2$ & $9.5 \pm 1$ \\
    Vector Scaling & $72.7 \pm 0.2$ & $38.7 \pm 0.2$ & $76.50 \pm 2$ & $13.90 \pm 0.8$ & $169.00 \pm 7$ & $31.2 \pm 2$ & $9.1 \pm 1$ \\
    Patching & $72.7 \pm 0.2$ & $38.9 \pm 0.2$ & $73.90 \pm 5$ & $12.00 \pm 2$ & $172.00 \pm 7$ & $23.5 \pm 2$ & $5.3 \pm 1$ \\
    \bottomrule
\end{tabular}
  }
\end{table*}

\begin{figure}[htbp]
  \centering

  \begin{subfigure}{\linewidth}
      \centering
      \includegraphics[width=0.9\linewidth]{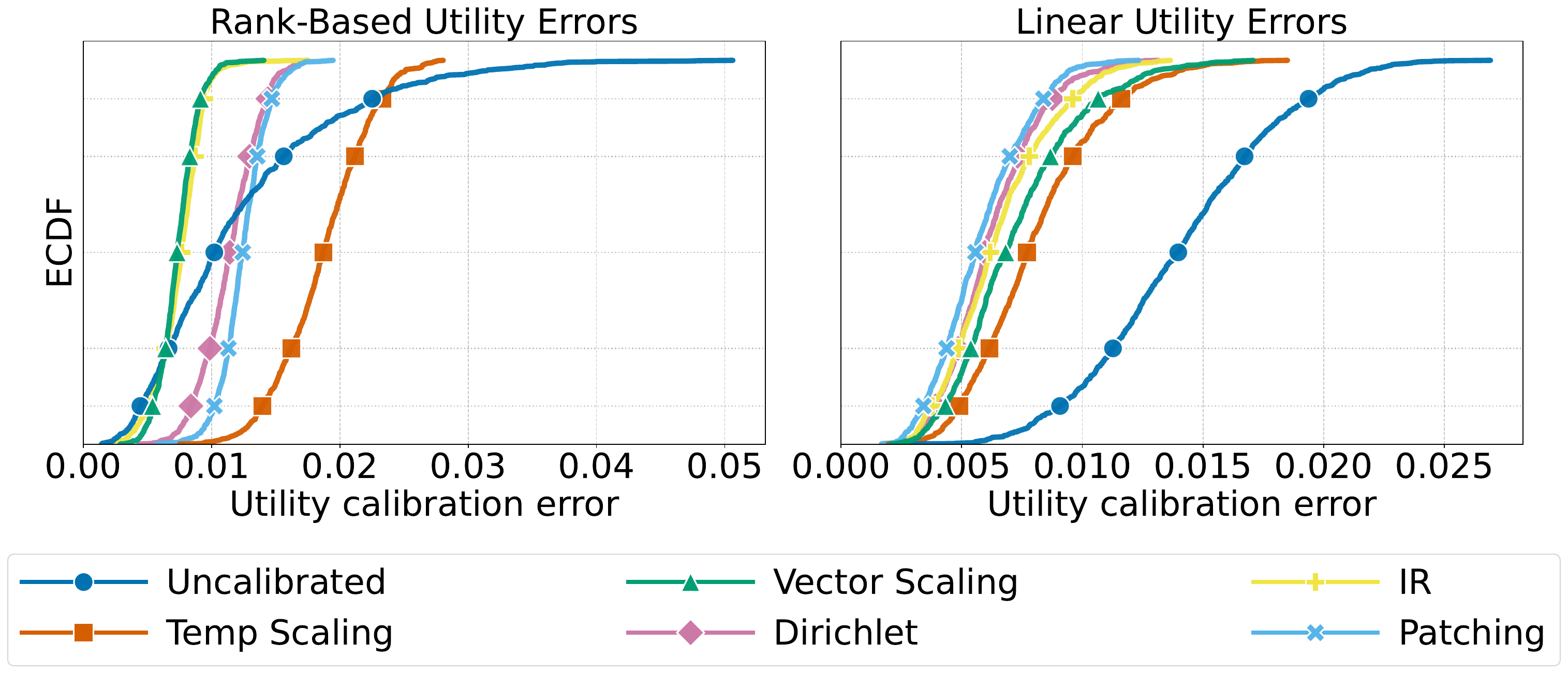}
      \caption*{\texttt{RepVGG-CIFAR10}}
  \end{subfigure}

  \begin{subfigure}{\linewidth}
      \centering
      \includegraphics[width=0.9\linewidth]{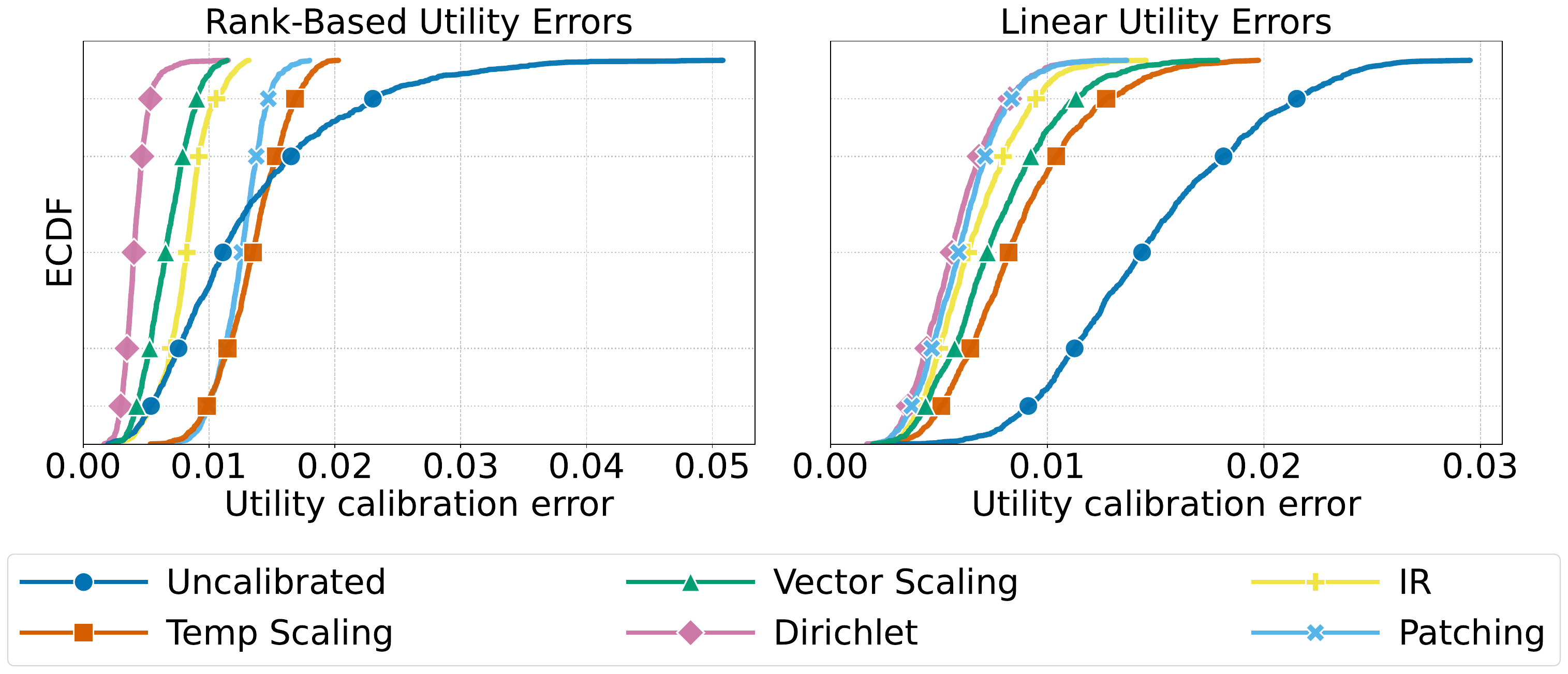}
      \caption*{\texttt{ResNet56-CIFAR10}}
  \end{subfigure}

  \caption{CIFAR10 eCDF plots for $\ucal_{\lin}$ and $\ucal_{\rank}$.}
  \label{fig:app_cifar10_curves}
\end{figure}

\begin{figure}[htbp]
  \centering

  \begin{subfigure}{\linewidth}
      \centering
      \includegraphics[width=0.9\linewidth]{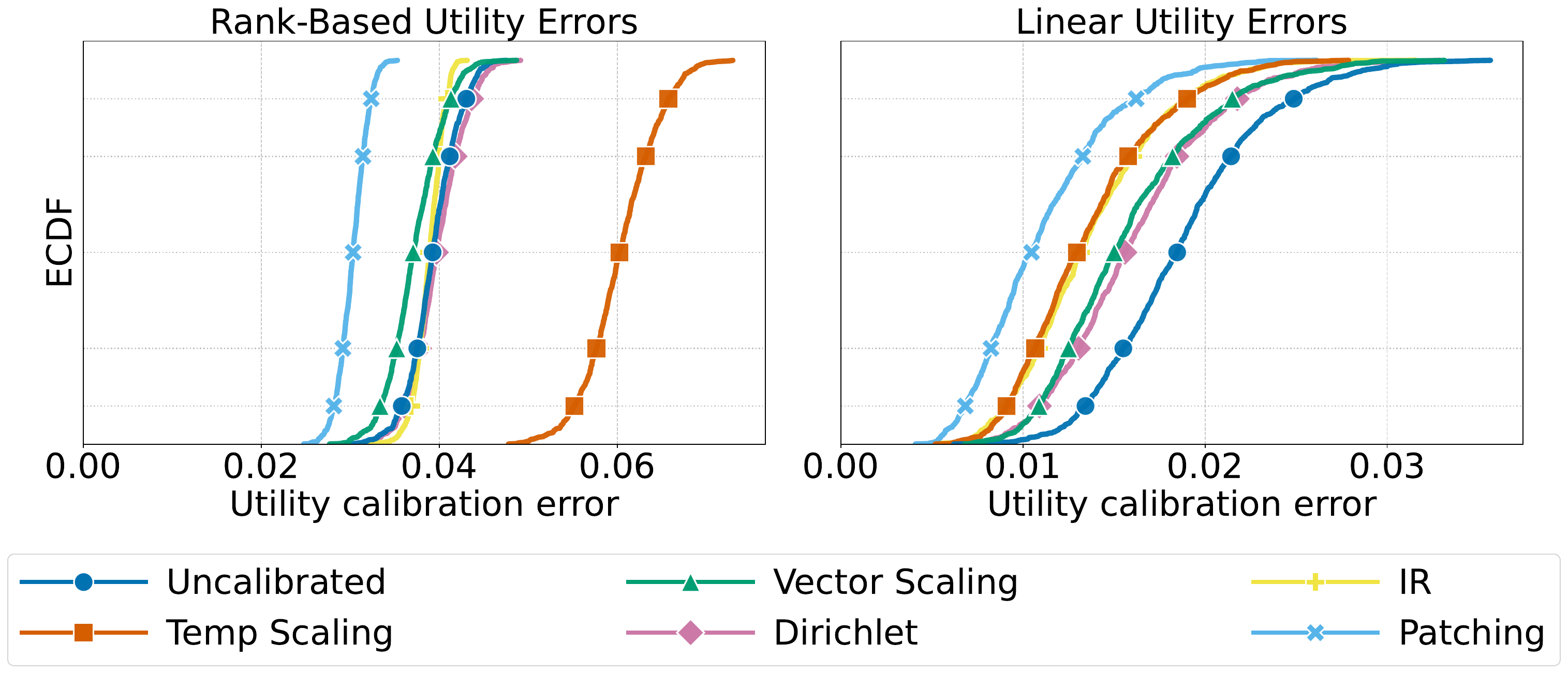}
      \caption*{\texttt{RepVGG-CIFAR100}}
  \end{subfigure}

  \begin{subfigure}{\linewidth}
      \centering
      \includegraphics[width=0.9\linewidth]{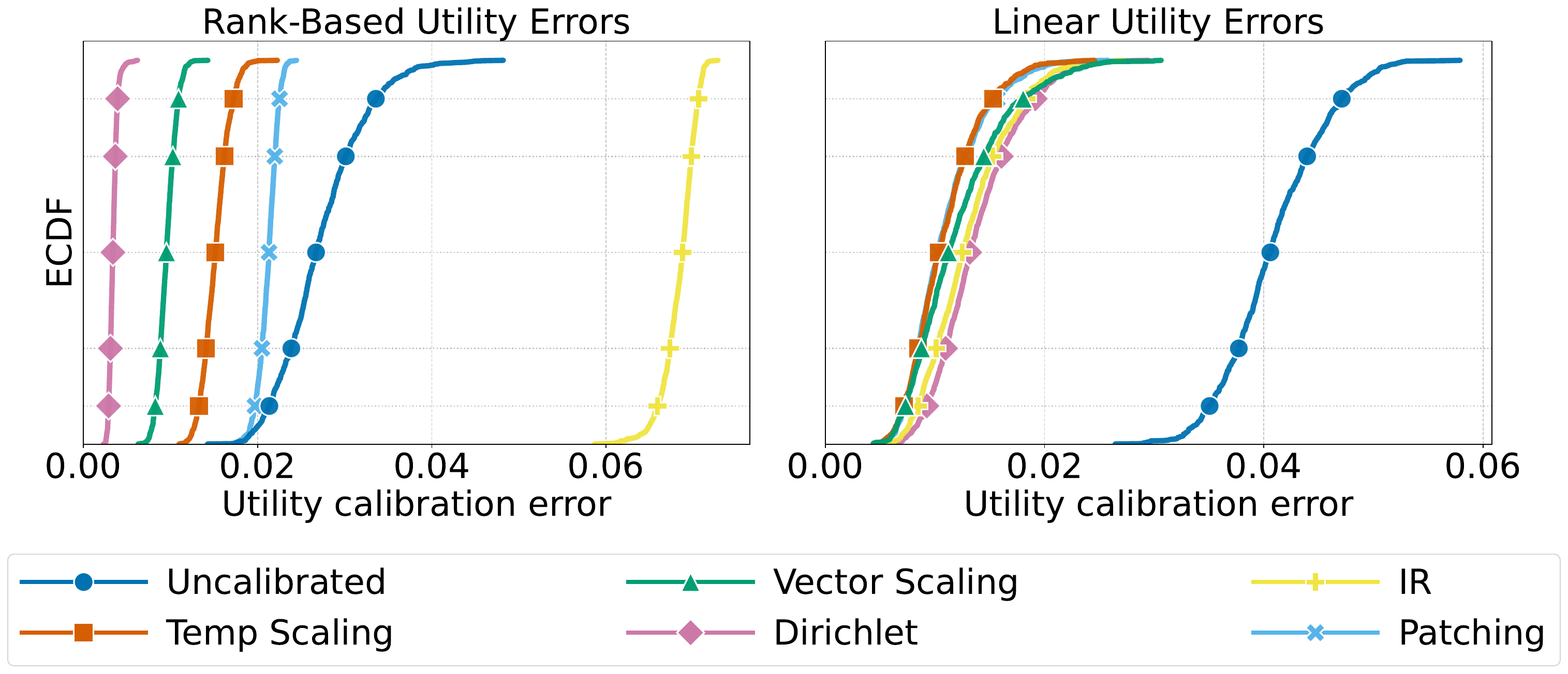}
      \caption*{\texttt{ResNet56-CIFAR100}}
  \end{subfigure}

  \caption{CIFAR100 eCDF plots for $\ucal_{\lin}$ and $\ucal_{\rank}$.}
  \label{fig:app_cifar100_curves}
\end{figure}

\begin{figure}[htbp]
  \centering

  \begin{subfigure}{\linewidth}
      \centering
      \includegraphics[width=0.9\linewidth]{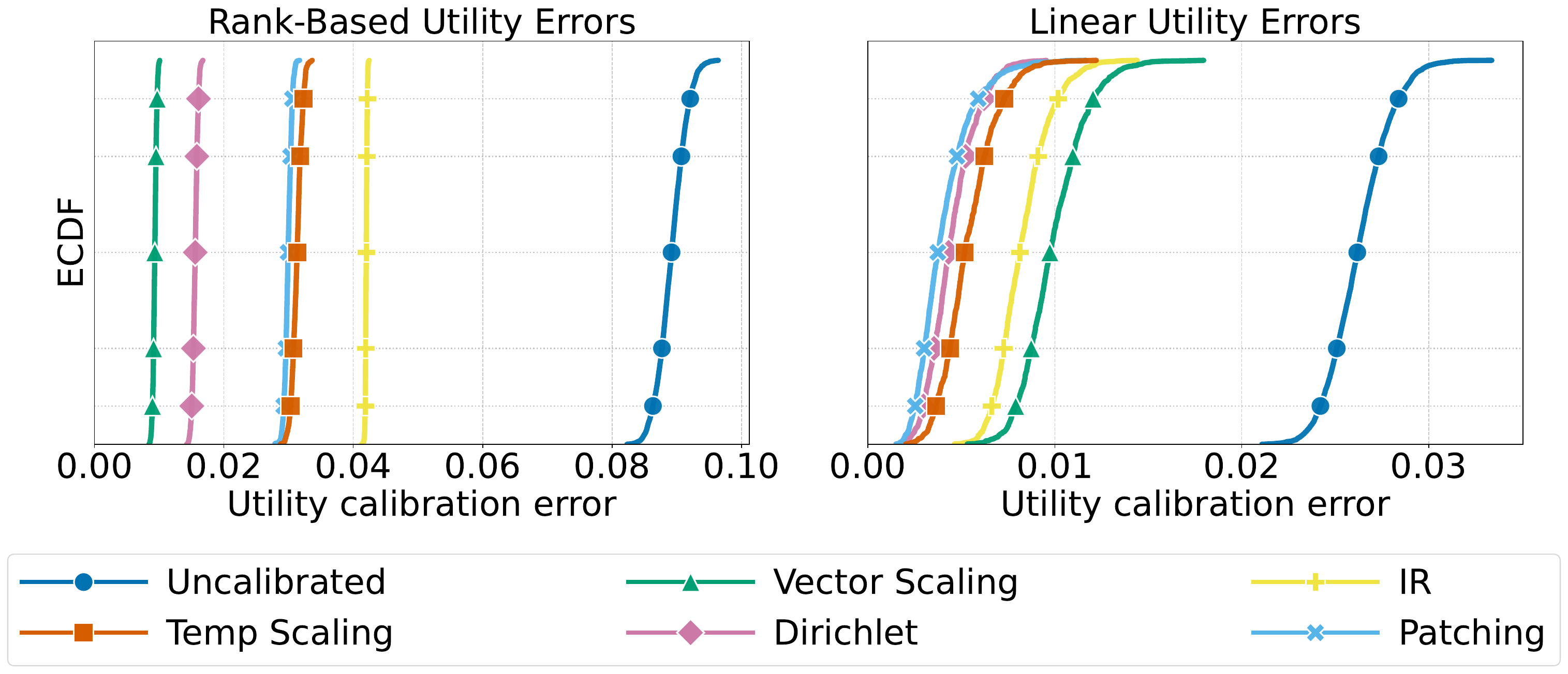}
      \caption*{\texttt{ViT\textunderscore Base\textunderscore P16\textunderscore 224-ImageNet1k}}
  \end{subfigure}

  \begin{subfigure}{\linewidth}
      \centering
      \includegraphics[width=0.9\linewidth]{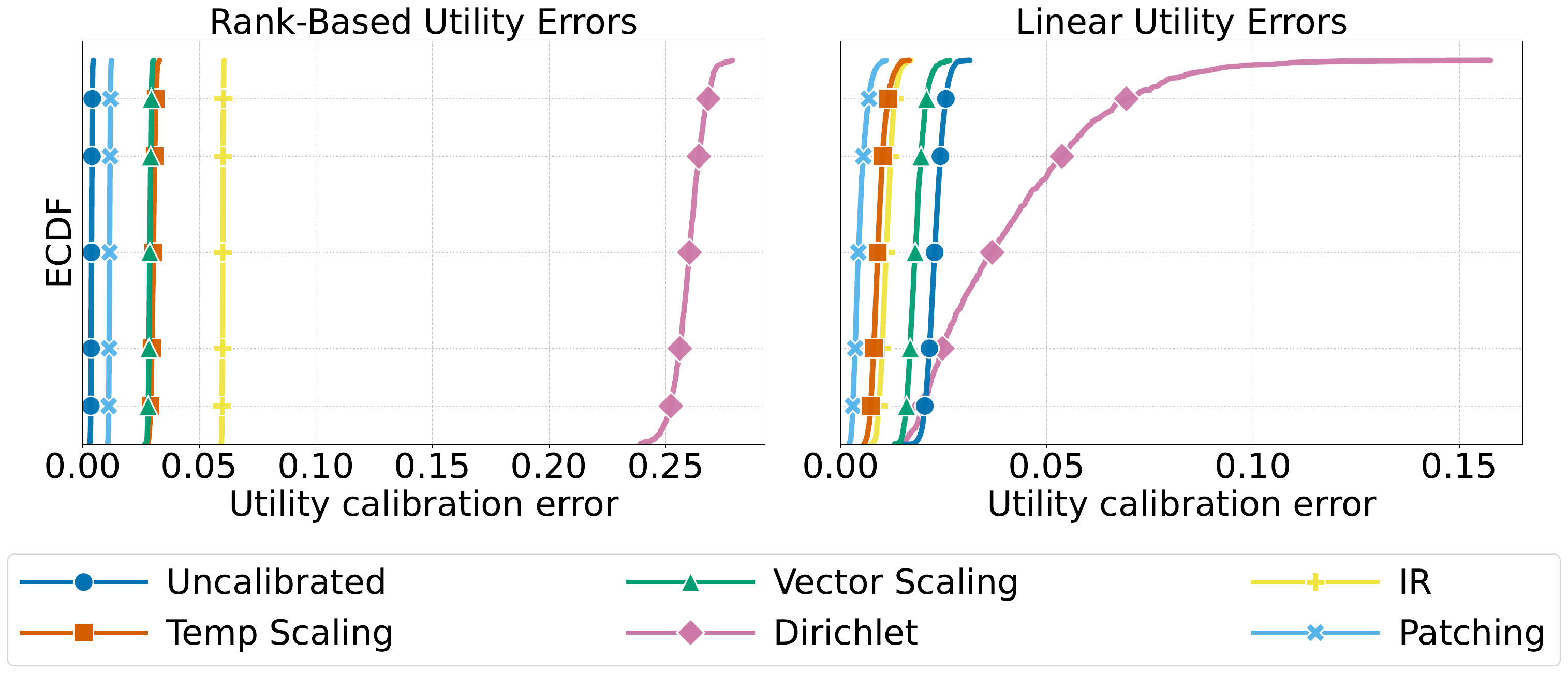}
      \caption*{\texttt{ResNet50-ImageNet1k}}
  \end{subfigure}

  \caption{ImageNet eCDF plots for $\ucal_{\lin}$ and $\ucal_{\rank}$.}
  \label{fig:app_imagenet_curves}
\end{figure}

\begin{figure}[htbp]
    \centering
    \includegraphics[width=0.9\linewidth]{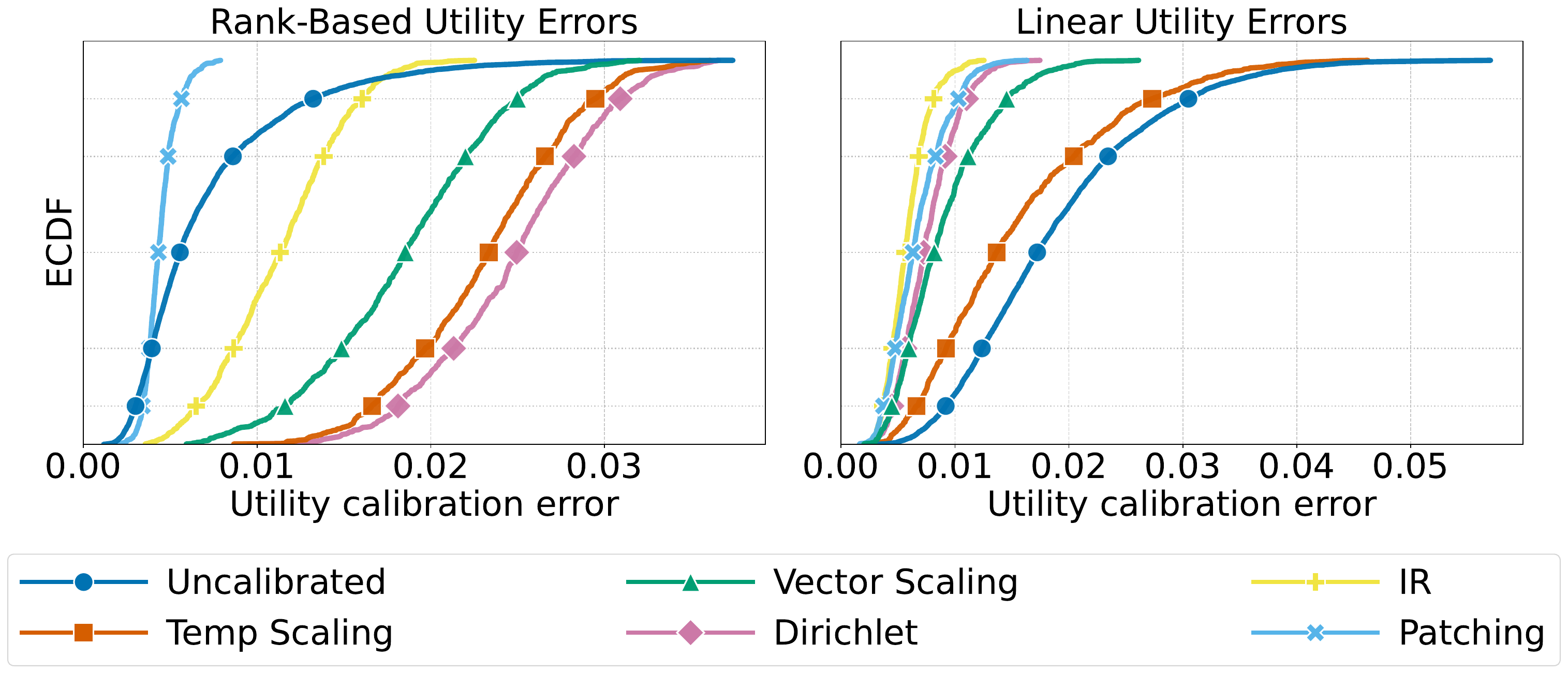}
    \caption{Yahoo Answers eCDF plots for $\ucal_{\lin}$ and $\ucal_{\rank}$.}
    \label{fig:app_yahoo_curves}
\end{figure}

\end{document}